\pdfoutput=1
\RequirePackage{fix-cm}
\documentclass[twocolumn]{svjour3}          
\smartqed  
\usepackage{graphicx}
\usepackage{times}
\usepackage{epsfig}
\usepackage{epstopdf}
\usepackage{amsmath}
\usepackage{amssymb}
\usepackage{caption}
\usepackage{tabularx}
\usepackage[table]{xcolor}
 \usepackage[normalem]{ulem}
\usepackage{algorithm}
\usepackage[noend]{algpseudocode}

\usepackage[section]{placeins}
\makeatletter
\def\BState{\State\hskip-\ALG@thistlm}
\makeatother
\definecolor{myGreen}{HTML}{33FF00}
\definecolor{myRed}{HTML}{FF3030}
\definecolor{myGrey}{HTML}{AA5555}
\definecolor{myWhite}{HTML}{FFFFFF}
\definecolor{maroon}{cmyk}{0,0.87,0.68,0.32}
\definecolor{petr}{HTML}{5555FF}
\definecolor{josef}{HTML}{FF3030}

\usepackage{blindtext}
\usepackage{booktabs}
\usepackage{url}
\usepackage {multirow}
\usepackage{hyperref}
\usepackage[font=footnotesize,skip=3pt,subrefformat=parens]{subcaption}

\def\I{\mathbf{I}}
\def\R{\mathbf{R}}
\def\L{\mathbf{L}}
\def\r{\mathbf{r}}

\def\a{\mathbf{a}}
\def\b{\mathbf{b}}
\def\bb{\mathbf{b}}
\def\cc{\mathbf{c}}
\def\zerom{\boldsymbol 0}
\def\onem{\boldsymbol 1}
\def\alpham{\boldsymbol \alpha}
\def\betam{\boldsymbol \beta}
\def\thetam{\boldsymbol \theta}
\def\epsilonm{\boldsymbol \epsilon}
\def\deltam{\boldsymbol \delta}
\def\prodm{\odot}

\newcommand{\colorRevision}{\color{black}}

\journalname{IJCV}
\usepackage{natbib}
\setcitestyle{authoryear,round,semicolon}

\begin{document}
\begin{sloppypar}
\title{DI-Retinex: Digital-Imaging Retinex Theory for Low-Light Image Enhancement
}


\author{Shangquan Sun         \and
        Wenqi Ren          \and
        Jingyang Peng   \and
        Fenglong Song \and
        Xiaochun Cao   
}


\institute{Shangquan Sun \at
            Institute of Information Engineering, Chinese Academy of Sciences, Beijing 100085, China.\\
            The School of Cyber Security, University of Chinese Academy of Sciences, Beijing 100049, China.\\
              \href{sunshangquan@iie.ac.cn}{sunshangquan@iie.ac.cn}           
           \and
           Wenqi Ren \at
           The School of Cyber Science and Technology, Shenzhen Campus of Sun Yat-sen University, Shenzhen 518107, China.\\
              \href{renwq3@mail.sysu.edu.cn}{renwq3@mail.sysu.edu.cn}
            \and
            Jingyang Peng \at
            Huawei Noah’s Ark Lab. \\
            \href{pengjingyang1@huawei.com}{pengjingyang1@huawei.com}
            \and
            Fenglong Song \at
            Huawei Noah’s Ark Lab. \\
            \href{songfenglong@huawei.com}{songfenglong@huawei.com}
            \and
            Xiaochun Cao \at 
            The School of Cyber Science and Technology, Shenzhen Campus of Sun Yat-sen University, Shenzhen 518107, China.\\
            \href{caoxiaochun@mail.sysu.edu.cn}{caoxiaochun@mail.sysu.edu.cn}
}   
\vspace{-5mm}
\date{Received: date / Accepted: date}

\maketitle

\begin{abstract}
   Many existing methods for low-light image enhancement (LLIE) based on Retinex theory ignore important factors that affect the validity of this theory in digital imaging, such as noise, quantization error, non-linearity, and dynamic range overflow.
    In this paper, we propose a new expression called Digital-Imaging Retinex theory (DI-Retinex) through theoretical and experimental analysis of Retinex theory in digital imaging. 
    %
    %
    Our new expression includes an offset term in the enhancement model, which allows for pixel-wise brightness contrast adjustment with a non-linear mapping function.
    In addition, to solve the low-light enhancement problem in an unsupervised manner, we propose an image-adaptive masked reverse degradation loss in Gamma space.
    We also design a variance suppression loss for regulating the additional offset term. 
    Extensive experiments show that our proposed method outperforms all existing unsupervised methods in terms of visual quality, model size, and speed. Our algorithm can also assist downstream face detectors in low-light, as it shows the most performance gain after the low-light enhancement compared to other methods.
\end{abstract}
\section{Introduction}

In low-light conditions, digital imaging often suffers from a number of degradations, such as low visibility and high levels of noise. While professional equipment and techniques such as long exposure, fill light, and large-aperture lenses can help alleviate these issues, many amateur photographers still struggle with low-light images. Therefore, there is a need for a portable and efficient low-light image enhancement (LLIE) algorithm that can effectively restore underexposed images.

Some fundamental works, e.g., Retinex theory~\citep{land1971lightness,land1977retinex}, have attempted to explain the relation among scene radiance reaching human eyes, material reflectance, and source illumination. 
Many LLIE algorithms and models have been designed based on this theory. However, it is important to note that the original theory was developed for radiance reaching the retina, rather than for measuring imaging intensity in computer vision. Some LLIE works have recognized the issue of amplified noise after enhancement, but there are other important factors, such as quantization error, non-linearity, and dynamic range overflow, that have not been thoroughly discussed. Simply applying Retinex theory from optics to computer vision can lead to imprecision and incompleteness in LLIE.

\begin{figure*}[hpt]
\small
  \centering
  \begin{subfigure}{0.2465\linewidth}
    \includegraphics[width=1\linewidth]{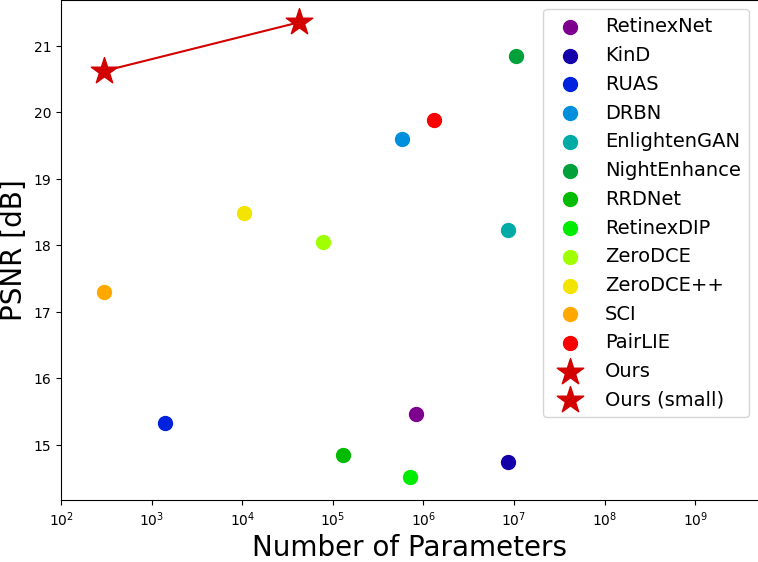}
      \caption{PSNR vs. model size on LOL-v1}
      \label{fig:space_lolv1}
  \end{subfigure}
  \hfill
  \begin{subfigure}{0.2465\linewidth}
    \includegraphics[width=1\linewidth]{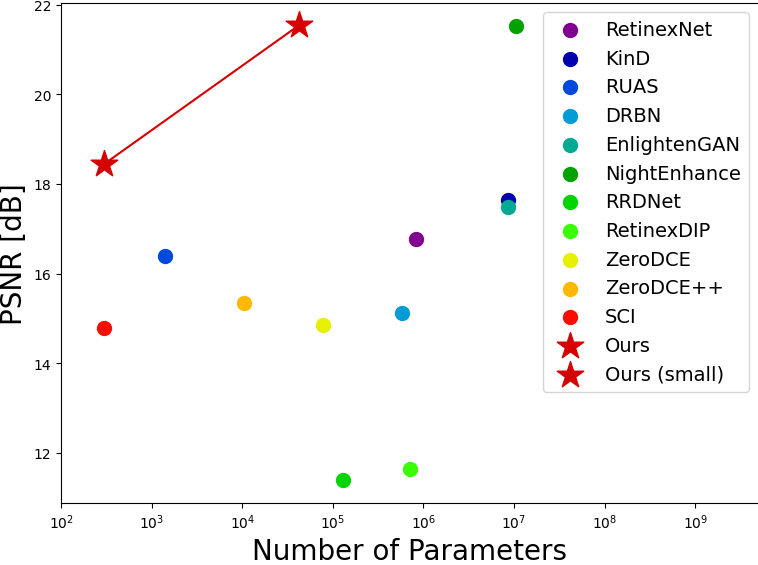}
      \caption{PSNR vs. model size on LOL-v2}
      \label{fig:space_lolv2}
  \end{subfigure}
  \hfill
  \begin{subfigure}{0.2465\linewidth}
    \includegraphics[width=1\linewidth]{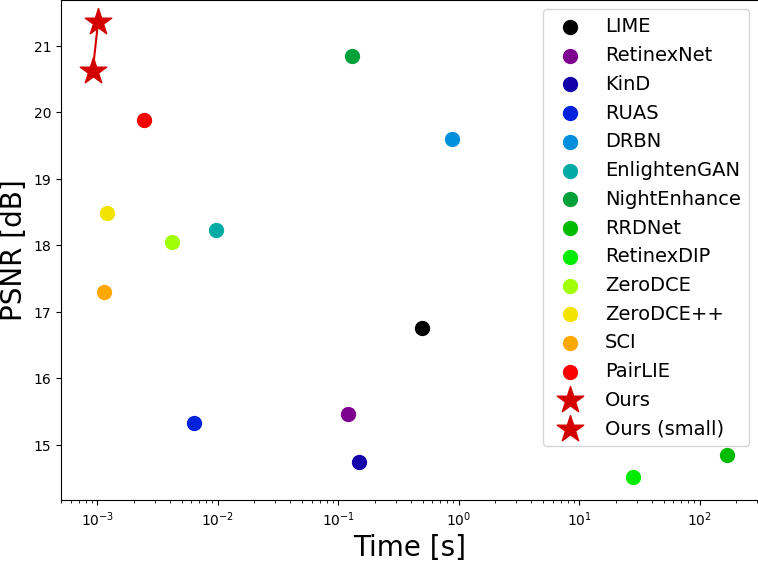}
      \caption{PSNR vs. run time on LOL-v1}
      \label{fig:time_lolv1}
  \end{subfigure}
  \hfill
  \begin{subfigure}{0.2465\linewidth}
    \includegraphics[width=1\linewidth]{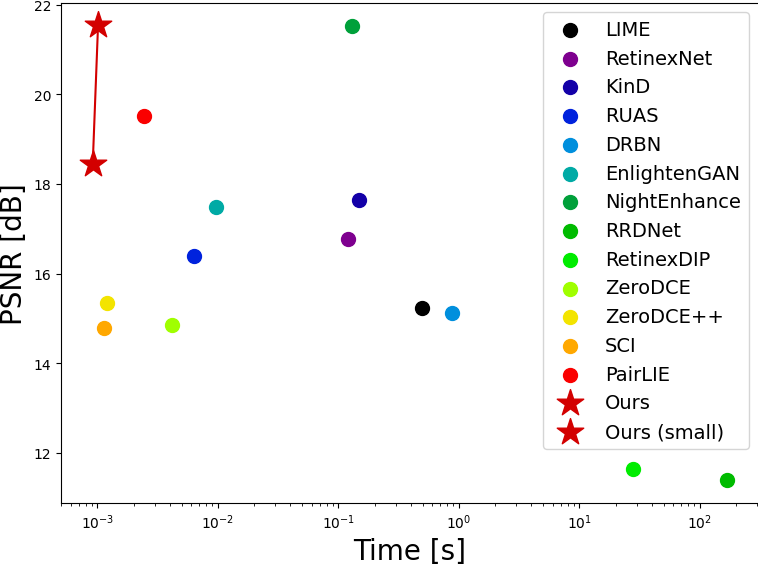}
      \caption{PSNR vs. run time on LOL-v2}
      \label{fig:time_lolv2}
  \end{subfigure}
\setlength{\abovecaptionskip}{-2pt} 
\setlength{\belowcaptionskip}{-2pt}
  \caption{The trade-off between performance, and inference time and model size. Our method can achieve the best low-light enhancement performance with the smallest parameter number or at the fastest speed consistently.}
  \label{fig::space_time_perf}
  \vspace{-0.2cm}
\end{figure*}
In this study, we conduct a theoretical analysis of the various factors that can contribute to noise, quantization error, non-linearity, and dynamic range compression when applying Retinex theory to digital imaging. We propose an extended version of Retinex theory specifically designed for computer vision called Digital-Imaging Retinex theory (DI-Retinex). 
Through our analysis, we demonstrate the existence of an offset with a non-zero mean and an amplified variance that can be caused by these factors when applying DI-Retinex theory to low-light image enhancement.

Building on the observation and analysis of the DI-Retinex theory, we design a brightness contrast adjustment algorithm that addresses the LLIE problem using a lightweight network consisting of three convolutional layers.
The network predicts the contrast and brightness coefficients in the brightness contrast adjustment function and enhances degraded images by plugging predictions in the function.
To guide the network towards generating optimal coefficient pairs, we use a masked reverse degradation loss and a variance suppression loss, which enable the network to learn in a zero-shot manner\footnote{Zero-shot learning in low-level vision tasks represents the method requiring neither paired or unpaired training data, in contrast to its concept in high-level visual tasks~\citep{li2021low,zhang2019zero}.}, without requiring paired or unpaired training data.

To sum up, our contributions can be summarized as:

\begin{itemize}
    \item We analyze Retinex theory from the perspective of optics and computational photography and discuss the variation when transferring it to digital imaging. In addition to noise, Retinex theory should also involve quantization error, non-linearity, and dynamic range overflow in its formulation. As such, we propose a new DI-Retinex theory specifically for LLIE.
    
    \item We present a contrast brightness function to solve the LLIE problem derived from DI-Retinex theory. In addition, we propose a masked reverse degradation loss with Gamma encoding and a variance suppression loss to guide the network in a zero-shot learning manner. 
    
    \item Extensive experiments demonstrate that the proposed method outperforms all the zero-shot and unsupervised learning-based LLIE methods in terms of visual quality, objective metrics, model size and speed. Moreover, our approach can be used as a preprocessing step for downstream tasks such as detection.
\end{itemize}

\section{Related Works}

\subsection{Classic Retinex Theory}

Land and McCann~\citep{land1971lightness,land1977retinex} conducted a series of photometrical experiments, showing the scene radiance reaching human eye is the product of an intrinsic reflectance and incident illumination, which is named as Retinex Theory. 
It is formally written as follows,
\begin{equation}\label{eq::retinex}
    \I^e = \L \prodm \R,
\end{equation}
\noindent where $\prodm$ denotes Hadamard product, $\I^e$ means radiance reaching human eye, $\L$ is illuminance, and $\R$ is reflectance. For images under different exposure conditions but with an identical scene, reflectance $R$ keeps unchanged since it is determined by and only related to the intrinsic material property of object surface. This also implies that sensation of color is mainly dependent on reflectance. 






Many model-based methods employ Retinex theory for better modeling~\citep{wang2013naturalness,wang2014variational,fu2014novel,fu2016weighted,guo2016lime,fu2016fusion,cai2017joint,fu2018retinex,li2018structure,ren2018joint,xu2020star}. In recent years, deep learning with powerful learning ability and inference speed dominates the field of LLIE~\citep{Chen2018Retinex,zhang2019kindling,wang2019underexposed,wang2019progressive,fan2020integrating,yang2021sparse,li2018lightennet,liu2021retinex,ma2022toward}. Nearly one-third of deep learning methods also adopt Retinex theory for the sake of better enhancement effect and physical explanation~\citep{li2021low}. Therefore, a theoretical analysis of Retinex theory in digital photography is vital for establishing a valid physical model.

Note that the classic Retinex Theory has an assumption of proper exposure and the experiments of Land and McCann~\citep{land1971lightness,land1977retinex} with a telescopic photometer conclude that Eq.~\ref{eq::retinex} holds for ``eye'' rather than camera. Simply borrowing Eq.~\ref{eq::retinex} for solving LLIE problem may be inaccurate and incomplete due to violation of the assumption and condition.

Many existing methods~\citep{jobson1997properties,jobson1997multiscale,ma2022toward,wang2019underexposed,liu2021retinex,ying2017new} presumably regard the reflectance component of Retinex theory as the final enhanced result, which is not what Retinex theory originally means and affects the final performance~\citep{li2021low}. 
Other methods~\citep{li2018structure,zhang2019kindling,ren2018joint} notice the existence of amplified noise. But 
they neglect other involved errors, e.g., quantization error and dynamic range compression, when transferring Retinex theory from optics to digital imaging.





\subsection{Model-based LLIE Methods}

The earlier LLIE methods use histogram equalization (HE)~\citep{pizer1990contrast,abdullah2007dynamic} in global level~\citep{coltuc2006exact,ibrahim2007brightness} and local region~\citep{lee2013contrast,stark2000adaptive}. Other methods utilize Retinex-theory variants including classic Retinex theory~\citep{fu2016weighted,guo2016lime,wang2014variational,cai2017joint,fu2016fusion,xu2020star}, single-scale Retinex~\citep{jobson1997properties}, multi-scale Retinex~\citep{jobson1997multiscale}, adaptive multi-scale Retinex~\citep{lee2013adaptive}, Naturalness Retinex theory~\citep{wang2013naturalness}, Robust Retinex theory~\citep{li2018structure,ren2018joint}, for designing physically explicable algorithms where degraded image is decomposed into reflectance and illuminance during iterative optimization. There are also some methods~\citep{dong2010fast,li2015low} transforming LLIE problem into a dehazing problem. The S-curve in photography can also be used to enhance underexposed image~\citep{yuan2012automatic}. 


\subsection{Data-Drive LLIE Methods}

\noindent\textbf{Supervised Learning.}
{\colorRevision
With the developing of deep learning, the earliest deep learning-based LLIE methods focus on learning the mapping from low-light image to normal exposed one on paired training set.
Various network architectures are developed for better enhancement, including a stacked-sparse denoising autoencoder in LLNet~\citep{lore2017llnet}, U-Net~\citep{chen2018learning}, a multi-branch network named MBLLEN~\citep{lv2018mbllen}, Retinex-Net composed of a Decom-net and an Enhance-net~\citep{Chen2018Retinex}, a lightweight LightenNet~\citep{li2018lightennet}, a frequency decomposition network~\citep{cai2018learning}, a decoder network extracting global and local features~\citep{wang2019underexposed}, siamese network~\citep{chen2019seeing}, 3D U-Net~\citep{jiang2019learning,zhang2021learning}, progressive Retinex network~\citep{wang2019progressive}, three-branch subnetworks~\citep{zhang2019kindling,yang2021sparse}, a complex Encoder-Decoder~\citep{ren2019low}, a frequency-based decomposition-and-enhancement model~\citep{xu2020learning}, a model integrating semantic segmentation~\citep{fan2020integrating}, an edge-enhanced multi-exposure fusion network~\citep{zhu2020eemefn}, residual network~\citep{wang2020lightening}, a Retinex-inspired unrolling model by architecture search~\citep{liu2021retinex,liu2022learning}, Signal-to-Noise-Ratio (SNR)-aware transformer~\citep{xu2022snr}, a deep color consistent network~\citep{zhang2022deep}, a deep unfolding Retinex network~\citep{wu2022uretinex}, a normalizing flow model~\citep{wang2022low}, Gamma correction model~\cite{Wang2023Low}, contrastive self-distillation~\citep{Fu2023you}, segmentation assisting model~\citep{Wu2023learning}, structure-aware generator~\citep{Xu2023low}, vision Transformers~\citep{Cai2023Retinexformer} and diffusion models~\citep{Wang2023ExposureDiffusion,Yi2023diffretinex}. The above-mentioned architectures tend to be complicated and a large scale of paired training set with low/normal exposed images are necessary for training the networks. Large model size impedes fast inference speed. Besides, collecting large paired datasets are time-consuming and expansive in terms of human labours.
}

\noindent\textbf{Semi-Supervised Learning.}
Several semi-supervised learning (SSL) methods~\citep{yang2020fidelity,yang2021band} use a deep recursive band network (DRBN) to learn a linear band representation between the pair of low-light image and corresponding ground-truth. For unpaired data, a perceptual quality-driven adversarial learning is adopted for guiding enhancement. The requirement of paired data still restricts the flexibility of the SSL methods.

\noindent\textbf{Unpaired-supervised Learning.}
To alleviate the problem of lacking paired training data, unpaired supervised learning (UL) requires a set of normal exposed images unpaired with low-light images. EnlightenGAN~\citep{jiang2021enlightengan} employs adversarial learning with a local-global discriminator structure and a perceptual loss of self-regularization. NightEnhance~\citep{jin2022unsupervised} enhances night image with dark regions and light effects by a layer decomposition network and a light-effects suppression network with the guidance of multiple unsupervised layer-specific prior losses. However, in many common scenarios unpaired normally exposed images are still lacking or of low quality. The enhancement quality of UL methods highly relies on the quality of normally exposed images.

\noindent\textbf{Zero-Shot Learning.}
To fully get rid of the limitation of paired data, zero-shot learning only needs low-light images for training.
ExCNet~\citep{zhang2019zero} is composed of a neural network to estimate the S-curve to adjust luminance and map from back-lit input to enhanced output. The network is trained by minimizing a block-based energy expression as loss function. The idea of S-curve is adopted in Zero-DCE~\citep{guo2020zero} and Zero-DCE++~\citep{li2021learning} where a network estimating intermediate light-enhancement curve is used to restore input image iteratively. They design four sophisticated loss terms to regulate spatial consistency, exposure level, color constancy, and illumination smoothness. 
RRDNet~\citep{zhu2020zero} solves underexposed problem by decomposing reflectance, illumination and noise components by a three-branch network with three loss terms aiming at Retinex reconstruction, texture enhancement and illumination-guided noise estimation. 
RetinexDIP~\citep{zhao2021retinexdip} combines Deep Image Prior and Retinex decomposition and is guided by a group of component characteristics-related losses.
SCI~\citep{ma2022toward} incorporates a self-calibrated illumination module for the acceleration for training, a cooperative fidelity loss term and a smoothness loss term to achieve enhancement effect intrinsically. 
PairLIE~\citep{fu2023pairlie} leverages the pair of low-light images to train a self-learnining network based on Retinex theory.
They are facing an unsatisfactory enhancement performance due to either delicate loss terms or the incompleteness of classic Retinex theory in LLIE task.

\section{Analysis of Retinex Theory in Low-Light}\label{sec::proof}

\begin{figure}[tp]

  \centering
  \begin{subfigure}{0.405\linewidth}
    \includegraphics[width=1\linewidth]{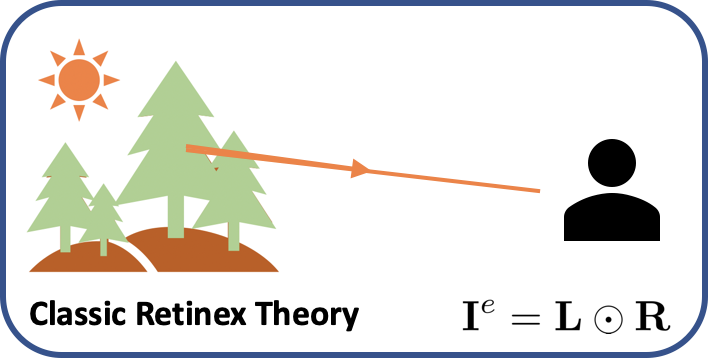}
      \caption{{\scriptsize Diagram of classic Retinex.}}
      \label{fig:fig1a}
  \end{subfigure}
  \hfill
  \begin{subfigure}{0.585\linewidth}
    \includegraphics[width=1\linewidth]{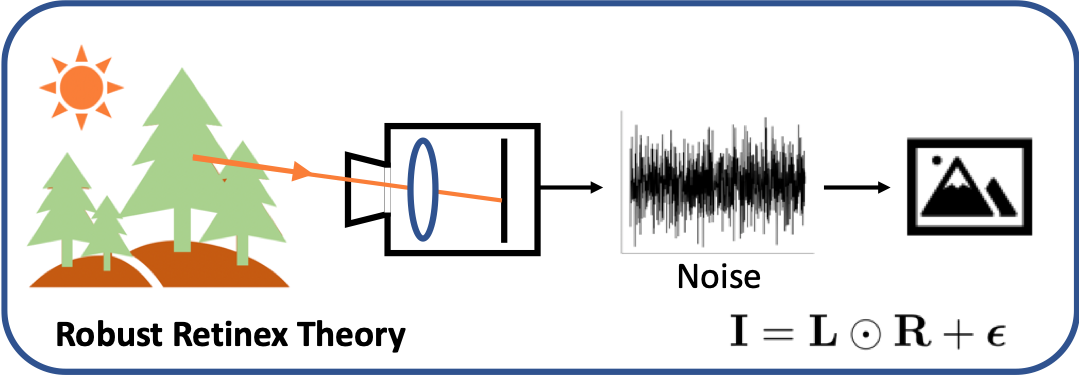}
      \caption{{\scriptsize Diagram of robust Retinex theory~\citep{li2018structure}.}}
      \label{fig:fig1b}
  \end{subfigure}
  
  \centering
  \begin{subfigure}{1\linewidth}
    \includegraphics[width=1\linewidth]{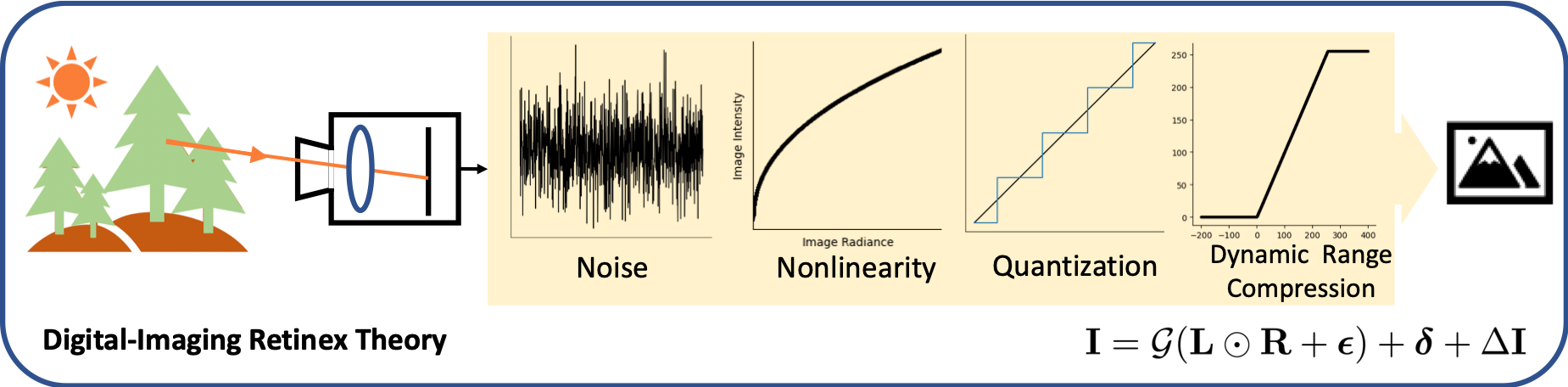}
      \caption{The diagram of our Digital-Imaging Retinex theory.}
      \label{fig:fig1c}
  \end{subfigure}
\setlength{\abovecaptionskip}{-2pt} 
\setlength{\belowcaptionskip}{-2pt}
  \caption{The classic Retinex theory is targeted for scene radiance directly incident into human eye as shown in \ref{fig:fig1a}. Some methods, e.g., \citep{li2018structure,zhang2019kindling} consider the noise as shown in \ref{fig:fig1b}. However, many other factors including quantization error, non-linearity due to camera response and dynamic range overflow due to physical limitation, are ignored. Our DI-Retinex theory in \ref{fig:fig1c} takes all of them into consideration.}
  \label{fig:DI-Retinex}
\end{figure}

In this section, we first show the existence and cause of four factors influencing the expression of the classic Retinex theory, i.e., noise, quantization error, non-linearity mapping, and dynamic range overflow. Then a novel expression of DI-Retinex theory and its corresponding enhancement model are given based on the analysis.

\subsection{Imaging Noise}\label{sec::noise}
When scene radiance reaches imaging devices, various sources of noise consisting of read noise~\citep{janesick2007photon,mackay2001subelectron}, dark current noise~\citep{moon2007application,healey1994radiometric}, photon shot noise~\citep{schottky1918spontane,blanter2000shot}, source follower noise~\citep{antal20011} etc., are reported and should be taken into account~\citep{haus2000electromagnetic}. Therefore, Retinex theory is rewritten by taking noise into account, i.e.,
\begin{equation}
    \I = \L \prodm \R + \epsilonm,
\end{equation}
where $\epsilonm$ is the combination of all significant noise sources during digital imaging. 
Among them, photon shot and dark current noise are theoretically modeled by Poisson distribution with parameters equal to ideal incoming photon and generated electrons respectively~\citep{blanter2000shot,moon2007application}. Though read noise, sometimes modeled by a Gaussian distribution, can be approximated by a Poisson distribution as well~\citep{alter2006intensity}.
In this work, we denote all noise sources as a whole by $\epsilonm$.



\subsection{Non-linearity.}
In a digital camera, an analog-to-digital converter (ADC) is deployed to convert voltage value with a range of $[0,+\infty]$ to pixel intensity with a restricted range $[0,255]$ in RGB color space. A non-linear response function achieves the conversion, compensates the difference of human adaptation in dark and bright regions according to Weber's Law~\citep{weber1831pulsu}, and suppresses representation space.
The non-linear function can be commonly modeled by a Gamma transformation or Gamma encoding as follows~\citep{mann1994beingundigital,mitsunaga1999radiometric}:
\begin{equation}
    \mathcal{G}(\I) = \mu + \lambda \I^\gamma,
\end{equation}
where $\gamma$ is a coefficient set to nearly $0.5$ empirically, e.g., $1/2.2$, and $\mu, \lambda$ are two camera-specific parameters measuring bias and scale respectively.

\subsection{Quantization error} 
Besides introducing non-linearity, the ADC converts a continuous voltage signal to a discrete $N$-bit value. For example, in RGB ($N=8$) and RAW ($N=16$) color space, there are $2^8=256$ and $2^{16}=65536$ possible pixel values for each channel, respectively. However, the conversion from continuous space to discrete one yields a quantization error that can be expressed by a uniform distribution:
\begin{equation}
    \deltam_{i, j} \sim \mathcal{U}\left(-\frac{q}{2}, \frac{q}{2}\right),
\end{equation}
where $\deltam_{i, j}$ is the $(i, j)$-th item of a quantization error $\deltam$, $\frac{q}{2}$ is the half of least significant bit $q$. The discrete intensity can be formulated by continuous intensity plus a quantization noise, i.e., $\I_{disc.} = \I_{cont.} + \deltam$.

\subsection{Dynamic Range Overflow}

Another problem when transferring Retinex theory from human eyes to digital imaging is dynamic range limitation. Human eyes are known to have a significantly large dynamic range compared to vision sensors~\citep{banterleadvanced}. Therefore, scene dynamic range is presumably within the perceptual dynamic range of human eyes in the classic Retinex theory. However, a considerable scene dynamic range may exceed the physical tolerance of imaging devices. Therefore, a clamp function $\mathcal{C}_0^k$ is defined as Eq.~\ref{eq::clampFunc} to clip all exceeding values.
\begin{equation}\label{eq::clampFunc}
    \mathcal{C}_0^k (x) = \left\{
    \begin{aligned} 
        &0, \text{$x < 0$} \\
        &k, \text{$x > k$} 
    \end{aligned}
    \right.,
\end{equation}
where $k$ is the max pixel value, e.g., $k=255$ in RGB space. By defining a masked offset matrix $\Delta\I$, clamp function can also be expressed as follows,
\begin{equation}
    \mathcal{C}_0^k \left(\I\right) = \I + \Delta \I 
\end{equation}
\noindent where $\Delta\I$ has zero items at the pixels whose intensity $\I_{i,j}\in[0,k]$ and non-zero values $\Delta\I_{i,j} = k - \I_{i,j}$ with $\I_{i,j} > k$ or $\Delta\I_{i,j} = -\I_{i,j}$ with $\I_{i,j} < 0$.

\subsection{Digital-Imaging Retinex Theory}

Combining all aforementioned factors, we reformulate the classic Retinex theory as follows:
\begin{equation}\label{newRetinexWOclamp}
    \I = \mathcal{G}(\L \prodm \R + \epsilonm) + \deltam + \Delta \I ,
\end{equation}
The image irradiance perceived by the imaging device is ideal scene radiance $\L\prodm\R$ plus a noise term. Then a camera response function $\mathcal{G}$ is applied to the noisy irradiance. Finally, a quantization noise and a masked offset due to dynamic range compression are added. A comparison of the diagrams of our DI-Retinex theory and previous Retinex theories is shown in Fig.~\ref{fig:DI-Retinex}. Despite the complicated expression, it can be utilized to derive an efficient LLIE model.

\begin{figure}[hpt]
\small
  \centering
  \begin{minipage}[b]{.30\linewidth}
    \centering
  \begin{subfigure}{1.0\linewidth}
    \includegraphics[width=1\linewidth]{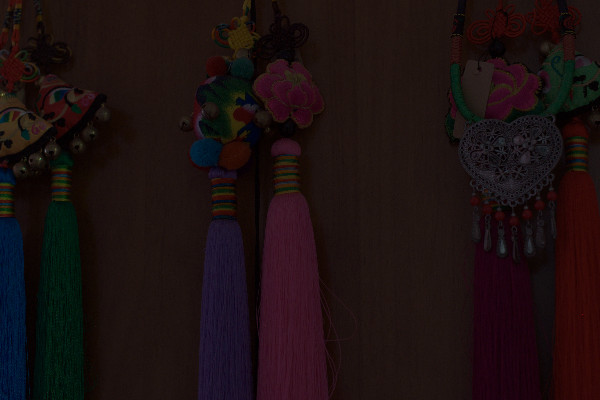}
      \caption{\(\I_l\) example}
      \label{fig:plot_joint_hist_lite1}
  \end{subfigure}
    \hfill
  \begin{subfigure}{1.0\linewidth}
    \includegraphics[width=1\linewidth]{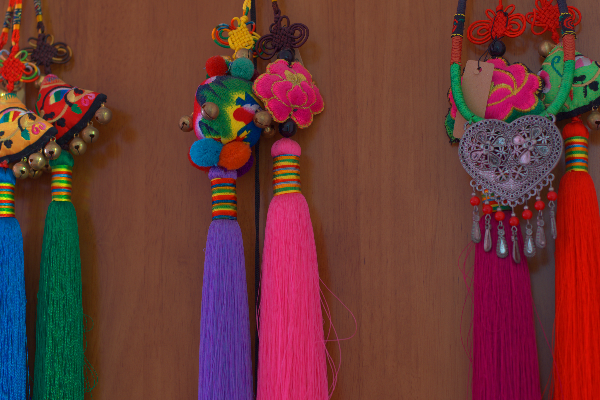}
      \caption{\(\I_h\) example}
      \label{fig:plot_joint_hist_lite2}
  \end{subfigure}
\end{minipage} 
  \hfill
  \begin{subfigure}{0.69\linewidth}
    \includegraphics[width=1\linewidth]{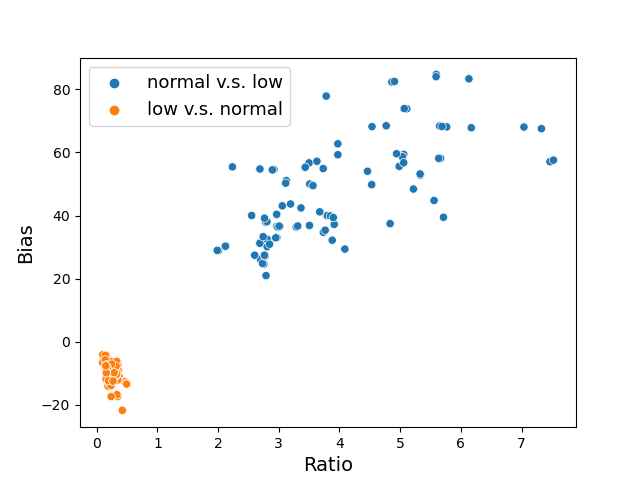}
      \caption{{\colorRevision Scatter plot of ratios and biases}}
      \label{fig:plot_joint_hist_lite3}
  \end{subfigure}
  
  \centering
  \begin{subfigure}{0.495\linewidth}
    \includegraphics[width=1\linewidth]{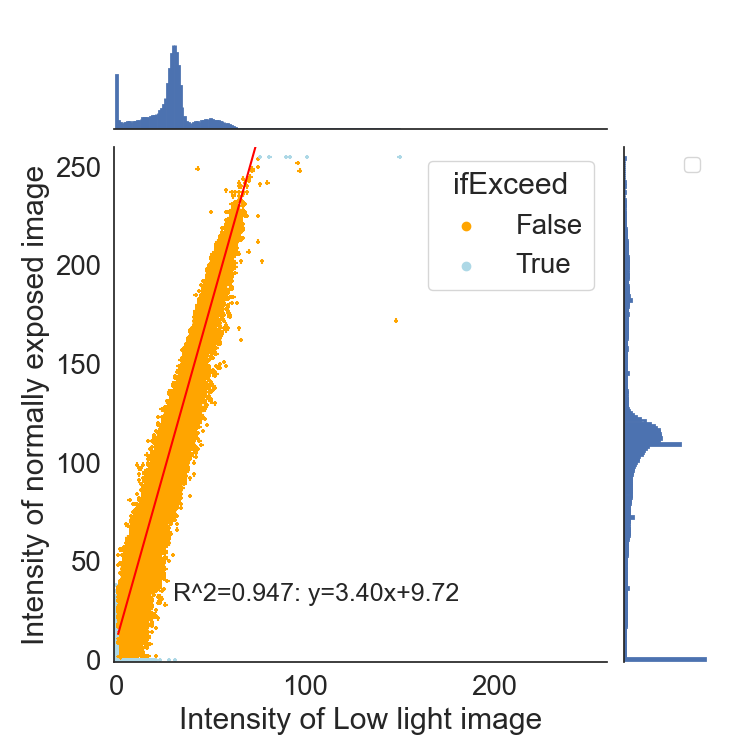}
      \caption{Joint histogram of $\I_l$ vs. $\I_h$.}
      \label{fig:plot_joint_hist_lite4}
  \end{subfigure}
  \hfill
  \begin{subfigure}{0.495\linewidth}
    \includegraphics[width=1\linewidth]{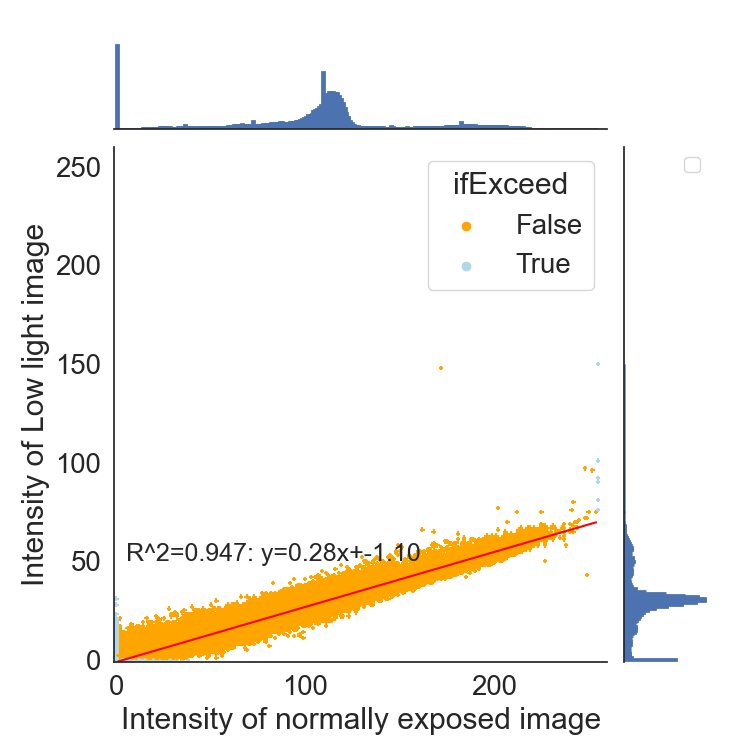}
      \caption{Joint histogram of $\I_h$ vs. $\I_l$}
      \label{fig:plot_joint_hist_lite5}
  \end{subfigure}
\setlength{\abovecaptionskip}{-2pt} 
\setlength{\belowcaptionskip}{-2pt}
  \caption{{\colorRevision The statistical experiments for showing the existence of $\betam$ with a non-neglegible magnitude and the existence of $\betam'$ with a small magnitude.}}
  \label{fig::plot_joint_hist_lite}
  \vspace{-0.2cm}
\end{figure}

\subsection{LLIE Model based on DI-Retinex}

The existing LLIE methods assuming $\R$ as a final enhanced image use an enhancement function $\I_h = \alpham \prodm \I_l $ with a network predicting $f(\I_l) = \alpham$. We extend its formulation by proving Theorem~\ref{theorem:1} involving an offset $\betam$.

\begin{theorem}\label{theorem:1}
Given an underexposed image $\I_l \in \mathbb{R}^{H\times W}$ and one of its possible corresponding properly-exposed images $\I_h \in \mathbb{R}^{H\times W}$, $\exists \alpham, \betam \in \mathbb{R}^{H\times W}$ such that the following relation holds,
\begin{equation}\label{relation}
    \I_h \approx \alpham \I_l + \betam.
\end{equation}
\end{theorem}

\begin{proof}
For a low-light and a normally exposed image, their expressions of DI-Retinex theory can be obtained
\begin{align}\label{eq::newlowhigh}
        \left\{
            \begin{aligned}
                \I_l &= \mathcal{G}(\L_l \prodm \R + \epsilonm_l) + \deltam_l + \Delta\I_l\\
                \I_h &= \mathcal{G}(\L_h \prodm \R + \epsilonm_h) + \deltam_h + \Delta\I_h.
            \end{aligned}
        \right. 
\end{align}
Note that $\R$ keeps unchanged because reflectance components are determined by the intrinsic physical characteristics of scene objects. The expression of $\I_h$ can be formulated as a function of $\I_l$ as,
\begin{align}\label{eq::newllie_model}
    &\I_h \approx \alpham \I_l + \betam, \\
    \text{where} &\left\{
    \begin{small}
    \begin{aligned}
        \alpham &= \left(\L_h \oslash \L_l\right)^\gamma  \\
        \betam &= \mu + \deltam_h + \Delta\I_h - \alpham \prodm \left(\mu + \deltam_l + \Delta\I_l\right),
    \end{aligned}
    \end{small}
    \right.
\end{align}
The notation of $\oslash$ means element-wise division between two matrices. The complete derivation of Eq.~\ref{eq::newllie_model} is attached in the appendix. Note that with the assumption of flux consistency along time dimension across spatial pixels during receiving radiance in a short period, $\alpham$ can be further treated as a constant matrix.
\end{proof}

Some may argue that the noise term $\betam$ in Eq.~\ref{eq::newllie_model} can be regarded as $\zerom$ due to the magnitude difference between the signal and the offset term $\betam$. However, we can show its mean is hardly zero, and a factor of $\alpham^2$ considerably amplifies its variance.
The detailed discussion is given in the appendix. 
We also conduct statistical experiments on a low light and normally exposed image pair from LOL-v1~\citep{Chen2018Retinex}. We plot the joint histogram of paired pixels as shown in Fig.~\ref{fig::plot_joint_hist_lite}. 
The dots in orange are those of $\I \in (0, k)$ and the dots in light blue are those of $\I = k$ or $\I = 0$  representing the pixels may have dynamic range overflow. 
A linear regression of $\I_h=\alpham \I_l + \betam$ is done for the dots $\I_h < k$. 
The regressed line is plotted in red and its coefficients are attached beside the line. 
The regression is fitted well by the linear model with a $R^2$ close to 1. 
From the statistical experiment, we find $\betam$ has a non-neglegible magnitude.

\section{Method}

Based on the extended DI-Retinex theory, we notice that directly regarding reflectance $\R$ as the final enhanced result is inappropriate. Similarly, only estimating a single matrix $\alpham$ such that the enhanced image $\tilde \I_h = \alpham \I_l$ is also inaccurate. This is because there exists an offset $\betam$ between the normal image $\I_h$ and proportionally enhanced image $\alpham \I_l$. Therefore, we adopt a contrast brightness adjustment function involving a flexible offset for efficient enhancement.

\subsection{Contrast Brightness Adjustment Function}

The algorithm commonly used in brightness adjustment~\citep{PILEnhancer} is as follows:
\begin{equation}\label{eq::brightness}
    f(\I;a) = a \I,
\end{equation}
where $a$ is a brightness adjustment coefficient ranging from 0 to infinity. When $a<1$, the image becomes dimmer, and when $a=0$ it turns entirely black. When $a=1$, the image keep unchanged. When $a>1$, the image becomes brighter. The algorithm proportionally enhances pixel intensity and thus may generate images with improper contrast.

The algorithm commonly used in contrast adjustment~\citep{PILEnhancer} is as follows:
\begin{equation}\label{eq::contrast}
\begin{split}
    f(\I;a) &= a \I + (1-a) \overline{\I} \\
    &= a (\I - \overline{\I}) + \overline{\I},
\end{split}
\end{equation}
where $a$ is a contrast adjustment coefficient ranging from 0 to infinity and $\overline{\I}$ is the average intensity of the input image. When $a$ equals $0$ or $1$, the image becomes a solid gray image or keeps unchanged, respectively. The idea is to enlarge the contrast range $\I - \overline{\I}$ of zero mean and then add the average intensity $\overline{\I}$ back to make sure the brightness of the image remains unchanged.

However, both Eq.~\ref{eq::brightness} and Eq.~\ref{eq::contrast} cannot adjust contrast and brightness simultaneously. Therefore, a contrast brightness adjustment function~\citep{GPUImageContrastFilter} is used as follows:
\begin{equation}\label{eq::contrastbrightness1}
\begin{split}
    f(\I;a,b) = a (\I - (1-b)\overline{\I}) + (1+b) \overline{\I},
\end{split}
\end{equation}
where $a$ and $b$ are the factors for contrast and brightness adjustment, respectively. A $b>0$ yields a brighter output and $b < 0$ gives a dimmer one. Compared to Eq.~\ref{eq::contrast}, a biased average $(1-b)\overline{\I}$ is subtracted from pixel intensity and later a reverse biased one $(1+b)\overline{\I}$ is added back. By Eq.~\ref{eq::contrastbrightness1}, the image can be adjusted in terms of contrast and brightness simultaneously. 

Neural networks are known to be good at predicting output with a restricted range, e.g., $[-1, 1]$, while $\a_{i, j} \in [0, +\infty)$ and the unknown range of $b$ make network difficult to predict them correctly. Therefore, we first make the $b$ restricted within $[-1,1]$ so that the network can learn to predict them better. Through substituting $\overline{\I}$ by the half of max pixel value $k/2$, Eq.~\ref{eq::contrastbrightness1} can be rewritten 
\begin{equation}\label{eq::contrastbrightness2}
\begin{split}
    f(\I;a,b) = a \left(\I - \frac{k}{2} (1-b)\right) + \frac{k}{2} (1+b) .
\end{split}
\end{equation}
When $a=1$, Eq.~\ref{eq::contrastbrightness2} is reduced to $f(\I; a,b)=\I + bk$, where $b=1$ turns all pixels to max intensity, namely pure white, and $b=-1$ makes all pixels non-positive, namely pure black. $b\in [-1,1]$ increases ($b>0$) or reduces ($b>0$) the brightness of image. And $a\neq 1$ enlarges ($a>1$) or reduces ($a<1$) the contrast of image.

The coefficient $a\in [0, +\infty)$ with range to infinity also hinders the accurate prediction of network. Therefore, a mapping function $g: [-1, 1] \rightarrow [0, +\infty)$ can be introduced. We have experimented several functions $g$, and empirically found the following one the best.
\begin{equation}\label{eq::mapping}
    \begin{split}
        a=g (c) = \tan\left( \frac{45 + (45 - \tau) c}{180} \pi\right)
    \end{split}
\end{equation}
where $\tau$ is a small number to avoid $a\rightarrow\infty$ causing error in program. Based on the function Eq.~\ref{eq::contrastbrightness2} and mapping function Eq.~\ref{eq::mapping}, we propose to use an enhancement network $\mathcal{H}$ consisting of solely three $3\times 3$ convolutional layers with parameters $\thetam$ to predict pixel-wise coefficients $\a$ and $\cc$ for efficient contrast brightness adjustment.
\begin{equation}\label{eq::contrastbrightness3}
\begin{split}
    \tilde \I_h &= \a \left(\I_l - \frac{k}{2}(1-\b)\right) + \frac{k}{2} (1+\b) \\
    &=\a \I_l + \frac{k}{2}(\a \b - \a + \b + 1), 
\end{split}
\end{equation}
where $[\b, \cc] = \mathcal{H} (\I_l; \thetam)$, and $\a = g(\cc)$. The expression is in line of the proportional relationship in Eq.~\ref{eq::newllie_model}. We design two losses for guiding the parameters $\a$ and $\b$ to learn the features of $\alpham$ and $\betam$.

\begin{figure*}[htp]
  \centering
  \begin{subfigure}{0.196\linewidth}
    \includegraphics[width=1\linewidth]{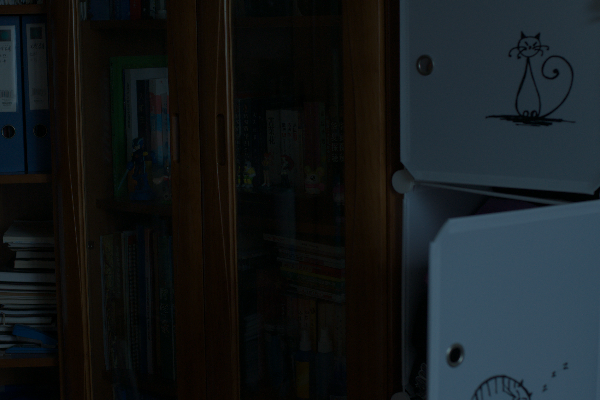}
      \caption{Input}
      \label{fig:lolv1-2-a}
  \end{subfigure}
  \hfill
  \begin{subfigure}{0.196\linewidth}
    \includegraphics[width=1\linewidth]{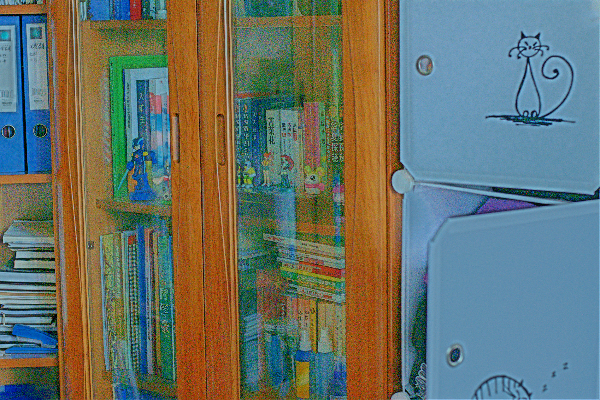}
      \caption{RetinexNet}
      \label{fig:lolv1-2-c}
  \end{subfigure}
  \hfill
  \begin{subfigure}{0.196\linewidth}
    \includegraphics[width=1\linewidth]{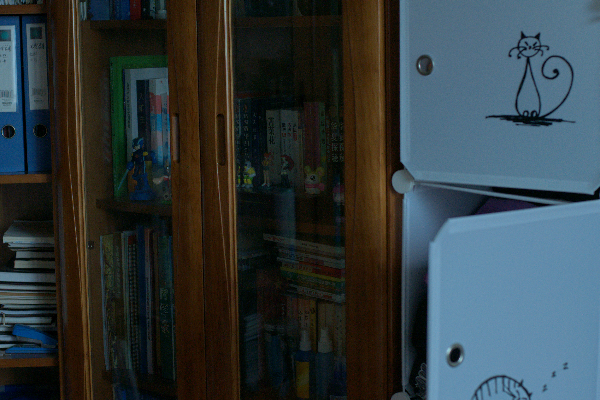}
      \caption{RRDNet}
      \label{fig:lolv1-2-b}
  \end{subfigure}
  \hfill
  \begin{subfigure}{0.196\linewidth}
    \includegraphics[width=1\linewidth]{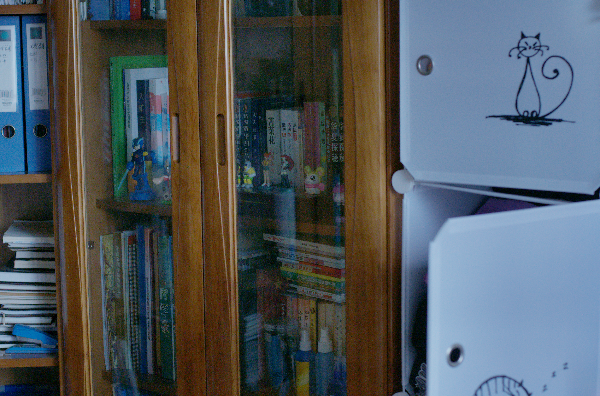}
      \caption{ZeroDCE++}
      \label{fig:lolv1-2-e}
  \end{subfigure}
  \hfill
  \begin{subfigure}{0.196\linewidth}
    \includegraphics[width=1\linewidth]{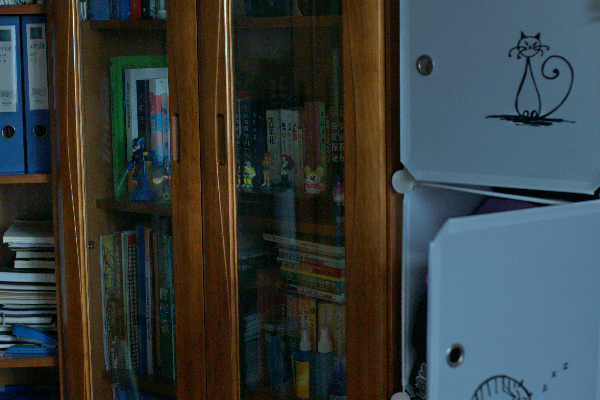}
      \caption{RetinexDIP}
      \label{fig:lolv1-2-f}
  \end{subfigure}
  
  \centering
  \begin{subfigure}{0.196\linewidth}
    \includegraphics[width=1\linewidth]{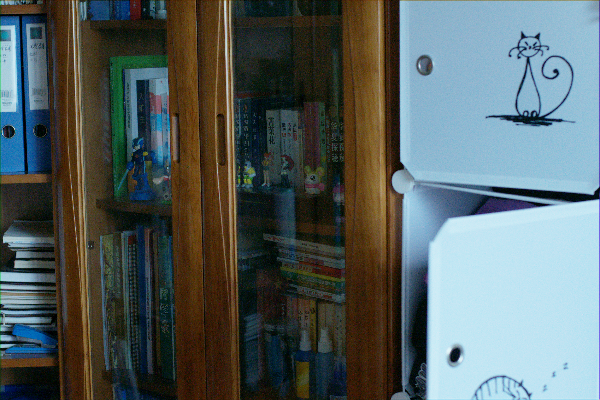}
      \caption{SCI}
      \label{fig:lolv1-2-g}
  \end{subfigure}
  \hfill
  \begin{subfigure}{0.196\linewidth}
    \includegraphics[width=1\linewidth]{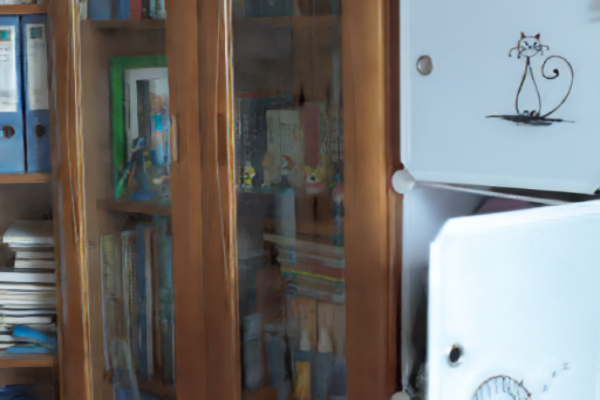}
      \caption{NightEnhance}
      \label{fig:lolv1-2-h}
  \end{subfigure}
  \hfill
  \begin{subfigure}{0.196\linewidth}
    \includegraphics[width=1\linewidth]{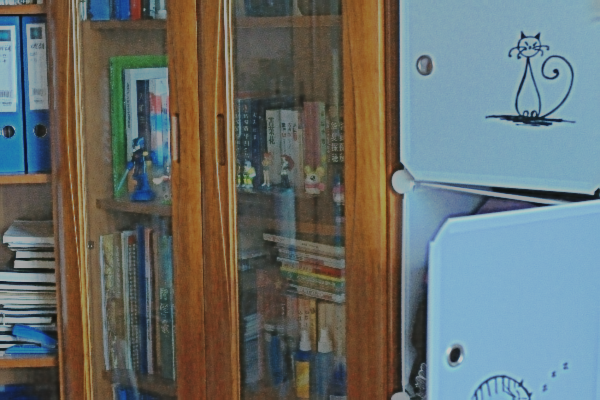}
      \caption{PairLIE}
      \label{fig:lolv1-2-d}
  \end{subfigure}
  \hfill
  \begin{subfigure}{0.196\linewidth}
    \includegraphics[width=1\linewidth]{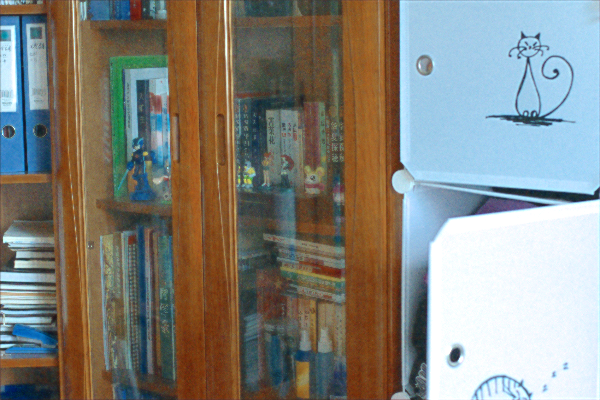}
      \caption{Ours}
      \label{fig:lolv1-2-i}
  \end{subfigure}
  \hfill
  \begin{subfigure}{0.196\linewidth}
    \includegraphics[width=1\linewidth]{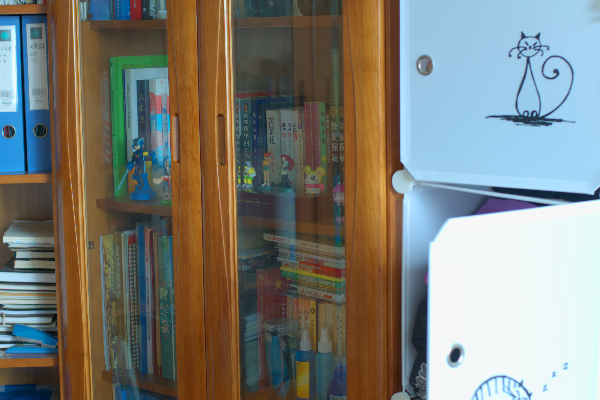}
      \caption{GT}
      \label{fig:lolv1-2-j}
  \end{subfigure}
\setlength{\abovecaptionskip}{-3pt} 
\setlength{\belowcaptionskip}{-3pt}
  \caption{{\colorRevision A visual comparison of enhancement results on LOL-v1. Please zoom in for better visualization.} }
  \label{fig:lolv1-2}
\end{figure*}
\begin{figure*}[htp]
  \centering
  \begin{subfigure}{0.196\linewidth}
    \includegraphics[width=1\linewidth]{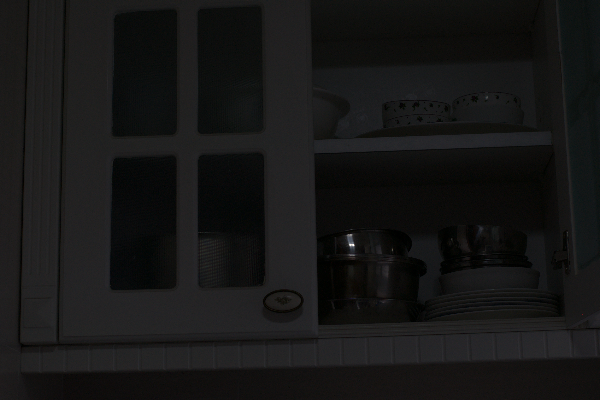}
      \caption{Input}
      \label{fig:lolv1-6-a}
  \end{subfigure}
  \hfill
  \begin{subfigure}{0.196\linewidth}
    \includegraphics[width=1\linewidth]{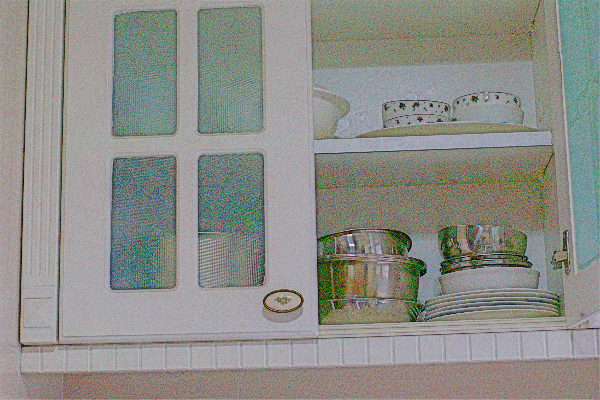}
      \caption{RetinexNet}
      \label{fig:lolv1-6-c}
  \end{subfigure}
  \hfill
  \begin{subfigure}{0.196\linewidth}
    \includegraphics[width=1\linewidth]{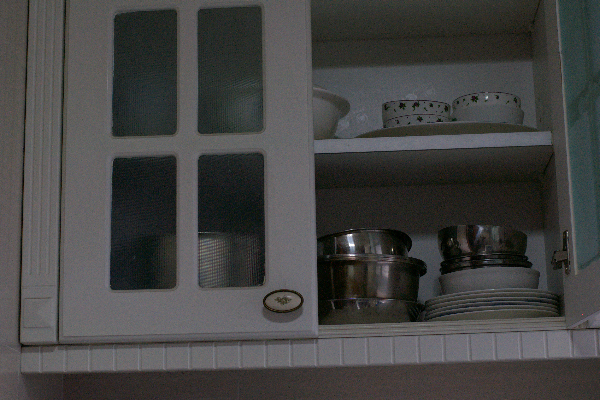}
      \caption{RRDNet}
      \label{fig:lolv1-6-d}
  \end{subfigure}
  \hfill
  \begin{subfigure}{0.196\linewidth}
    \includegraphics[width=1\linewidth]{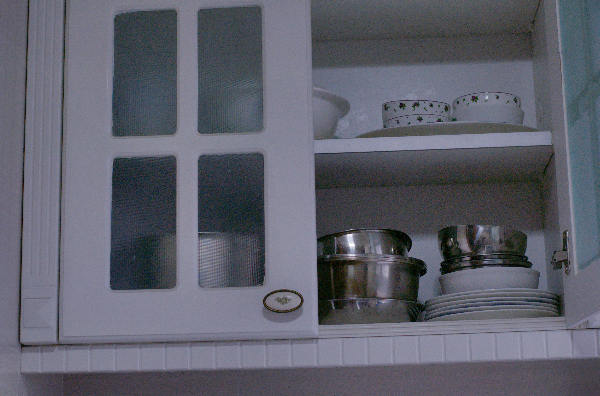}
      \caption{ZeroDCE++}
      \label{fig:lolv1-6-e}
  \end{subfigure}
  \hfill
  \begin{subfigure}{0.196\linewidth}
    \includegraphics[width=1\linewidth]{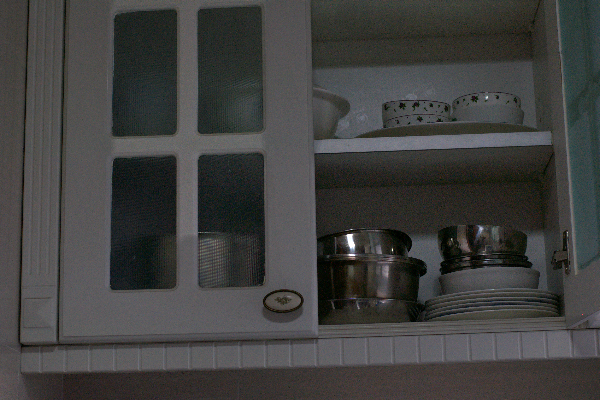}
      \caption{RetinexDIP}
      \label{fig:lolv1-6-f}
  \end{subfigure}
  
  \centering
  \begin{subfigure}{0.196\linewidth}
    \includegraphics[width=1\linewidth]{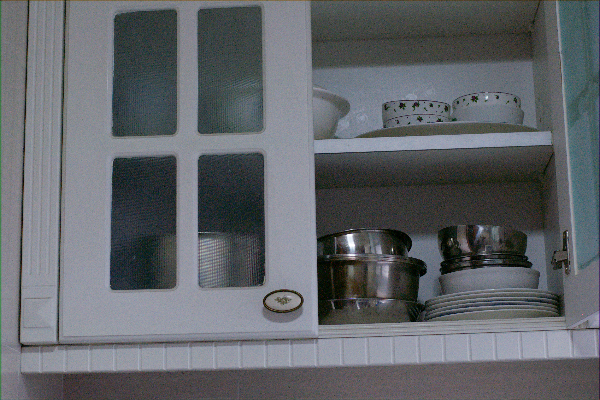}
      \caption{SCI}
      \label{fig:lolv1-6-g}
  \end{subfigure}
  \hfill
  \begin{subfigure}{0.196\linewidth}
    \includegraphics[width=1\linewidth]{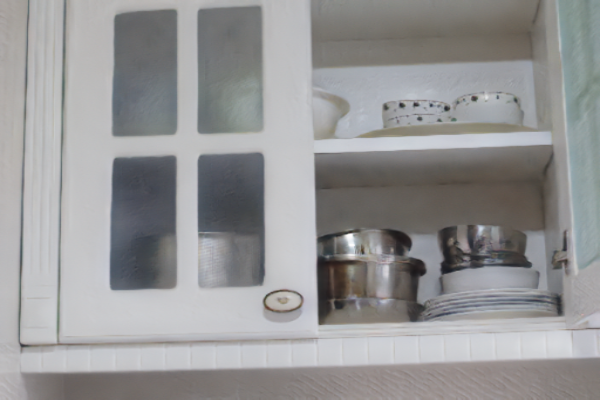}
      \caption{NightEnhance}
      \label{fig:lolv1-6-h}
  \end{subfigure}
  \hfill
  \begin{subfigure}{0.196\linewidth}
    \includegraphics[width=1\linewidth]{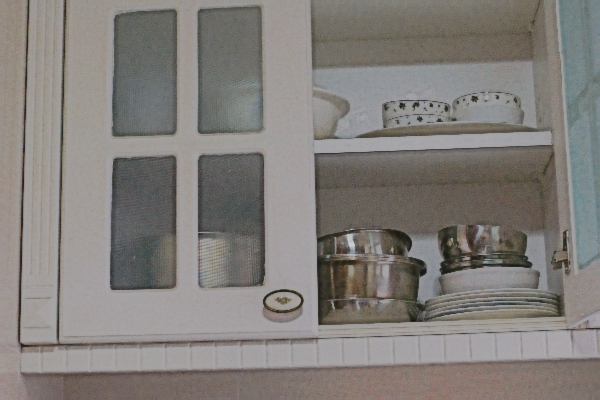}
      \caption{PairLIE}
      \label{fig:lolv1-6-b}
  \end{subfigure}
  \hfill
  \begin{subfigure}{0.196\linewidth}
    \includegraphics[width=1\linewidth]{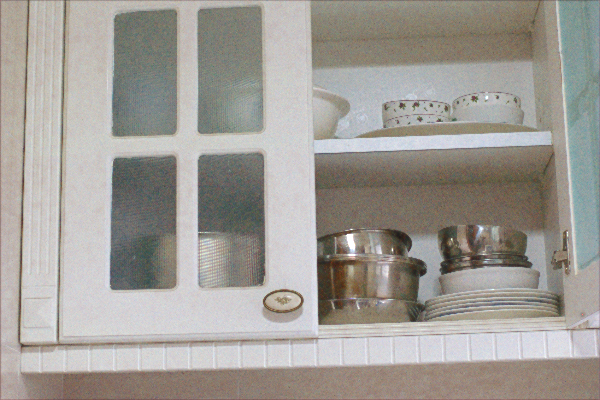}
      \caption{Ours}
      \label{fig:lolv1-6-i}
  \end{subfigure}
  \hfill
  \begin{subfigure}{0.196\linewidth}
    \includegraphics[width=1\linewidth]{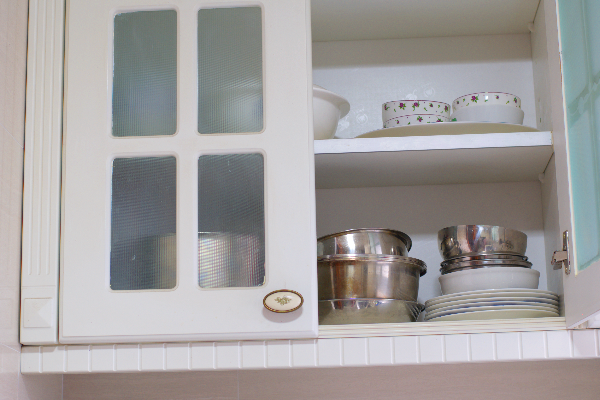}
      \caption{GT}
      \label{fig:lolv1-6-j}
  \end{subfigure}
\setlength{\abovecaptionskip}{-3pt} 
\setlength{\belowcaptionskip}{-3pt}
  \caption{{\colorRevision A visual comparison of enhancement results on LOL-v1. Please zoom in for better visualization.} }
  \label{fig:lolv1-6}
\end{figure*}

\subsection{Reverse Degradation Loss}

Though the enhancement based on Eq.~\ref{eq::newllie_model} has a non-negligible offset, the reverse degradation process transforming $\I_h$ to $\I_l$ can be approximated to a linear proportion. We can similarly obtain
\begin{align}\label{eq::reverseDegradation}
    &\I_l \approx \alpham' \I_h + \betam', \\
    \text{where} &\left\{
    \begin{small}
    \begin{aligned}
        \alpham' &= \left(\L_l \oslash \L_h \right)^\gamma\\
        \betam' &= \mu + \deltam_l + \Delta\I_l - \alpham' \prodm \left(\mu + \deltam_h + \Delta\I_h\right).
    \end{aligned}
    \end{small}
    \right.
\end{align}
We apply Gamma transformation to Eq.~\ref{eq::reverseDegradation} and can obtain $\I_l^{\frac{1}{\gamma}} \approx \left(\L_l \oslash \L_h \right) \I_h^{\frac{1}{\gamma}}$. The detailed discussion is attached in the appendix.
Therefore, we propose a masked reverse degradation loss as follows:
\begin{equation}\label{eq::reverseDegradationloss}
    \mathcal{L_{RD}} = \mathcal{M}\left(\tilde \I_h\right) \|\I_l^\frac{1}{\gamma} - \r'\tilde \I_h^\frac{1}{\gamma}\|,
\end{equation}
where $\tilde \I_h$ is the resulting enhanced image based on Eq.~\ref{eq::contrastbrightness3}, $\r'= \left(\frac{\overline{\I_l}}{E}\right)^{1/\gamma}$ with an exposure factor $E\sim \mathcal{N}(\eta,\sigma^2)$. Based on the finding that the average exposure of a normally exposed image is nearly $0.5$~\citep{mertens2007exposure,mertens2009exposure}, $\eta$ is set to $0.5$. The variance of Gaussian $\sigma^2$ specifies the divergence of targeted exposures, which will be determined in experimental settings. The power of $1/\gamma$ can suppress the magnitude of $\betam'$ and its details are discussed in the appendix. 
We aso plot the joint histogram of $\I_h$ vs. $\I_l$ as shown in the last plot of Fig.~\ref{fig::plot_joint_hist_lite}. 
A linear regression of $\I_l=\alpham \I_h + \betam'$ is done for the dots $\I_h \in (0,k)$. 
The regression is fitted well by the linear model with a large $R^2$. As shown, $\betam'$ has a small magnitude and is a negligible offset.

The mask function $\mathcal{M}$ is defined as
\begin{equation}
    \mathcal{M}\left(x\right) = 
    \left\{
    \begin{split}
        &1 \text{, if $x<k$} \\
        &0 \text{, otherwise}. 
    \end{split}
    \right.
\end{equation}
The mask function is introduced because the approximate linear proportion between $\I_l$ and $\I_h$ is limited in regions without dynamic range compression.

\subsection{Variance Suppression loss}

\begin{figure*}[htbp]
  \centering
  \begin{subfigure}{0.196\linewidth}
    \includegraphics[width=1\linewidth]{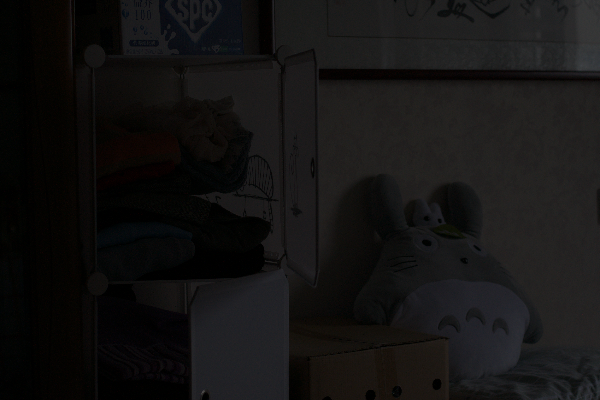}
      \caption{Input}
      \label{fig:lolv1-3-a}
  \end{subfigure}
  \hfill
  \begin{subfigure}{0.196\linewidth}
    \includegraphics[width=1\linewidth]{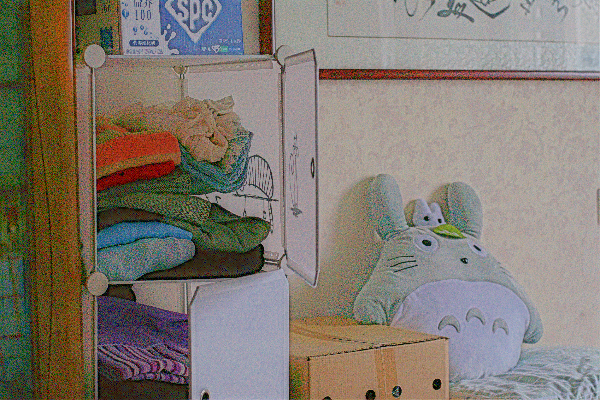}
      \caption{RetinexNet}
      \label{fig:lolv1-3-c}
  \end{subfigure}
  \hfill
  \begin{subfigure}{0.196\linewidth}
    \includegraphics[width=1\linewidth]{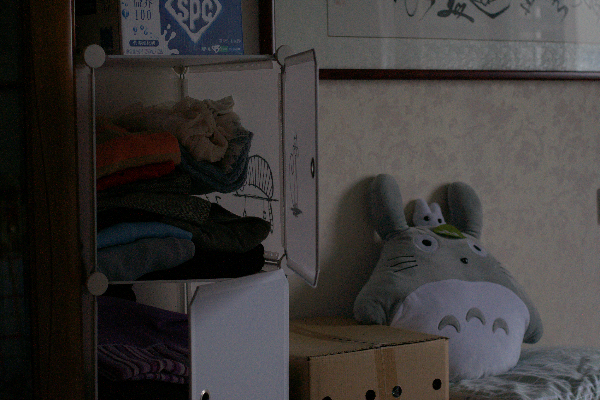}
      \caption{RRDNet}
      \label{fig:lolv1-3-d}
  \end{subfigure}
  \hfill
  \begin{subfigure}{0.196\linewidth}
    \includegraphics[width=1\linewidth]{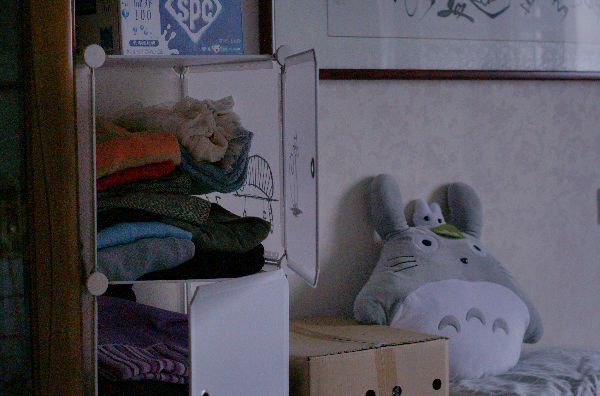}
      \caption{ZeroDCE++}
      \label{fig:lolv1-3-e}
  \end{subfigure}
  \hfill
  \begin{subfigure}{0.196\linewidth}
    \includegraphics[width=1\linewidth]{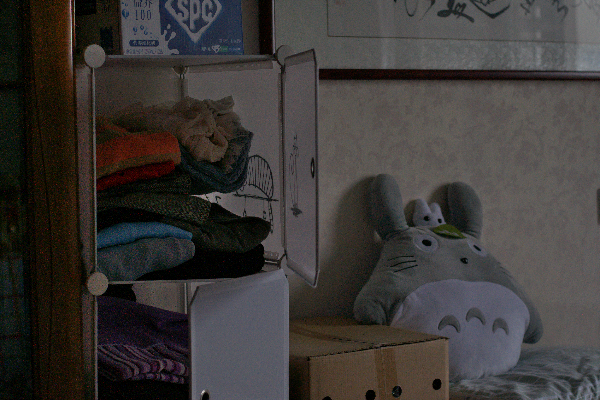}
      \caption{RetinexDIP}
      \label{fig:lolv1-3-f}
  \end{subfigure}
  
  \centering
  \begin{subfigure}{0.196\linewidth}
    \includegraphics[width=1\linewidth]{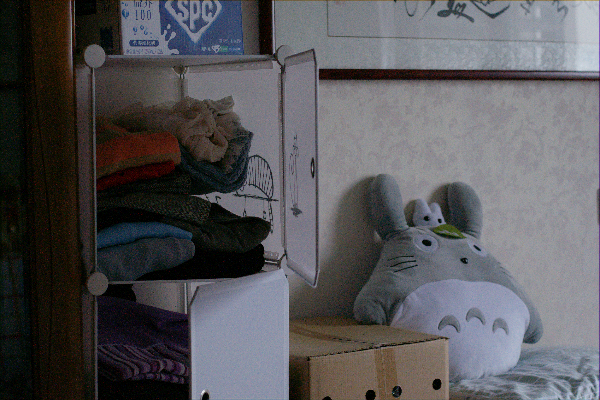}
      \caption{SCI}
      \label{fig:lolv1-3-g}
  \end{subfigure}
  \hfill
  \begin{subfigure}{0.196\linewidth}
    \includegraphics[width=1\linewidth]{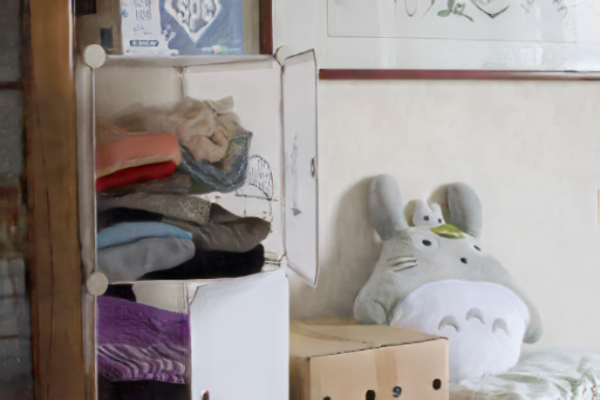}
      \caption{NightEnhance}
      \label{fig:lolv1-3-h}
  \end{subfigure}
  \hfill
  \begin{subfigure}{0.196\linewidth}
    \includegraphics[width=1\linewidth]{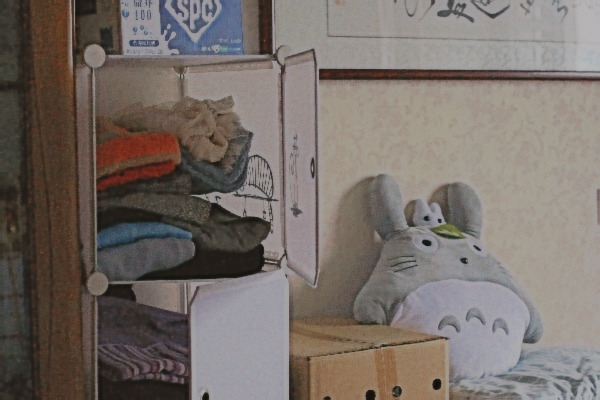}
      \caption{PairLIE}
      \label{fig:lolv1-3-b}
  \end{subfigure}
  \hfill
  \begin{subfigure}{0.196\linewidth}
    \includegraphics[width=1\linewidth]{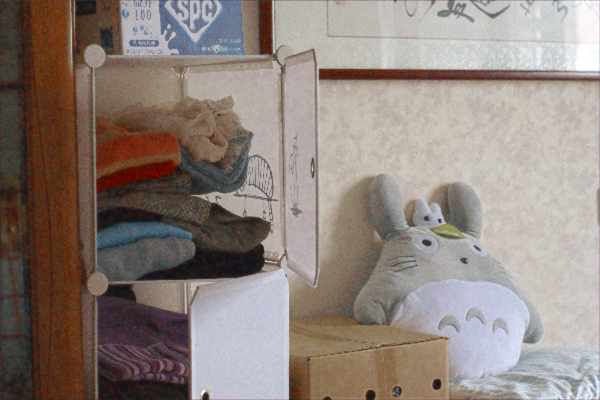}
      \caption{Ours}
      \label{fig:lolv1-3-i}
  \end{subfigure}
  \hfill
  \begin{subfigure}{0.196\linewidth}
    \includegraphics[width=1\linewidth]{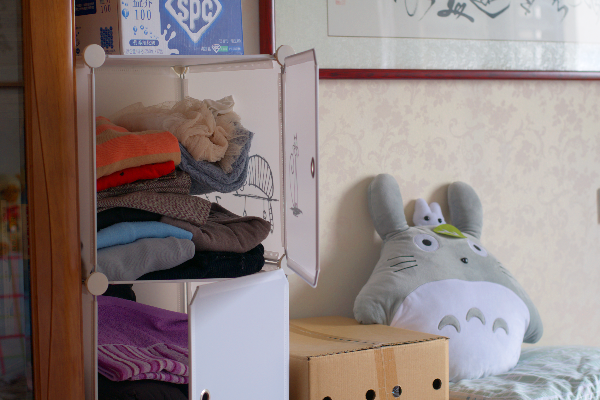}
      \caption{GT}
      \label{fig:lolv1-3-j}
  \end{subfigure}
\setlength{\abovecaptionskip}{-3pt} 
\setlength{\belowcaptionskip}{-3pt}
  \caption{{\colorRevision A visual comparison of enhancement results on LOL-v1. Please zoom in for better visualization.} }
  \label{fig:lolv1-3}
\end{figure*}
\begin{figure*}[htbp]
  \centering
  \begin{subfigure}{0.196\linewidth}
    \includegraphics[width=1\linewidth]{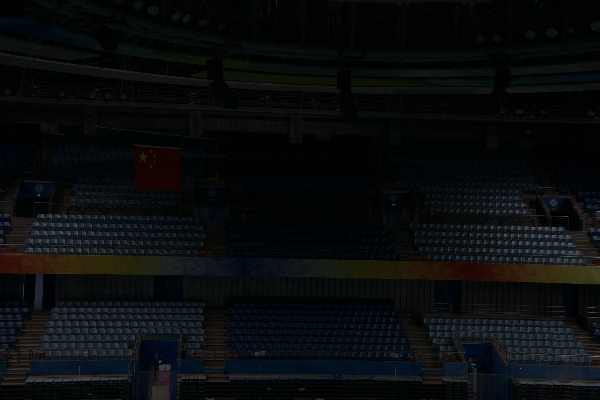}
      \caption{Input}
      \label{fig:lolv2-2-a}
  \end{subfigure}
  \hfill
  \begin{subfigure}{0.196\linewidth}
    \includegraphics[width=1\linewidth]{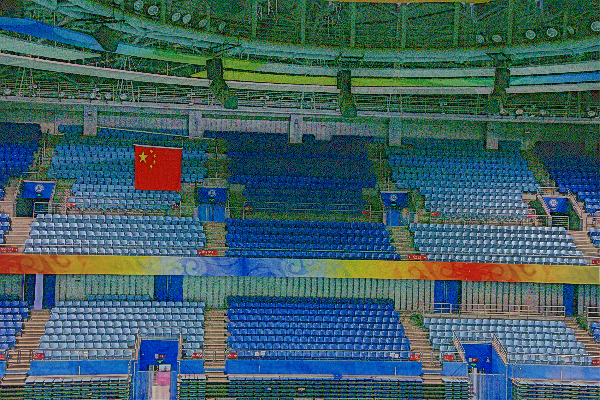}
      \caption{RetinexNet}
      \label{fig:lolv2-2-c}
  \end{subfigure}
  \hfill
  \begin{subfigure}{0.196\linewidth}
    \includegraphics[width=1\linewidth]{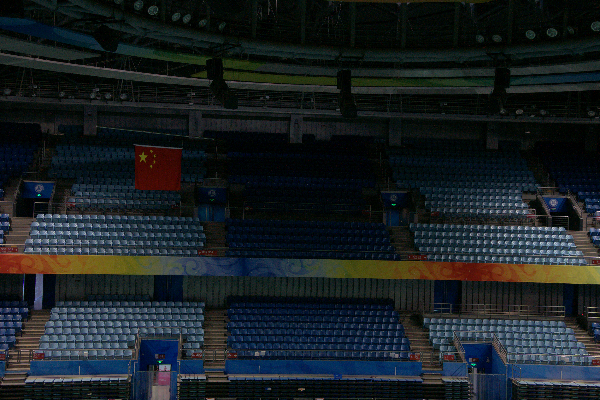}
      \caption{RRDNet}
      \label{fig:lolv2-2-d}
  \end{subfigure}
  \hfill
  \begin{subfigure}{0.196\linewidth}
    \includegraphics[width=1\linewidth]{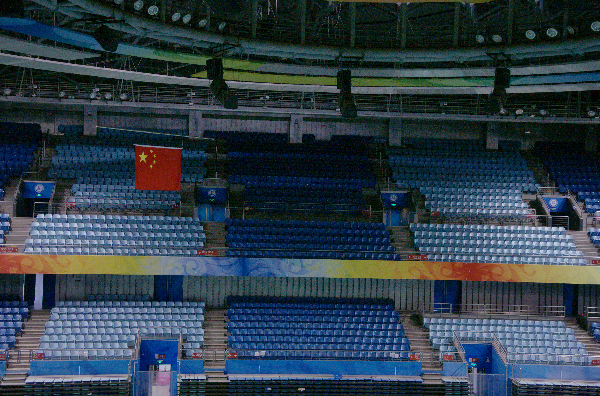}
      \caption{ZeroDCE++}
      \label{fig:lolv2-2-e}
  \end{subfigure}
  \hfill
  \begin{subfigure}{0.196\linewidth}
    \includegraphics[width=1\linewidth]{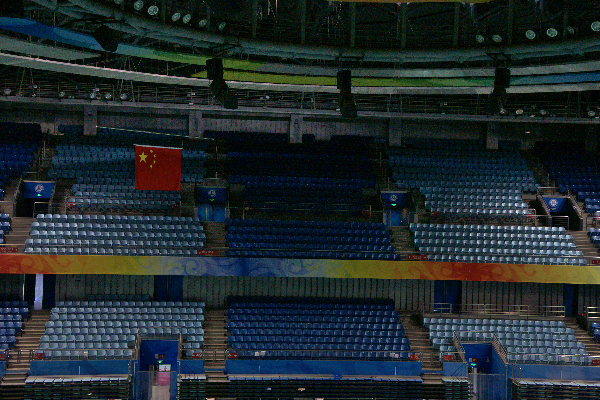}
      \caption{RetinexDIP}
      \label{fig:lolv2-2-f}
  \end{subfigure}
  
  \centering
  \begin{subfigure}{0.196\linewidth}
    \includegraphics[width=1\linewidth]{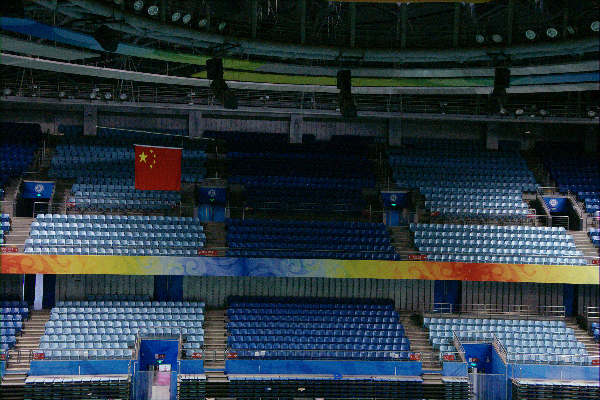}
      \caption{SCI}
      \label{fig:lolv2-2-g}
  \end{subfigure}
  \hfill
  \begin{subfigure}{0.196\linewidth}
    \includegraphics[width=1\linewidth]{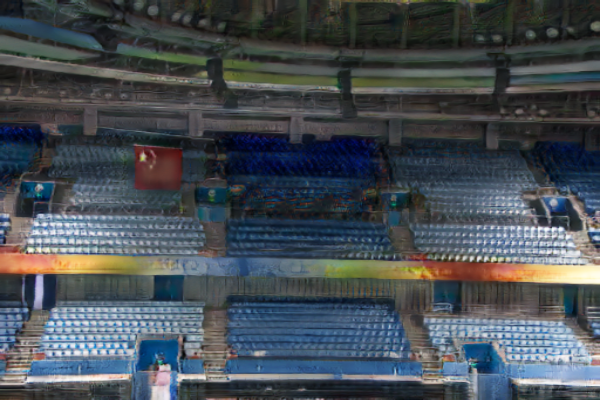}
      \caption{NightEnhance}
      \label{fig:lolv2-2-h}
  \end{subfigure}
  \hfill
  \begin{subfigure}{0.196\linewidth}
    \includegraphics[width=1\linewidth]{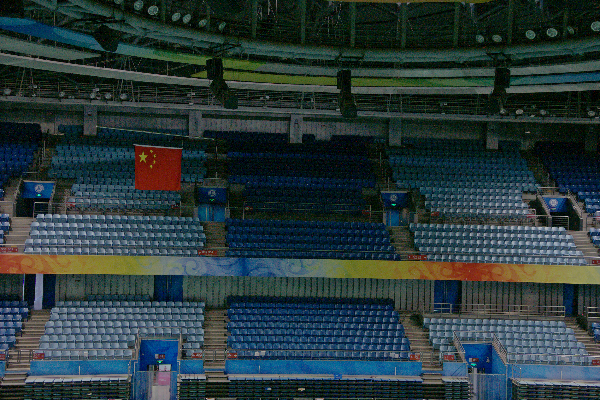}
      \caption{PairLIE}
      \label{fig:lolv2-2-b}
  \end{subfigure}
  \hfill
  \begin{subfigure}{0.196\linewidth}
    \includegraphics[width=1\linewidth]{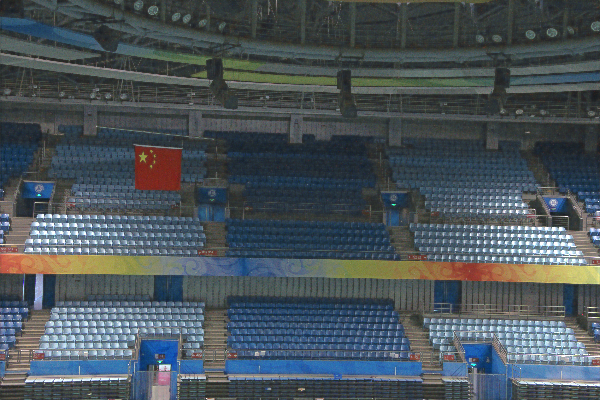}
      \caption{Ours}
      \label{fig:lolv2-2-i}
  \end{subfigure}
  \hfill
  \begin{subfigure}{0.196\linewidth}
    \includegraphics[width=1\linewidth]{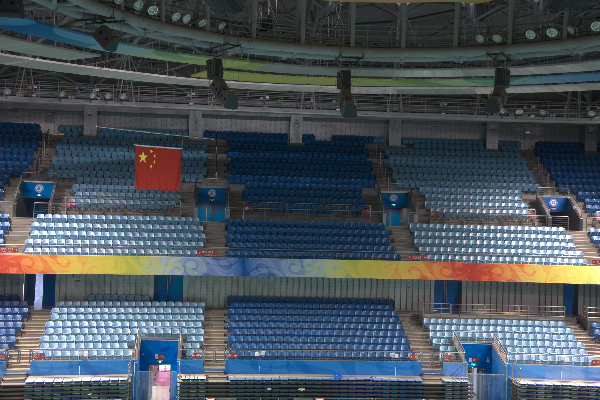}
      \caption{GT}
      \label{fig:lolv2-2-j}
  \end{subfigure}
\setlength{\abovecaptionskip}{-3pt} 
\setlength{\belowcaptionskip}{-3pt}
  \caption{{\colorRevision A visual comparison of enhancement results on LOL-v2. Please zoom in for better visualization.} }
  \label{fig:lolv2-2}
\end{figure*}

We want the second term in Eq.~\ref{eq::contrastbrightness3} consisting of $\a$ and $\b$ to learn the feature of $\betam$ in Eq.~\ref{eq::newllie_model}. We have shown that though $\betam$ has a non-zero mean that yields a necessary offset for modeling enhanced images, it has an amplified variance due to $\alpham > \onem$, which will contribute to an enhanced variance of recovered images. Therefore, we design a variance suppression loss such that its variance approaches zero. The expression of the variance suppression loss is given by 
\begin{equation}\label{eq::varianceSuppressionLoss}
    \mathcal{L}_{VS} = \|Var(\a\b -\a + \b + 1)\|,
\end{equation}
\noindent where $Var$ is computed along each channels. 
As a result, the overall loss of our model is as follows:
\begin{equation}\label{eq::overallLoss}
    \mathcal{L}_{overall} = \mathcal{L}_{RD} + \mathcal{L}_{VS} .
\end{equation}

{\colorRevision
\subsection{Relation between DI-Retinex and the enhancement model}
The traditional Retinex theory, originally borrowed from optics, has served as a fundamental guiding principle for numerous prior low-light enhancement works. 
However, it has come to our attention that the digital imaging process introduces deviations from the optical Retinex theory, resulting in the emergence of what we term the DI-Retinex theory. 
In accordance with this novel theory, we derive Eq.~\ref{eq::newllie_model}, wherein $\betam$ emerges as a significant term. 
In contrast, if we were to rely solely on the conventional Retinex theory, $\betam$ would assume a value of zero. 
This, in turn, renders the subsequently derived enhancement model in Eq.~\ref{eq::contrastbrightness3}, which incorporates an offset term, invalid.
In summary, the proposed DI-Retinex theory ensures the presence of a non-negligible offset term in the relationship between low-light and normal images. 
This, in turn, substantiates the formulation of the presented enhancement model, which features an offset term. 
Furthermore, the DI-Retinex theory establishes a reverse relationship between low-light and normal images, enabling us to derive the reverse degradation loss.
}

\section{Experiments}

\subsection{Experimental Settings}

\noindent\textbf{Datasets.} We adopt the datasets for comparison including 
LOL-v1~\citep{Chen2018Retinex}, LOL-v2~\citep{yang2021sparse}, and 
DARKFACE~\citep{yuan2019ug}. The details and statistics of the datasets are illustrated in the supplemental material.
The number of training and test sets of LOL-v1~\citep{Chen2018Retinex}, LOL-v2-Real~\citep{yang2021sparse} and DARKFACE~\citep{yuan2019ug} are illustrated in Table.~\ref{tab:statsData}. We randomly sample 1000 images for DARKFACE following ~\citep{li2021low,ma2022toward}. Note that LOL-v2 has a synthetic part and a real captured part. Since we are modeling realistic low light degradation by extended Retinex theory, we only experiment on its real part. 
\begin{table}[htbp]
\small
  \centering
  \caption{The statistics of datasets.}
    \begin{tabular}{ccc}
    \toprule
          & Train & Test \\
    \midrule
    LOL-v1 & 485   & 15 \\
    LOL-v2-Real & 689   & 100 \\
    \midrule
    \midrule
    \multirow{2}[2]{*}{DARKFACE} & \# of Images & \# of Faces \\
          & 1000  & 37040 \\
    \bottomrule
    \end{tabular}%
  \label{tab:statsData}%
\end{table}%

\begin{figure*}[htbp]
  \centering
  \begin{subfigure}{0.196\linewidth}
    \includegraphics[width=1\linewidth]{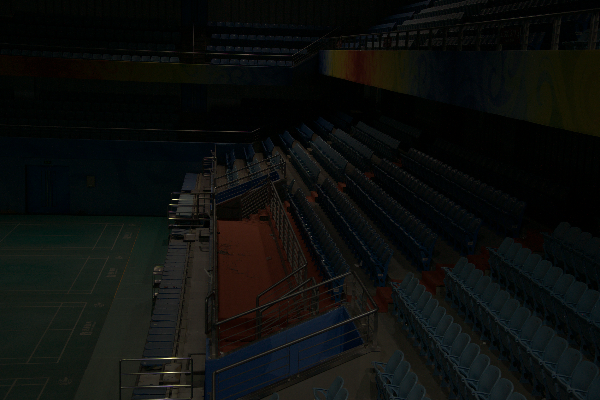}
      \caption{Input}
      \label{fig:lolv2-3-a}
  \end{subfigure}
  \hfill
  \begin{subfigure}{0.196\linewidth}
    \includegraphics[width=1\linewidth]{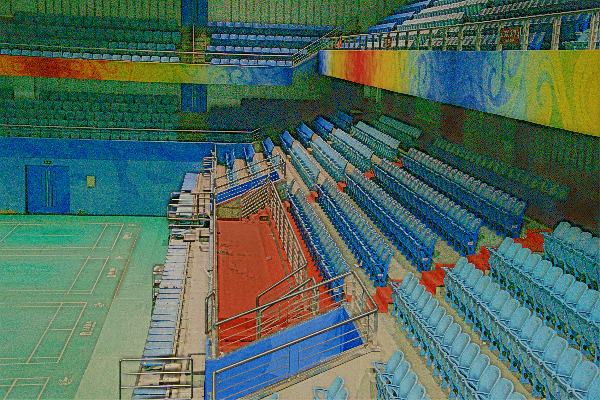}
      \caption{RetinexNet}
      \label{fig:lolv2-3-c}
  \end{subfigure}
  \hfill
  \begin{subfigure}{0.196\linewidth}
    \includegraphics[width=1\linewidth]{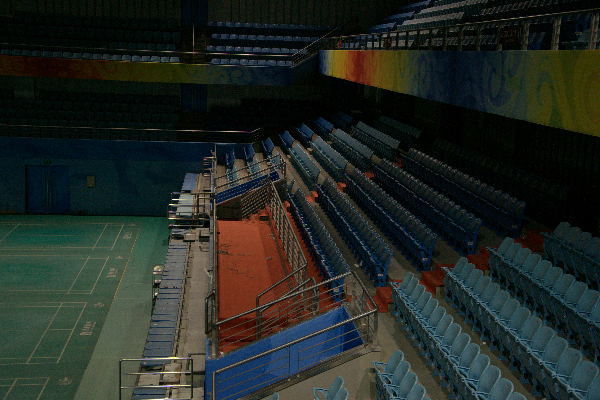}
      \caption{RRDNet}
      \label{fig:lolv2-3-d}
  \end{subfigure}
  \hfill
  \begin{subfigure}{0.196\linewidth}
    \includegraphics[width=1\linewidth]{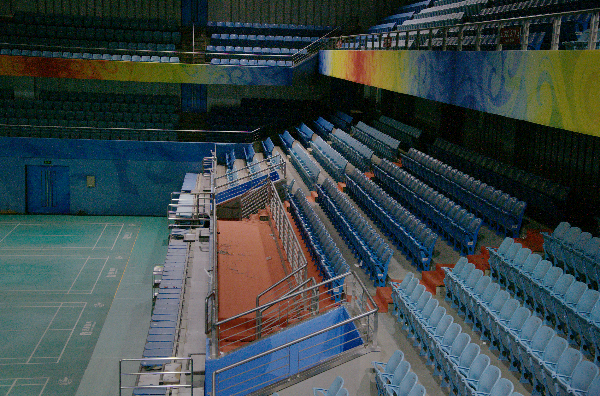}
      \caption{ZeroDCE++}
      \label{fig:lolv2-3-e}
  \end{subfigure}
  \hfill
  \begin{subfigure}{0.196\linewidth}
    \includegraphics[width=1\linewidth]{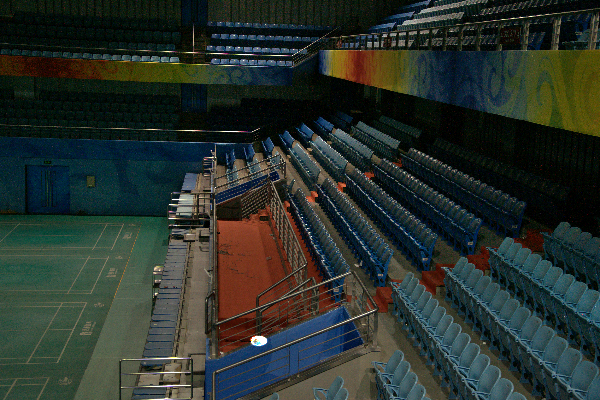}
      \caption{RetinexDIP}
      \label{fig:lolv2-3-f}
  \end{subfigure}
  
  \centering
  \begin{subfigure}{0.196\linewidth}
    \includegraphics[width=1\linewidth]{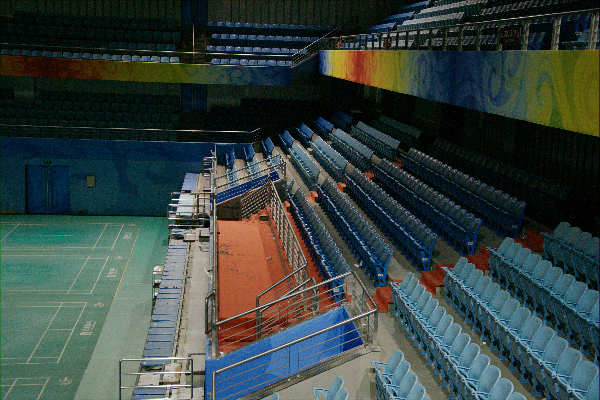}
      \caption{SCI}
      \label{fig:lolv2-3-g}
  \end{subfigure}
  \hfill
  \begin{subfigure}{0.196\linewidth}
    \includegraphics[width=1\linewidth]{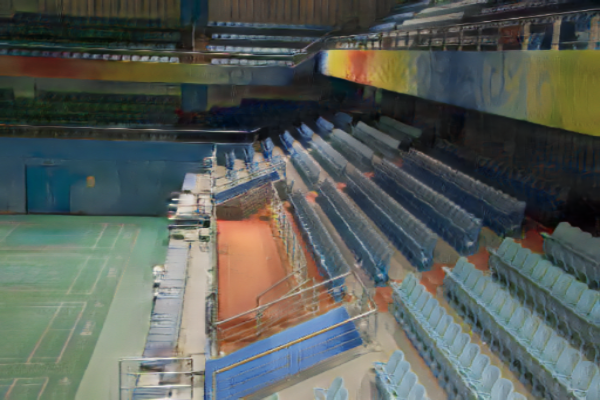}
      \caption{NightEnhance}
      \label{fig:lolv2-3-h}
  \end{subfigure}
  \hfill
  \begin{subfigure}{0.196\linewidth}
    \includegraphics[width=1\linewidth]{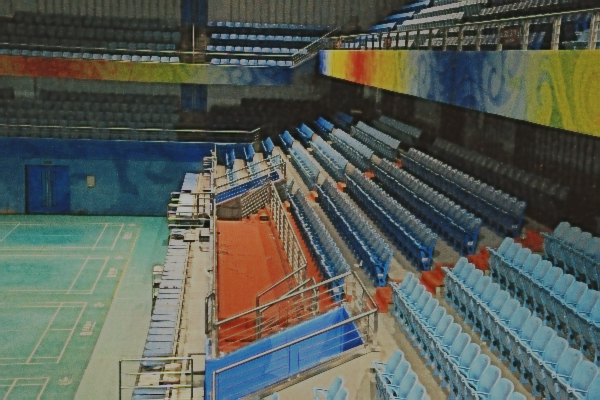}
      \caption{PairLIE}
      \label{fig:lolv2-3-b}
  \end{subfigure}
  \hfill
  \begin{subfigure}{0.196\linewidth}
    \includegraphics[width=1\linewidth]{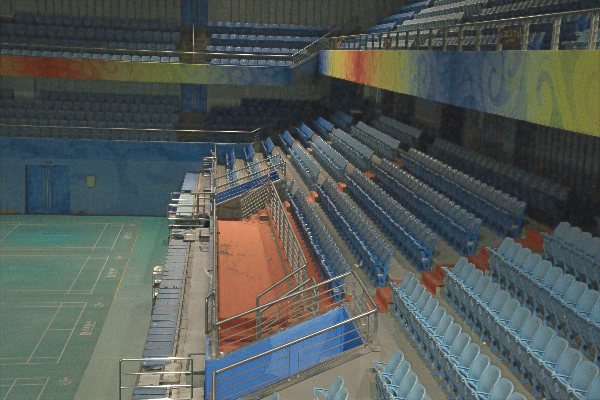}
      \caption{Ours}
      \label{fig:lolv2-3-i}
  \end{subfigure}
  \hfill
  \begin{subfigure}{0.196\linewidth}
    \includegraphics[width=1\linewidth]{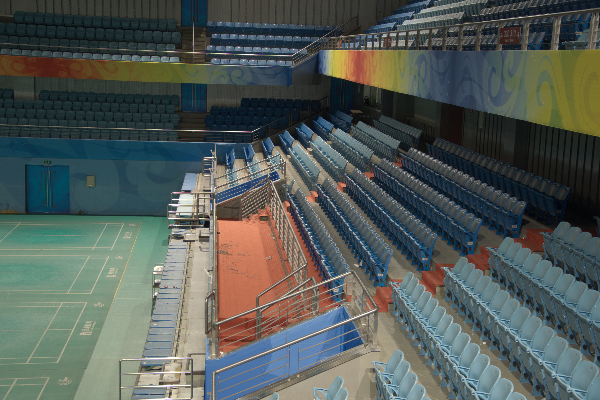}
      \caption{GT}
      \label{fig:lolv2-3-j}
  \end{subfigure}
\setlength{\abovecaptionskip}{-3pt} 
\setlength{\belowcaptionskip}{-3pt}
  \caption{{\colorRevision A visual comparison of enhancement results on LOL-v2. Please zoom in for better visualization.} }
  \label{fig:lolv2-3}
\end{figure*}
\begin{figure*}[htbp]
  \centering
  \begin{subfigure}{0.196\linewidth}
    \includegraphics[width=1\linewidth]{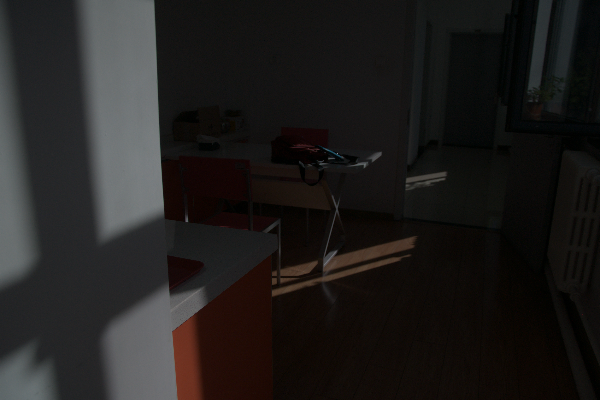}
      \caption{Input}
      \label{fig:lolv2-4-a}
  \end{subfigure}
  \hfill
  \begin{subfigure}{0.196\linewidth}
    \includegraphics[width=1\linewidth]{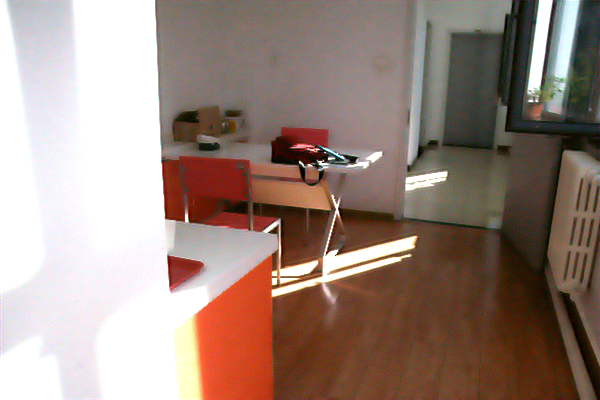}
      \caption{RUAS}
      \label{fig:lolv2-4-c}
  \end{subfigure}
  \hfill
  \begin{subfigure}{0.196\linewidth}
    \includegraphics[width=1\linewidth]{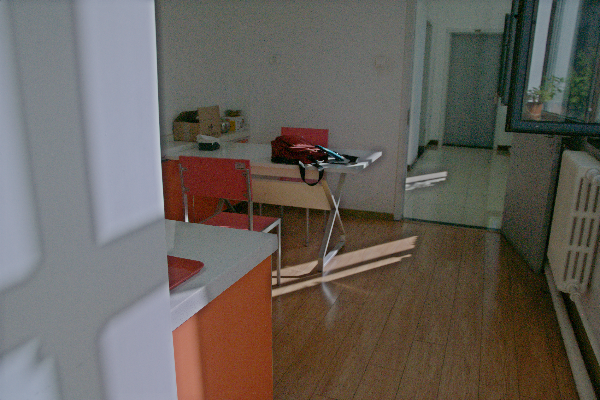}
      \caption{ZeroDCE}
      \label{fig:lolv2-4-d}
  \end{subfigure}
  \hfill
  \begin{subfigure}{0.196\linewidth}
    \includegraphics[width=1\linewidth]{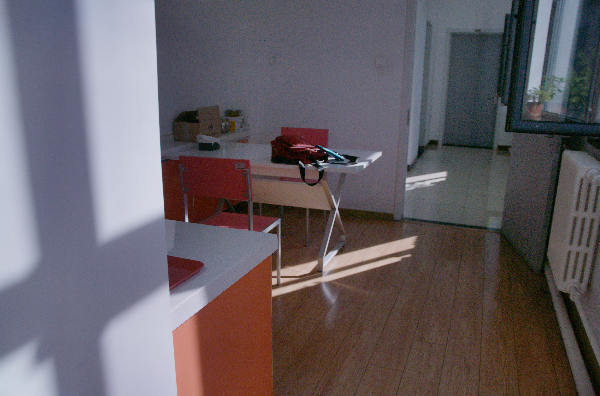}
      \caption{ZeroDCE++}
      \label{fig:lolv2-4-e}
  \end{subfigure}
  \begin{subfigure}{0.196\linewidth}
    \includegraphics[width=1\linewidth]{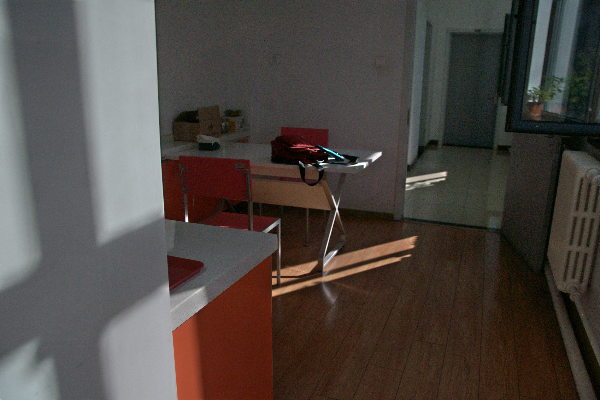}
      \caption{RetinexDIP}
      \label{fig:lolv2-4-f}
  \end{subfigure}
  
  \centering
  \begin{subfigure}{0.196\linewidth}
    \includegraphics[width=1\linewidth]{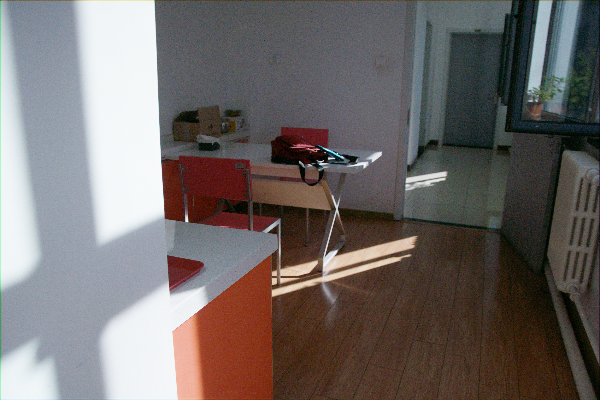}
      \caption{SCI}
      \label{fig:lolv2-4-g}
  \end{subfigure}
  \hfill
  \begin{subfigure}{0.196\linewidth}
    \includegraphics[width=1\linewidth]{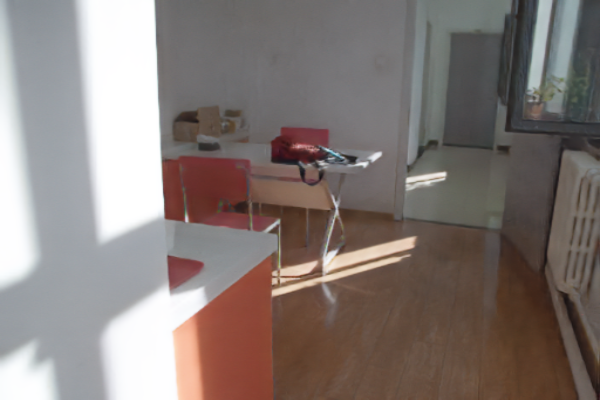}
      \caption{NightEnhance}
      \label{fig:lolv2-4-h}
  \end{subfigure}
  \hfill
  \begin{subfigure}{0.196\linewidth}
    \includegraphics[width=1\linewidth]{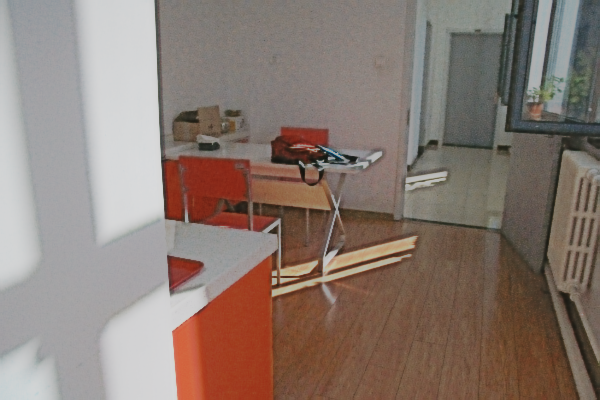}
      \caption{PairLIE}
      \label{fig:lolv2-4-b}
  \end{subfigure}
  \hfill
  \begin{subfigure}{0.196\linewidth}
    \includegraphics[width=1\linewidth]{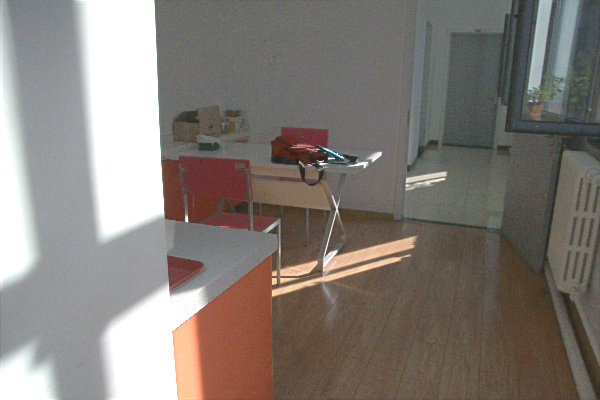}
      \caption{Ours}
      \label{fig:lolv2-4-i}
  \end{subfigure}
  \hfill
  \begin{subfigure}{0.196\linewidth}
    \includegraphics[width=1\linewidth]{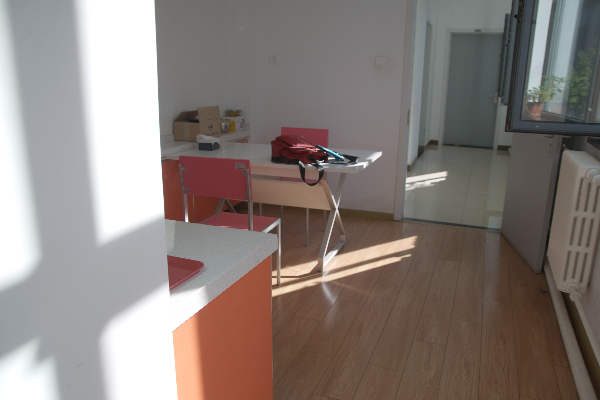}
      \caption{GT}
      \label{fig:lolv2-4-j}
  \end{subfigure}
  \caption{{\colorRevision A visual comparison of enhancement results on LOL-v2. Please zoom in for better visualization.} }
  \label{fig:lolv2-4}
\end{figure*}
\begin{figure*}[htbp]
  \centering
  \begin{subfigure}{0.196\linewidth}
    \includegraphics[width=1\linewidth]{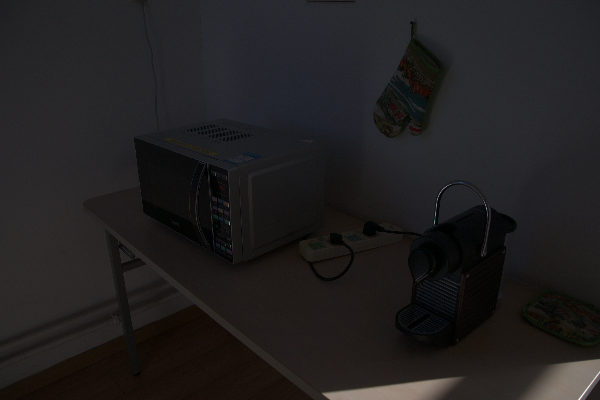}
      \caption{Input}
      \label{fig:lolv2-6-a}
  \end{subfigure}
  \hfill
  \begin{subfigure}{0.196\linewidth}
    \includegraphics[width=1\linewidth]{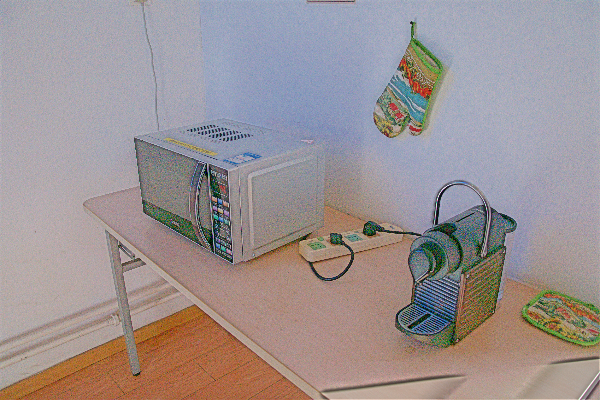}
      \caption{RetinexNet}
      \label{fig:lolv2-6-c}
  \end{subfigure}
  \hfill
  \begin{subfigure}{0.196\linewidth}
    \includegraphics[width=1\linewidth]{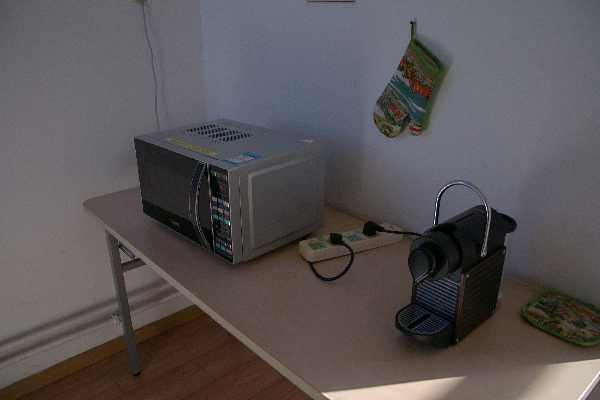}
      \caption{RRDNet}
      \label{fig:lolv2-6-d}
  \end{subfigure}
  \hfill
  \begin{subfigure}{0.196\linewidth}
    \includegraphics[width=1\linewidth]{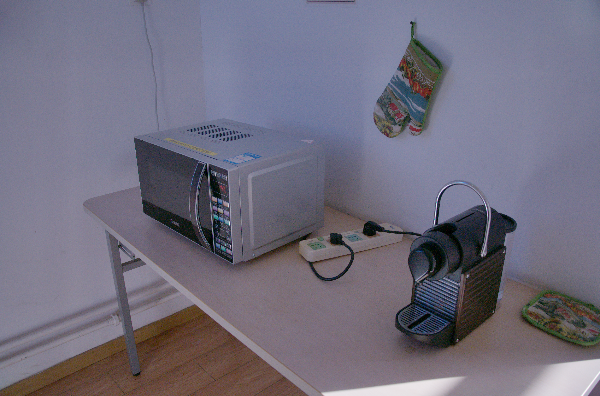}
      \caption{ZeroDCE++}
      \label{fig:lolv2-6-e}
  \end{subfigure}
  \hfill
  \begin{subfigure}{0.196\linewidth}
    \includegraphics[width=1\linewidth]{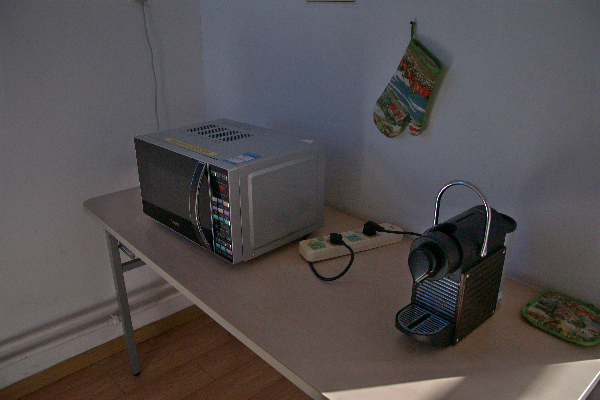}
      \caption{RetinexDIP}
      \label{fig:lolv2-6-f}
  \end{subfigure}

  \centering
  \begin{subfigure}{0.196\linewidth}
    \includegraphics[width=1\linewidth]{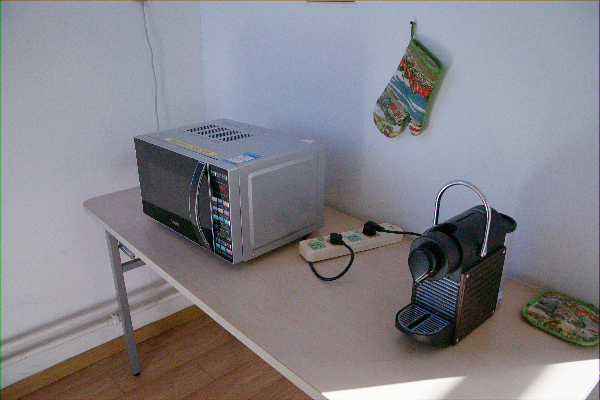}
      \caption{SCI}
      \label{fig:lolv2-6-g}
  \end{subfigure}
  \hfill
  \begin{subfigure}{0.196\linewidth}
    \includegraphics[width=1\linewidth]{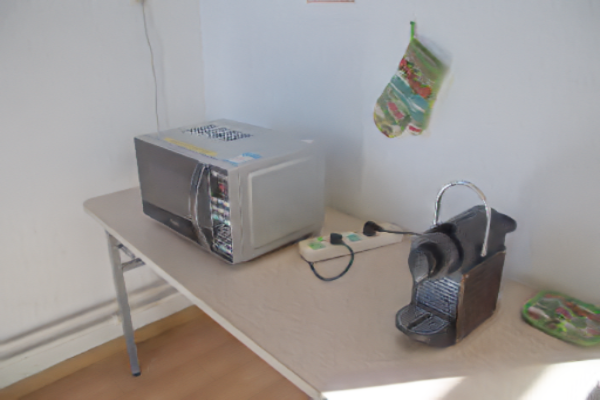}
      \caption{NightEnhance}
      \label{fig:lolv2-6-h}
  \end{subfigure}
  \hfill
  \begin{subfigure}{0.196\linewidth}
    \includegraphics[width=1\linewidth]{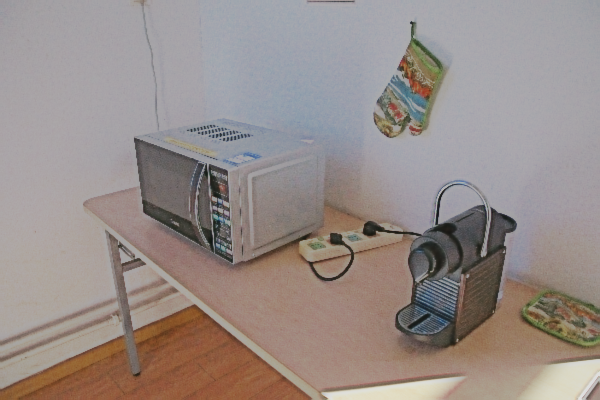}
      \caption{PairLIE}
      \label{fig:lolv2-6-b}
  \end{subfigure}
  \hfill
  \begin{subfigure}{0.196\linewidth}
    \includegraphics[width=1\linewidth]{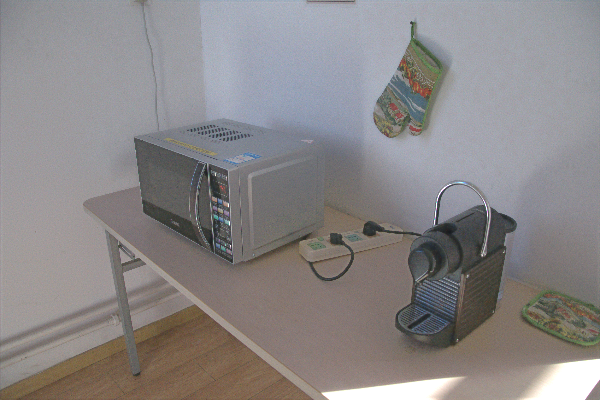}
      \caption{Ours}
      \label{fig:lolv2-6-i}
  \end{subfigure}
  \hfill
  \begin{subfigure}{0.196\linewidth}
    \includegraphics[width=1\linewidth]{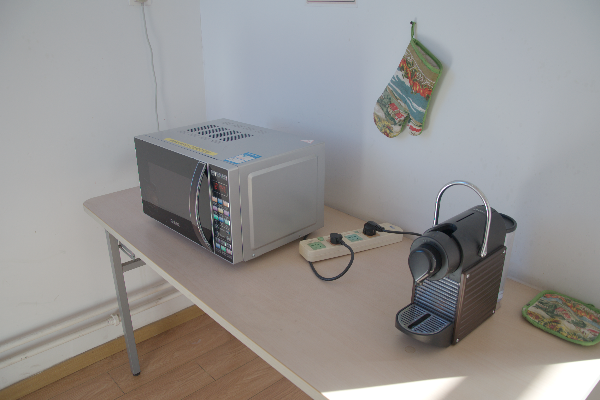}
      \caption{GT}
      \label{fig:lolv2-6-j}
  \end{subfigure}
\setlength{\abovecaptionskip}{-3pt} 
\setlength{\belowcaptionskip}{-3pt}
  \caption{{\colorRevision A visual comparison of enhancement results on LOL-v2. Please zoom in for better visualization.} }
  \label{fig:lolv2-6}
  \vspace{-0.1cm}
\end{figure*}

\noindent\textbf{Baselines.} We compare our methods with model-based method including LIME~\citep{guo2016lime}, supervised learning methods including RetinexNet~\citep{Chen2018Retinex}, KinD~\citep{zhang2019kindling}, and RUAS~\citep{liu2021retinex}, semi-supervised learning methods including DRBN~\citep{yang2021band}, unpaired supervised learning methods including EnlightenGAN~\citep{jiang2021enlightengan} and NightEnhance~\citep{jin2022unsupervised}, and zero-shot learning methods including RRDNet~\citep{zhu2020zero}, Zero-DCE~\citep{guo2020zero}, Zero-DCE++~\citep{li2021learning}, RetinexDIP~\citep{zhao2021retinexdip}, SCI~\citep{ma2022toward} and PairLIE~\citep{fu2023pairlie}. Since our method belongs to zero-shot learning, we focus on the comparison with the existing zero-shot learning methods.

\begin{table}[bp]
\tabcolsep=0.8mm
  \small
  \centering
  \caption{{\colorRevision The quantitative results of LLIE methods on LOL-v1~\citep{Chen2018Retinex} in terms of MSE [$\times 10^3$], PSNR in dB, SSIM~\citep{wang2004image} and LPIPS~\citep{zhang2018unreasonable}. }The best, second best and third best results are in {\color{red} red}, {\color{blue} blue} and \textcolor[RGB]{84,181,53}{ green} respectively.}
  \vspace{-3mm}
    \begin{tabular}{lcccc}
    \toprule
          & MSE$\downarrow$ & PSNR$\uparrow$ & SSIM$\uparrow$ & LPIPS$\downarrow$ \\
    \midrule
    Input & 12.622 & 7.771 & 0.181 & 0.560 \\
    LIME~\citep{guo2016lime}  & 2.269 & 16.760 & 0.560 & 0.350 \\
    RetinexNet~\citep{Chen2018Retinex} & 1.656 & 16.774 & 0.462 & 0.474 \\
    KinD~\citep{zhang2019kindling}  & 1.431 & 17.648 & {\color{red} 0.779} & {\color{red} 0.175} \\
    RUAS~\citep{liu2021retinex}  & 3.920 & 16.398 & 0.537 & 0.350 \\
    DRBN~\citep{yang2021band}  & 2.359 & 15.125 & 0.472 & 0.316 \\
    EnlightenGAN~\citep{jiang2021enlightengan} & 1.998 & 17.478 & 0.677 & 0.322 \\
    RRDNet~\citep{zhu2020zero} & 6.313 & 11.384 & 0.470 & 0.361 \\
    RetinexDIP~\citep{zhao2021retinexdip} & 6.050 & 11.646 & 0.501 & 0.317 \\
    ZeroDCE~\citep{guo2020zero} & 3.282 & 14.857 & 0.589 & 0.335 \\
    ZeroDCE++~\citep{li2021learning} & 3.035 & 15.342 & 0.603 & 0.316 \\
    SCI~\citep{li2021learning}   & 3.496 & 14.780 & 0.553 & 0.332 \\
    NightEnhance~\citep{jin2022unsupervised} & {\color{blue} 1.070} & {\color{blue} 21.521} & \textcolor[RGB]{84,181,53}{ 0.763} & \textcolor[RGB]{84,181,53}{ 0.235} \\
    PairLIE~\citep{fu2023pairlie} & \textcolor[RGB]{84,181,53}{ 1.419} & \textcolor[RGB]{84,181,53}{ 18.463} & 0.749 & 0.290\\
    \midrule
    Ours & {\color{red} 0.784} & {\color{red} 21.542} & {\color{blue} 0.766} & {\color{blue} 0.219} \\
    Ours (small) &   1.706    & 18.448  & 0.641  & 0.312 \\
    \bottomrule
    \end{tabular}%
  \label{tab:lolv1}%
  \vspace{-4mm}
\end{table}%

\noindent\textbf{Evaluation Criteria.} We employ four full-reference metrics, i.e., MSE, PSNR, SSIM~\citep{wang2004image} and LPIPS~\citep{zhang2018unreasonable}, 
and two metrics indicating efficiency, i.e., model size and inference time. For the dataset DARKFACE~\citep{yuan2019ug}, a Precision-Recall curve is plotted for performance indication.

\noindent\textbf{Implementations.} We train our model on the training set of each dataset. We use ADAM as the optimizer with an initial learning rate of $0.001$ and weight decay of $0.0001$. The variance of Gaussian distribution in the reverse degradation loss is set to $0.001$ and $\tau$ in Eq.~\ref{eq::mapping} is set to 0.2. The network is trained for $1000$ epochs. Since the network is small and learns mapping coefficients $\a$ and $\b$ globally, we do not crop patches during training. Batch size is set to $1$. All experiments are conducted on an NVIDIA GeForce GTX 1080 GPU and implemented by PyTorch~\citep{paszke2019pytorch}. Our codes will be released.

\begin{figure}[htbp]
\small
  \centering
  \begin{subfigure}{0.243\linewidth}
    \includegraphics[width=1\linewidth]{figs/lolv1/input/1.png}
      \caption{Input]}
      \label{fig:var_loss_show_input}
  \end{subfigure}
  \hfill
  \begin{subfigure}{0.243\linewidth}
    \includegraphics[width=1\linewidth]{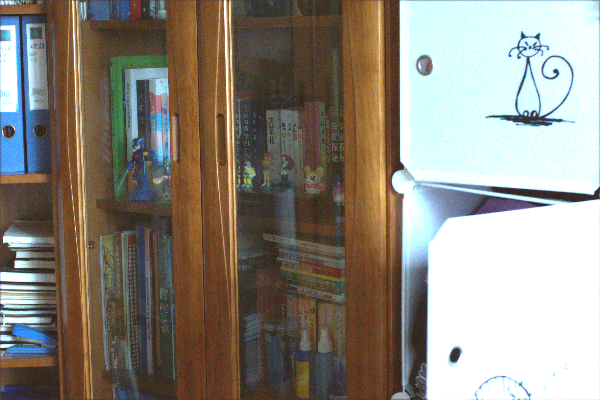}
      \caption{w/o $\mathcal{L}_{VS}$}
      \label{fig:var_loss_show_wovar}
  \end{subfigure}
  \hfill
  \begin{subfigure}{0.243\linewidth}
    \includegraphics[width=1\linewidth]{figs/lolv1/ours_/1.png}
      \caption{w/$\mathcal{L}_{VS}$}
      \label{fig:var_loss_show_ours}
  \end{subfigure}
  \hfill
  \begin{subfigure}{0.243\linewidth}
    \includegraphics[width=1\linewidth]{figs/lolv1/GT/1.png}
      \caption{Ground-truth}
      \label{fig:var_loss_show_gt}
  \end{subfigure}
  \caption{{\colorRevision The visual ablation comparison regarding $\mathcal{L}_{VS}$.}}
  \label{fig::var_loss_show}
  \vspace{-0.2cm}
\end{figure}

\begin{figure*}[htbp]
  \centering
  \begin{subfigure}{0.196\linewidth}
    \includegraphics[width=1\linewidth]{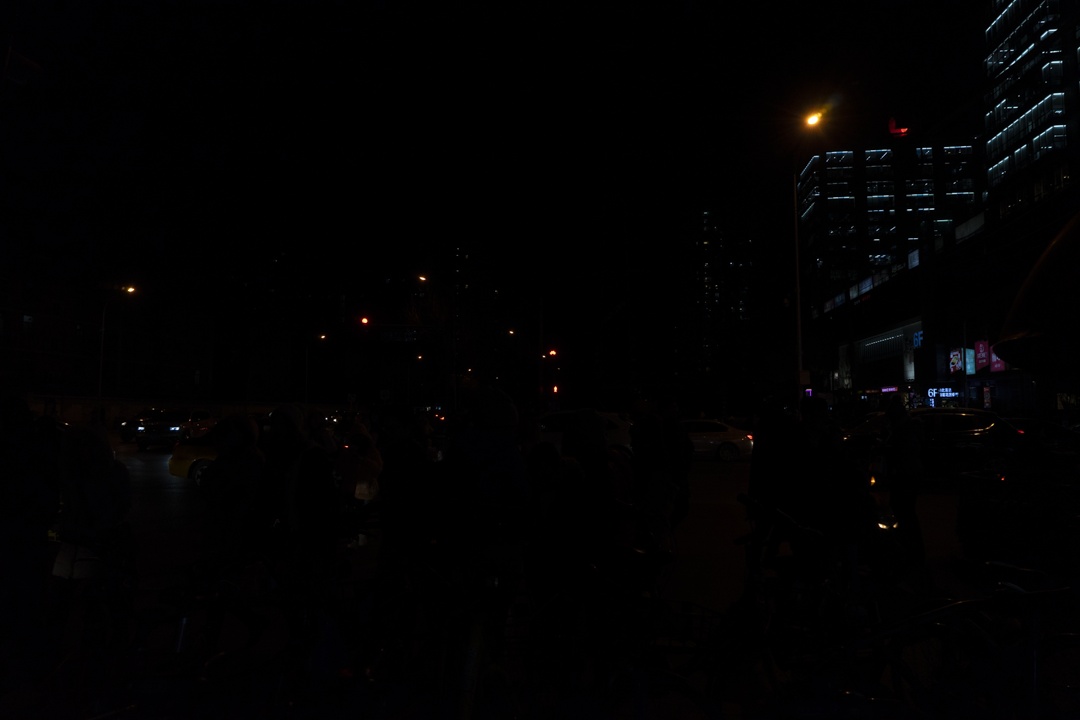}
      \caption{Input}
      \label{fig:darkface-1-a}
  \end{subfigure}
  \hfill
  \begin{subfigure}{0.196\linewidth}
    \includegraphics[width=1\linewidth]{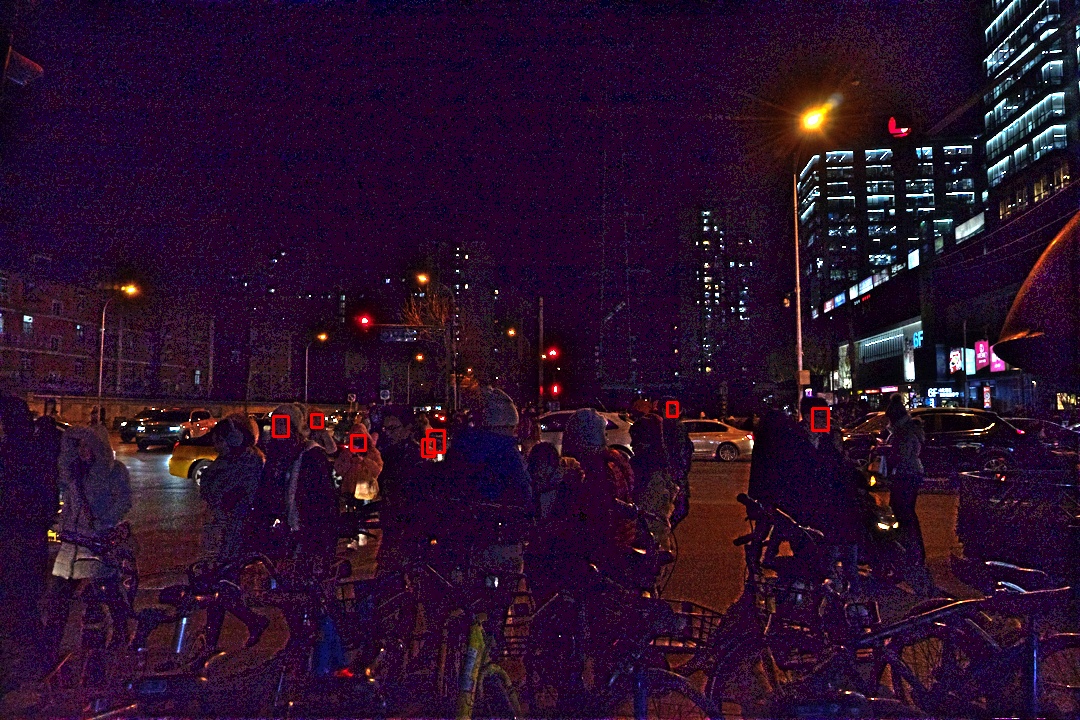}
      \caption{LIME}
      \label{fig:darkface-1-b}
  \end{subfigure}
  \hfill
  \begin{subfigure}{0.196\linewidth}
    \includegraphics[width=1\linewidth]{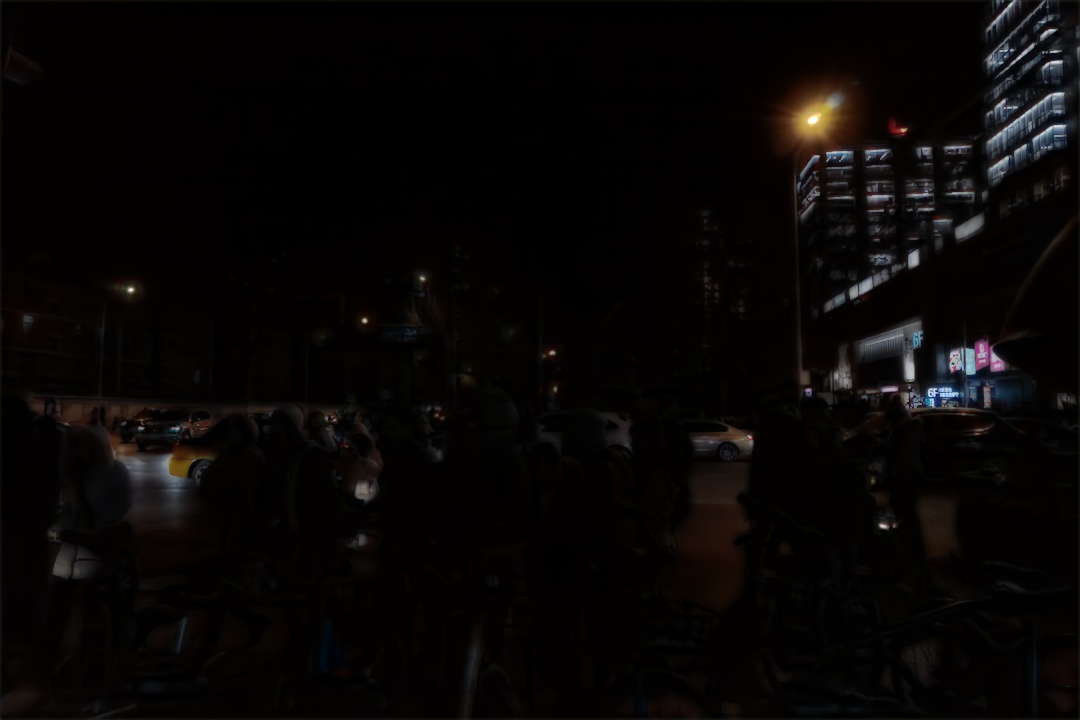}
      \caption{KinD}
      \label{fig:darkface-1-c}
  \end{subfigure}
  \hfill
  \begin{subfigure}{0.196\linewidth}
    \includegraphics[width=1\linewidth]{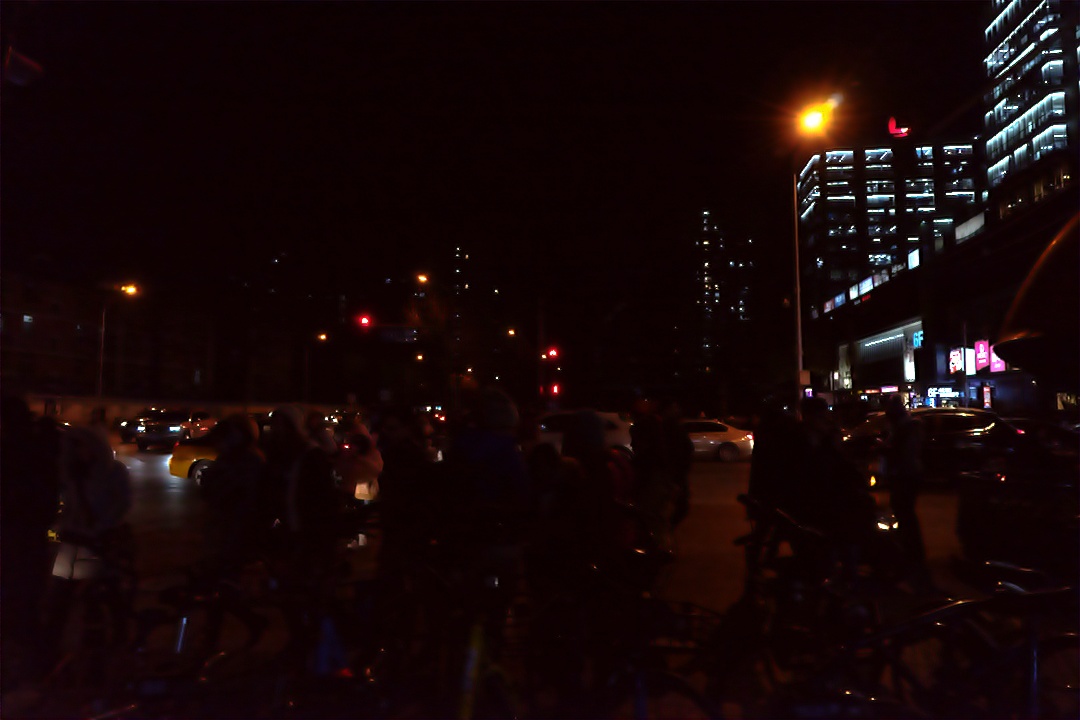}
      \caption{RUAS}
      \label{fig:darkface-1-d}
  \end{subfigure}
  \hfill
  \begin{subfigure}{0.196\linewidth}
    \includegraphics[width=1\linewidth]{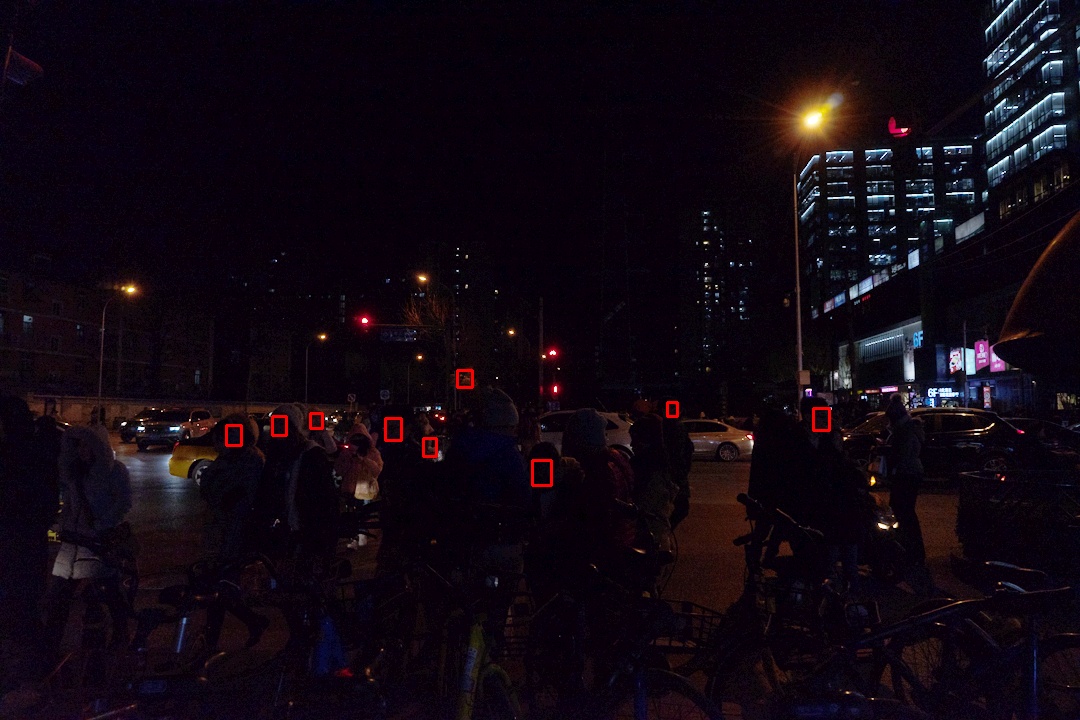}
      \caption{ZeroDCE++}
      \label{fig:darkface-1-f}
  \end{subfigure}
  
  \centering
  \centering
  \begin{subfigure}{0.196\linewidth}
    \includegraphics[width=1\linewidth]{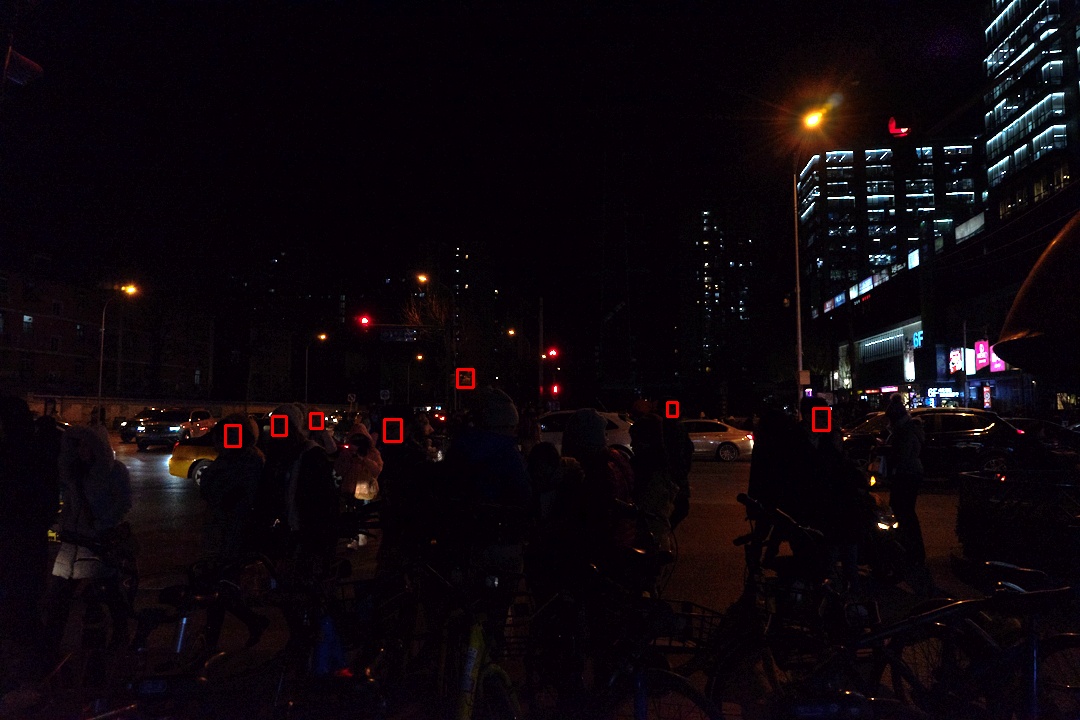}
      \caption{RetinexDIP}
      \label{fig:darkface-1-g}
  \end{subfigure}
  \hfill
  \begin{subfigure}{0.196\linewidth}
    \includegraphics[width=1\linewidth]{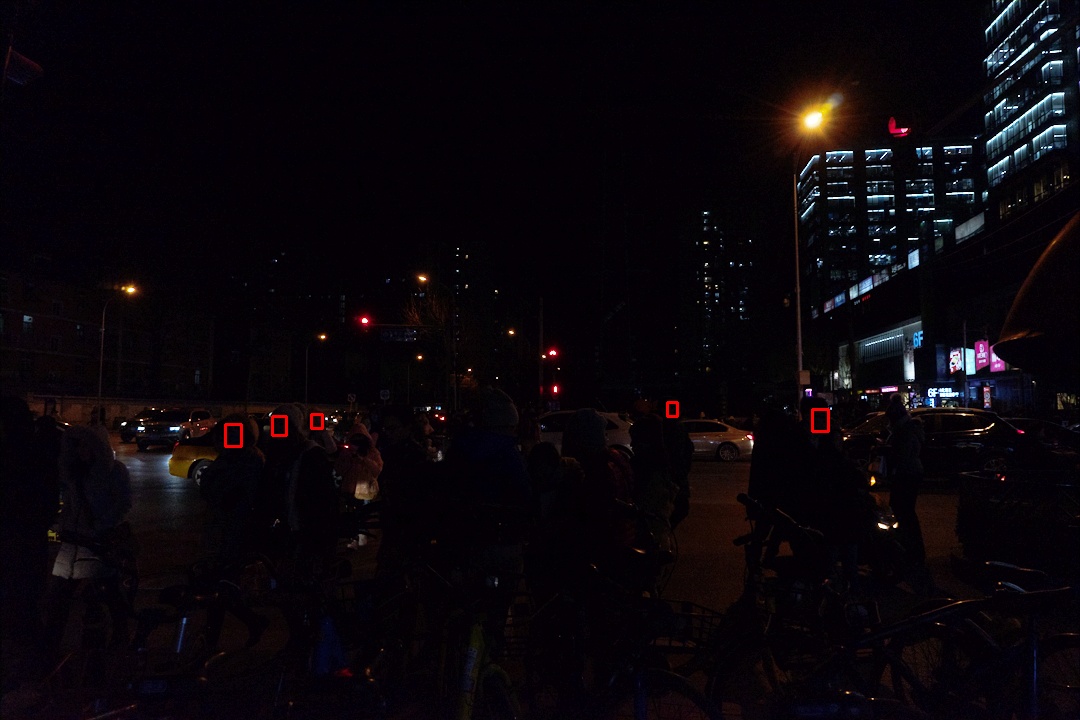}
      \caption{SCI}
      \label{fig:darkface-1-h}
  \end{subfigure}
  \hfill
  \begin{subfigure}{0.196\linewidth}
    \includegraphics[width=1\linewidth]{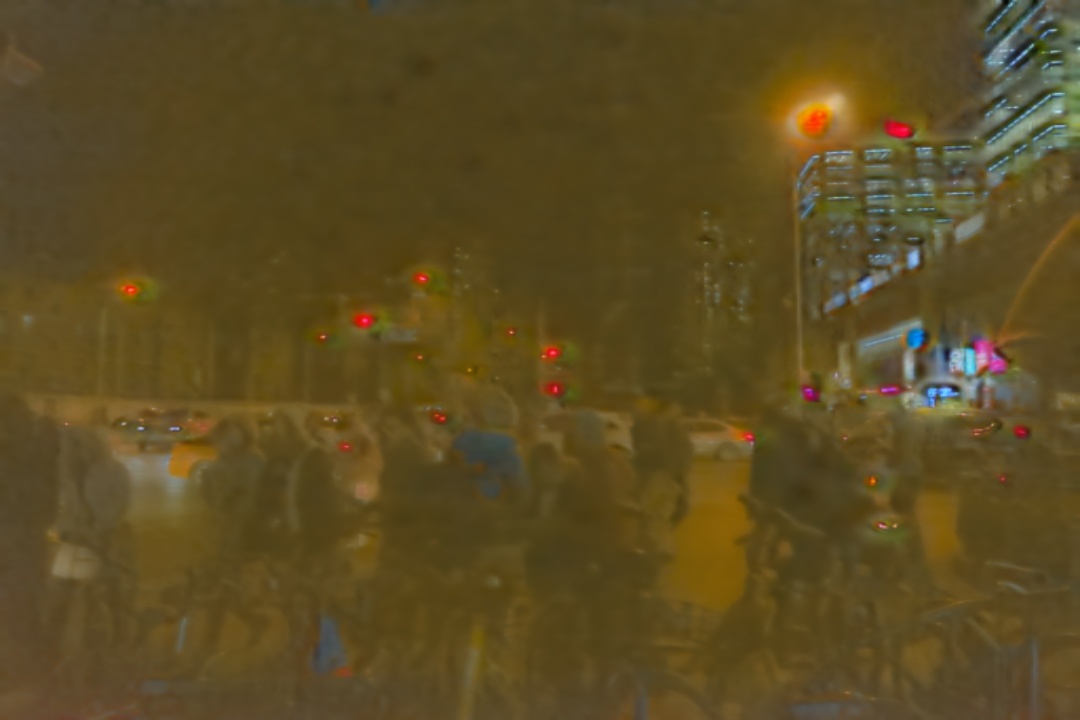}
      \caption{NightEnhance}
      \label{fig:darkface-1-i}
  \end{subfigure}
  \hfill
  \begin{subfigure}{0.196\linewidth}
    \includegraphics[width=1\linewidth]{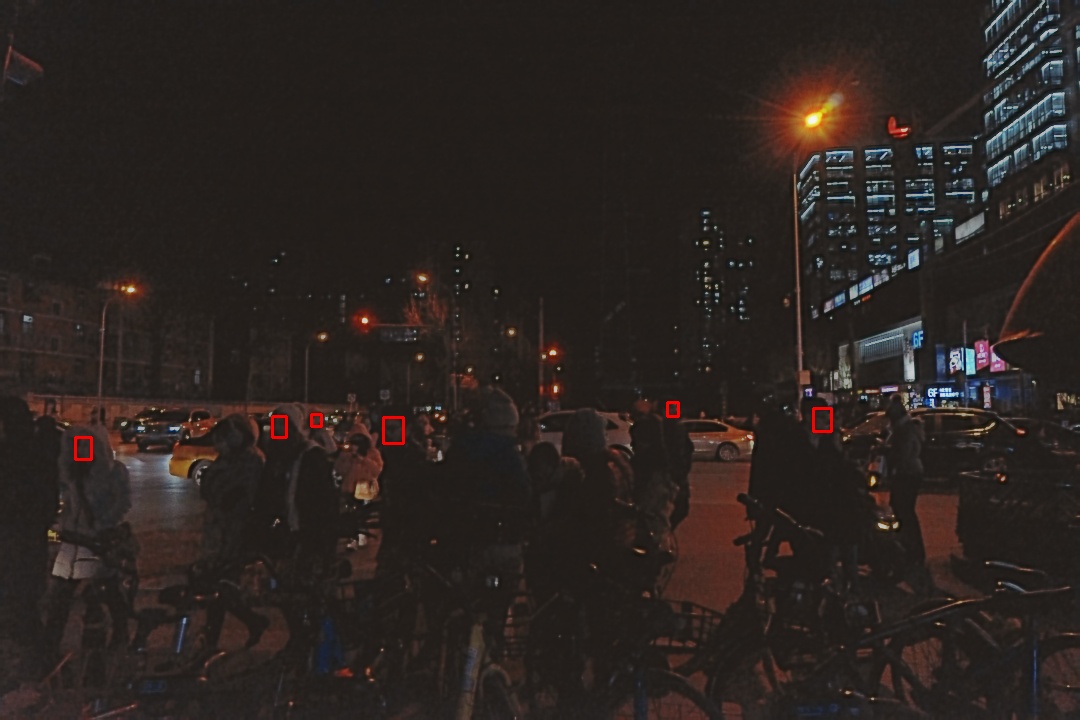}
      \caption{PairLIE}
      \label{fig:darkface-1-e}
  \end{subfigure}
  \hfill
  \begin{subfigure}{0.196\linewidth}
    \includegraphics[width=1\linewidth]{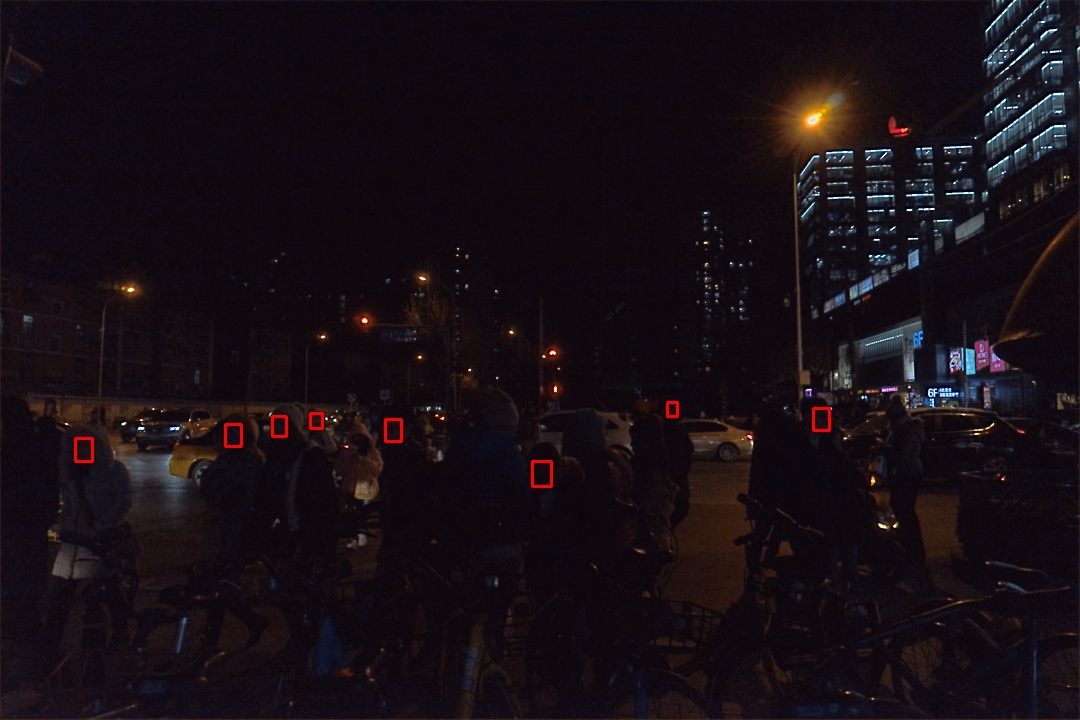}
      \caption{Ours}
      \label{fig:darkface-1-j}
  \end{subfigure}
\setlength{\abovecaptionskip}{-3pt} 
\setlength{\belowcaptionskip}{-3pt}
  \caption{{\colorRevision A visual comparison of enhancement and face detection results on DARKFACE. Please zoom in for better view.} }
  \label{fig:darkface-1}
\end{figure*}

\subsection{Results on LOL-v1}

We first evaluate our method on LOL-v1~\citep{Chen2018Retinex} by comparing its LLIE performance with other existing methods. The quantitative results are shown in Table.~\ref{tab:lolv1}. We can see that our model has the best performance in terms of MSE and PSNR. Furthermore, our method has the second-best value of SSIM and LPIPS. We also train another small version of our model by reducing internal (64 to 4) and output channels (6 to 2). The performance of our small model can still outperform most existing methods. Note that the number of parameters in our model is only $1/31$ of PairLIE~\citep{fu2023pairlie}, $1/251$ of NightEnhance~\citep{jin2022unsupervised} and $1/202$ of KinD~\citep{zhang2019kindling}.

We then give several examples for the visual comparison among the existing methods and ours in Fig.~\ref{fig:lolv1-2} to Fig.~\ref{fig:lolv1-3}. We can see in Fig.~\ref{fig:lolv1-2} that only our model restores the contents in the bookshelf. In addition, in the pure white region in the bottom right corner, other methods are ``afraid'' of enhancing the white pixels, which are meant to be overexposed in photography, towards overexposure. However, our proposed mask function $\mathcal{M}$ in $\mathcal{L}_{RD}$ deals with the dynamic range overflow and enables the successful enhancement of the white pixels. In Fig.~\ref{fig:lolv1-6}, most methods cannot recover the brightness of the imaging scene properly except NightEnhance~\citep{jin2022unsupervised} and ours. Compared with NightEnhance~\citep{jin2022unsupervised}, our method better enhances the shadow inside the cupboard. In Fig.~\ref{fig:lolv1-3}, still only ours and NightEnhance~\citep{jin2022unsupervised} can achieve the closest results to the ground-truth. However, our method restores the color of the wall with more vivid fidelity.

\subsection{Results on LOL-v2}
\vspace{-1mm}
\begin{table}[htbp]
\tabcolsep=0.8mm
  \small
  \centering
  \caption{{\colorRevision The quantitative results of LLIE methods on LOL-v2~\citep{yang2021sparse} in terms of MSE [$\times 10^3$], PSNR in dB, SSIM~\citep{wang2004image} and LPIPS~\citep{zhang2018unreasonable}. }The best, second best and third best results are in {\color{red} red}, {\color{blue} blue} and \textcolor[RGB]{84,181,53}{ green} respectively.}
  \vspace{-3mm}
    \begin{tabular}{lcccc}
    \toprule
          & MSE$\downarrow$ & PSNR$\uparrow$ & SSIM$\uparrow$ & LPIPS$\downarrow$ \\
    \midrule
    Input & 7.634 & 9.718 & 0.190 & 0.333 \\
    LIME~\citep{guo2016lime}  & 1.484 & 15.240 & 0.470 & 0.360 \\
    RetinexNet~\citep{Chen2018Retinex} & 1.719 & 15.470 & 0.560 & 0.421 \\
    KinD~\citep{zhang2019kindling}  & 4.029 & 14.740 & 0.641 & 0.302 \\
    RUAS~\citep{liu2021retinex}  & 2.540 & 15.330 & 0.520 & 0.322 \\
    DRBN~\citep{yang2021band}  & 1.843 & 19.600 &   \textcolor[RGB]{84,181,53}{ 0.764} & {\color{blue} 0.246} \\
    EnlightenGAN~\citep{jiang2021enlightengan} & 1.209 & 18.230 & 0.610 & 0.309 \\
    RRDNet~\citep{zhu2020zero} & 3.594 & 14.850 & 0.560 & 0.265 \\
    RetinexDIP~\citep{zhao2021retinexdip} & 3.157 & 14.513 & 0.546 & 0.274 \\
    ZeroDCE~\citep{guo2020zero} & 1.777 & 18.059 & 0.605 & 0.298 \\
    ZeroDCE++~\citep{li2021learning} & 1.569 & 18.491 & 0.617 & 0.290 \\
    SCI~\citep{ma2022toward}   & 2.132 & 17.304 & 0.565 & 0.286 \\
    NightEnhance~\citep{jin2022unsupervised} & {\color{red} 0.819} & {\color{blue} 20.850} &0.724 & 0.329 \\
    PairLIE~\citep{fu2023pairlie} & \textcolor[RGB]{84,181,53}{ 0.916} & 19.885 & {\color{blue}0.778} & 0.282 \\
    \midrule
    Ours  & {\color{blue} 0.827} & {\color{red} 21.362} & {\color{red} 0.795} & {\color{red} 0.225} \\
    Ours (small) & 1.064 & \textcolor[RGB]{84,181,53}{ 20.631} & 0.739 & \textcolor[RGB]{84,181,53}{ 0.252} \\
    \bottomrule
    \end{tabular}%
  \label{tab:lolv2}%
  \vspace{-4mm}
\end{table}%

Similar to LOL-v1~\citep{Chen2018Retinex}, we evaluate our method on LOL-v2~\citep{yang2021sparse}. The quantitative results are shown in Table.~\ref{tab:lolv2}. We can see that our model has the best performance in terms of PSNR, SSIM, and LPIPS. 
The performance of our model with smaller size is the third best in terms of all the metrics.

We also provide several illustrative examples for visually comparing our method and others in Fig.~\ref{fig:lolv2-2} to Fig.~\ref{fig:lolv2-6} sampled from LOL-v2~\citep{yang2021sparse}. The non-uniform brightness in the grandstand of stadium is challenging for all LLIE methods. Among them, ours and NightEnhance~\citep{jin2022unsupervised} can generate the closest results to the ground-truth. However, NightEnhance~\citep{jin2022unsupervised} makes the red flag fade, while ours restore the color better. In Fig~\ref{fig:lolv2-4}, only our solution yilelds the correct brightness of the white wall. In Fig.~\ref{fig:lolv2-3} and Fig.~\ref{fig:lolv2-6}, our method gives the closest results to the ground-truth.

\begin{figure*}[htbp]
\scriptsize
  \centering
  \begin{subfigure}{0.196\linewidth}
    \includegraphics[width=1\linewidth]{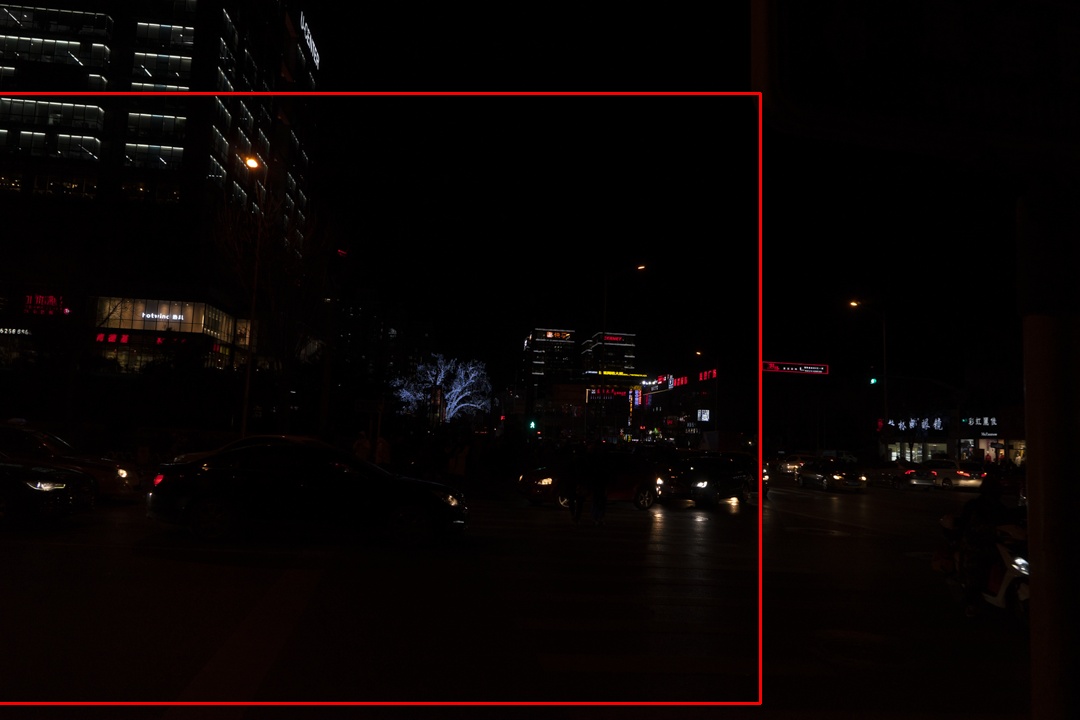}
      \caption{Input}
      \label{fig:darkface-2-a}
  \end{subfigure}
  \hfill
  \begin{subfigure}{0.196\linewidth}
    \includegraphics[width=1\linewidth]{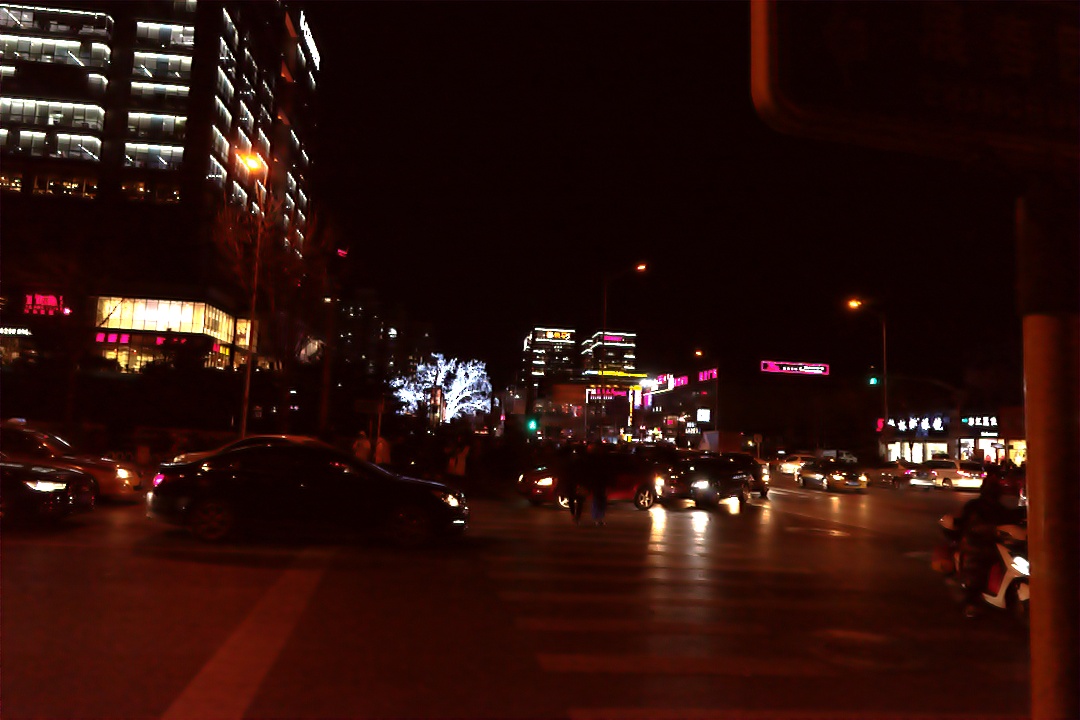}
      \caption{RUAS}
      \label{fig:darkface-2-b}
  \end{subfigure}
  \hfill
  \begin{subfigure}{0.196\linewidth}
    \includegraphics[width=1\linewidth]{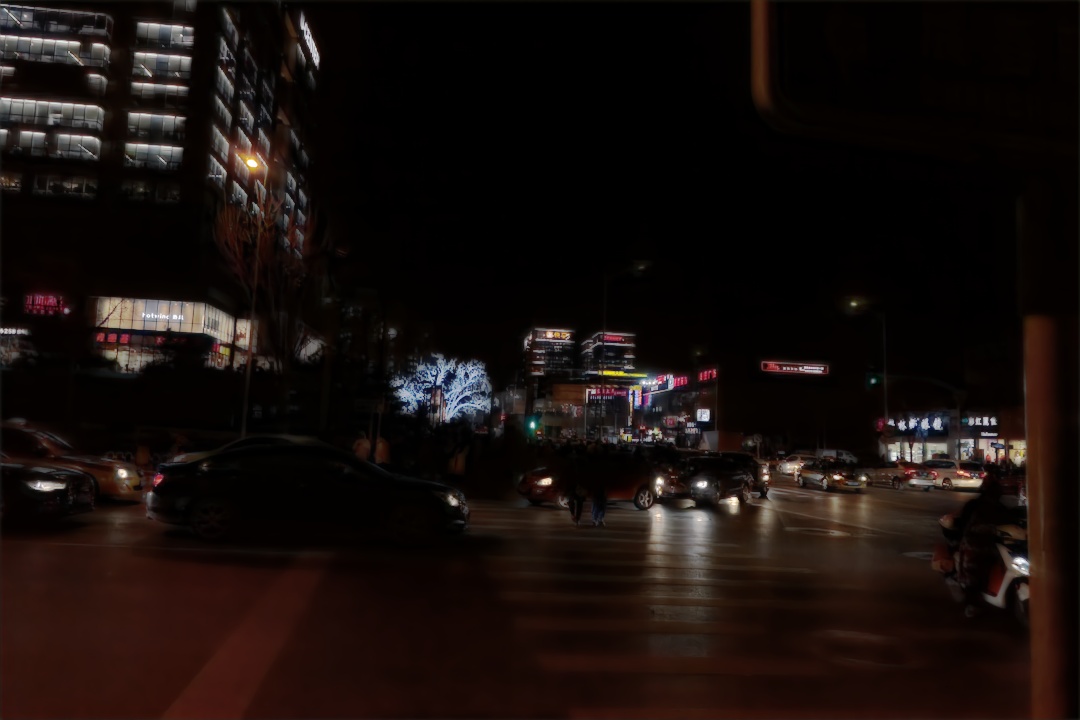}
      \caption{KinD}
      \label{fig:darkface-2-c}
  \end{subfigure}
  \hfill
  \begin{subfigure}{0.196\linewidth}
    \includegraphics[width=1\linewidth]{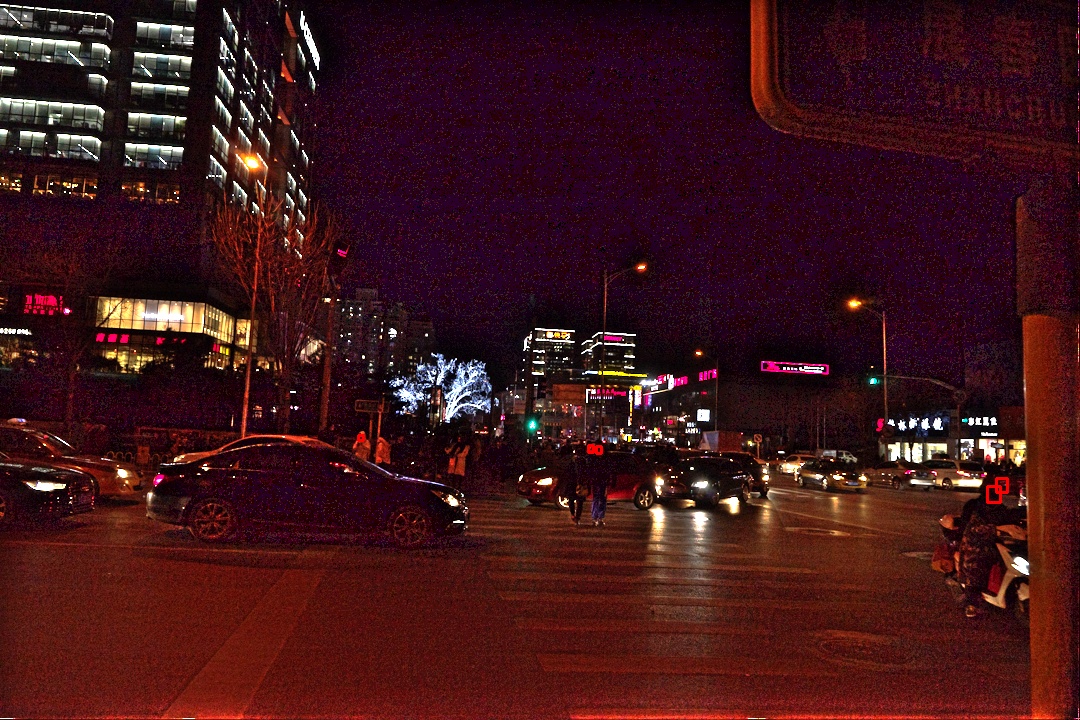}
      \caption{LIME}
      \label{fig:darkface-2-d}
  \end{subfigure}
  \hfill
  \begin{subfigure}{0.196\linewidth}
    \includegraphics[width=1\linewidth]{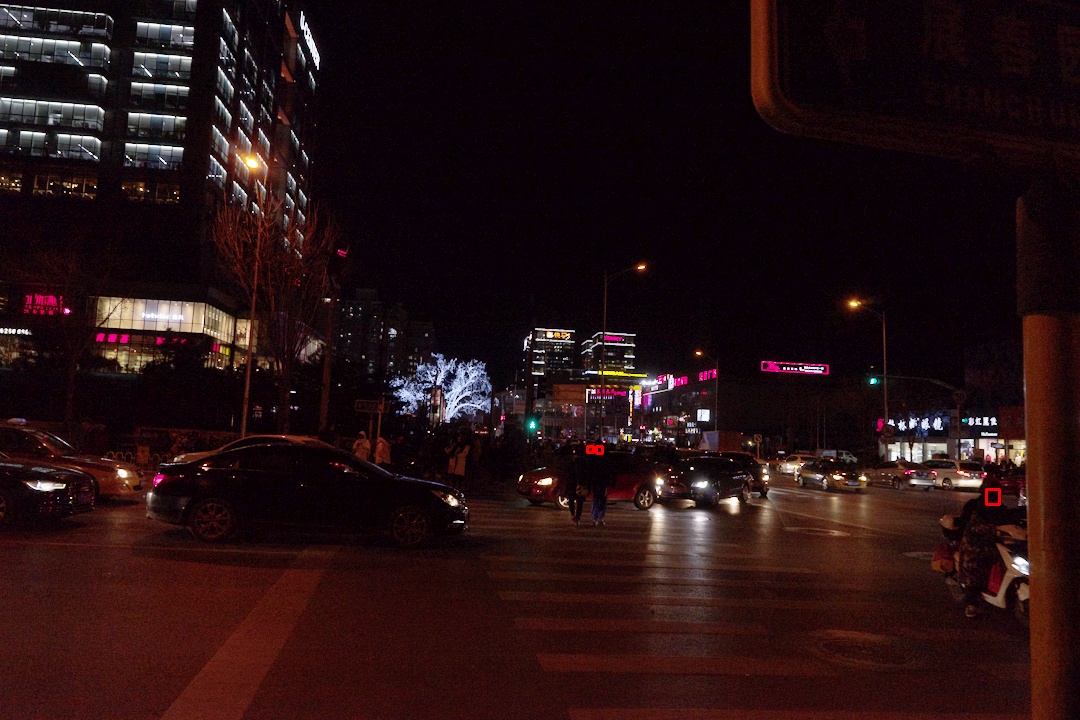}
      \caption{ZeroDCE++}
      \label{fig:darkface-2-f}
  \end{subfigure}
  
  \centering
  \begin{subfigure}{0.196\linewidth}
    \includegraphics[width=1\linewidth]{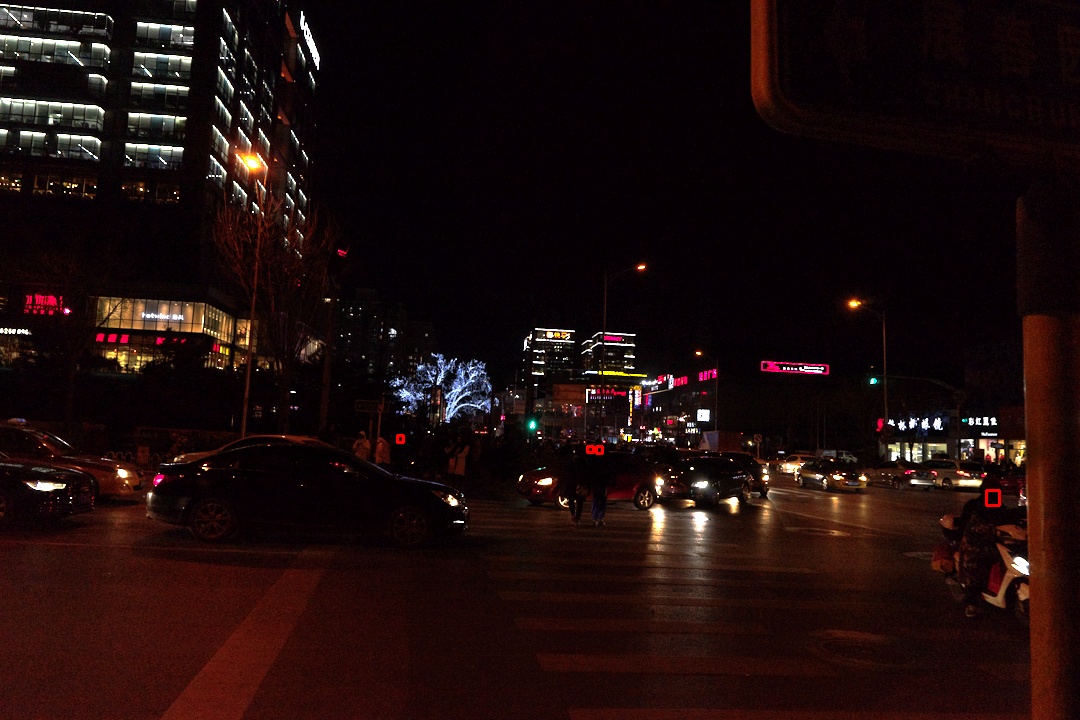}
      \caption{RetinexDIP}
      \label{fig:darkface-2-g}
  \end{subfigure}
  \hfill
  \begin{subfigure}{0.196\linewidth}
    \includegraphics[width=1\linewidth]{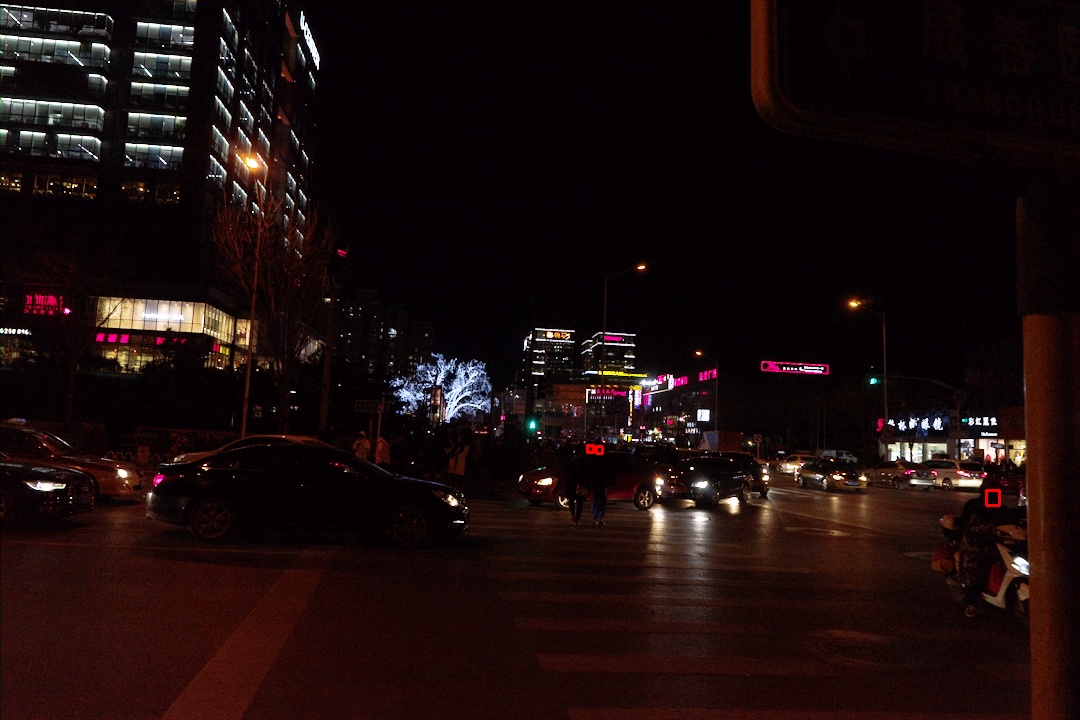}
      \caption{SCI}
      \label{fig:darkface-2-h}
  \end{subfigure}
  \hfill
  \begin{subfigure}{0.196\linewidth}
    \includegraphics[width=1\linewidth]{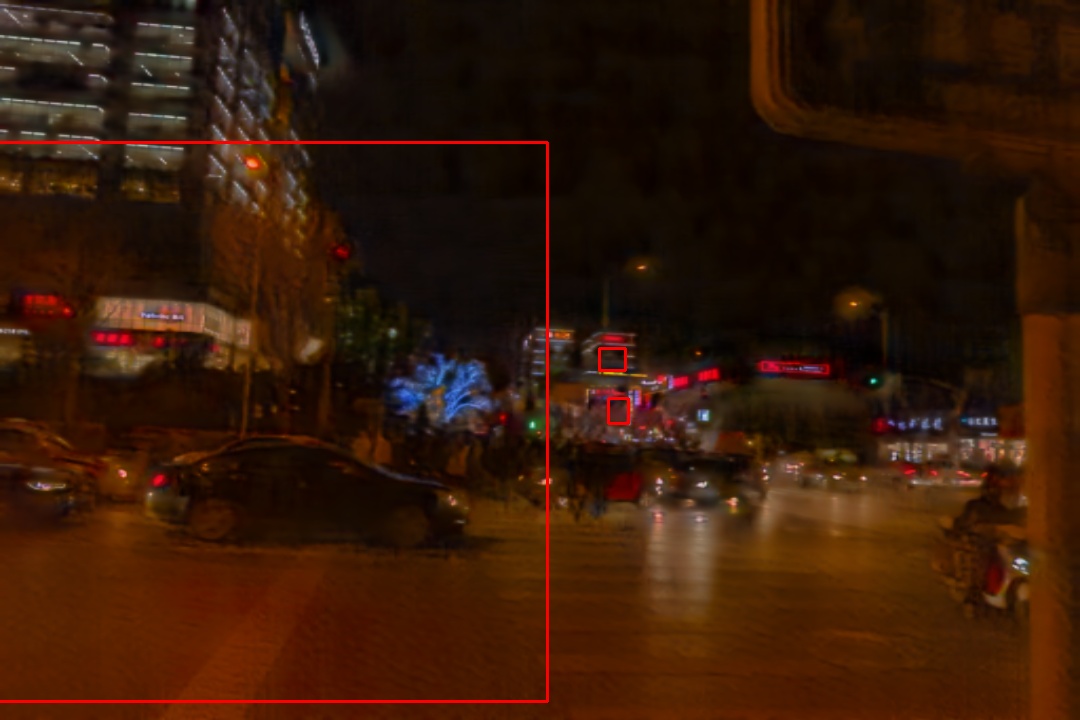}
      \caption{NightEnhance}
      \label{fig:darkface-2-i}
  \end{subfigure}
  \hfill
  \begin{subfigure}{0.196\linewidth}
    \includegraphics[width=1\linewidth]{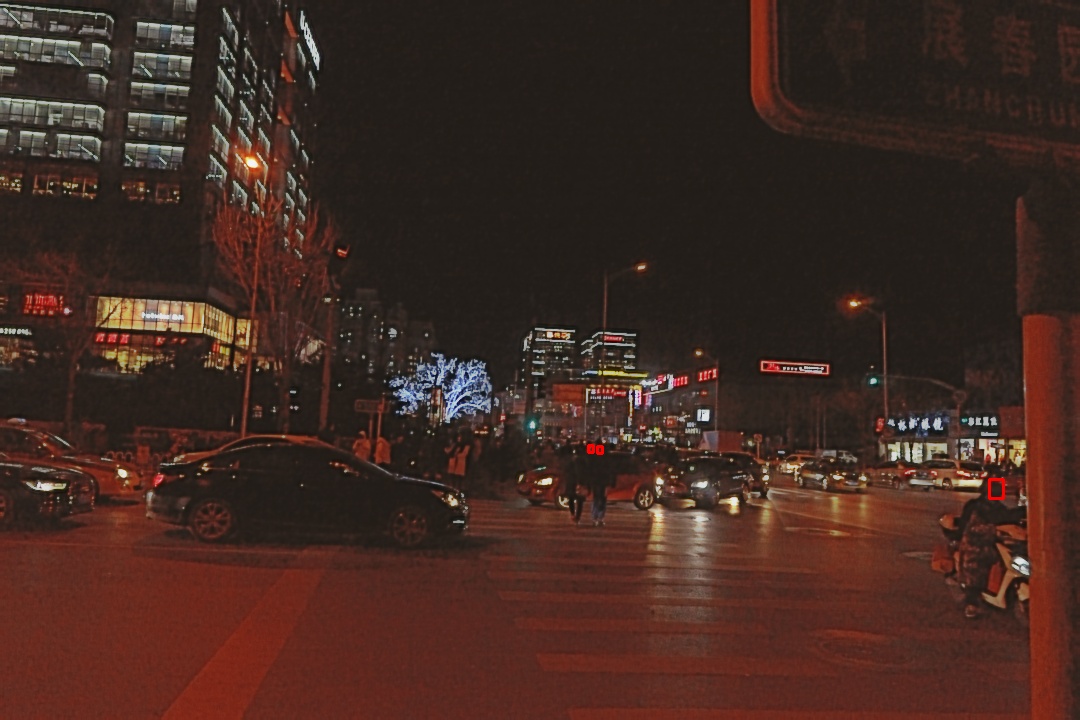}
      \caption{PairLIE}
      \label{fig:darkface-2-e}
  \end{subfigure}
  \hfill
  \begin{subfigure}{0.196\linewidth}
    \includegraphics[width=1\linewidth]{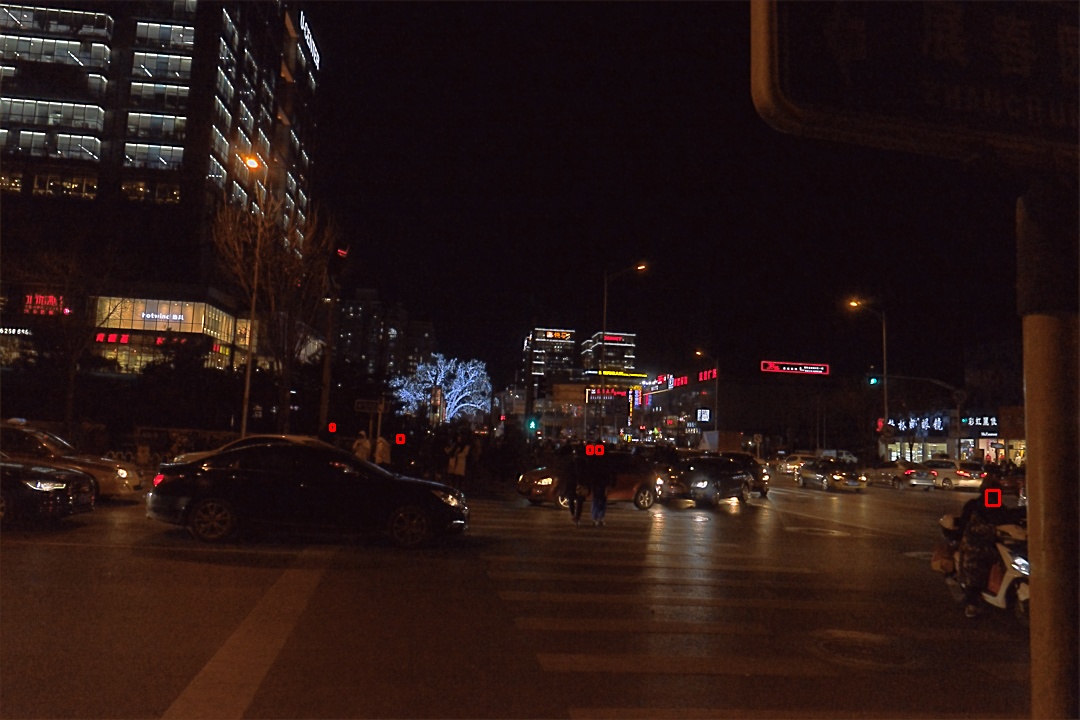}
      \caption{Ours}
      \label{fig:darkface-2-j}
  \end{subfigure}
\setlength{\abovecaptionskip}{-3pt} 
\setlength{\belowcaptionskip}{-3pt}
  \caption{{\colorRevision A visual comparison of enhancement and face detection results on DARKFACE. Please zoom in for better view. }}
  \label{fig:darkface-2}
\end{figure*}
\begin{figure*}[htbp]
  \centering
  \begin{subfigure}{0.196\linewidth}
    \includegraphics[width=1\linewidth]{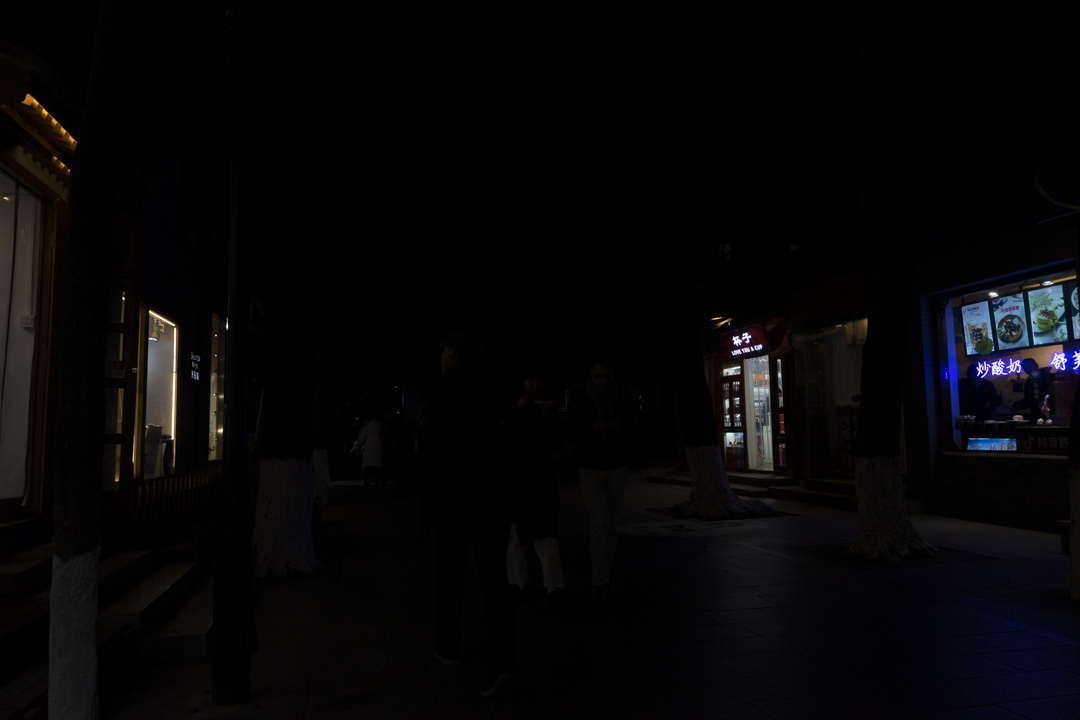}
      \caption{Input}
      \label{fig:darkface-3-a}
  \end{subfigure}
  \hfill
  \begin{subfigure}{0.196\linewidth}
    \includegraphics[width=1\linewidth]{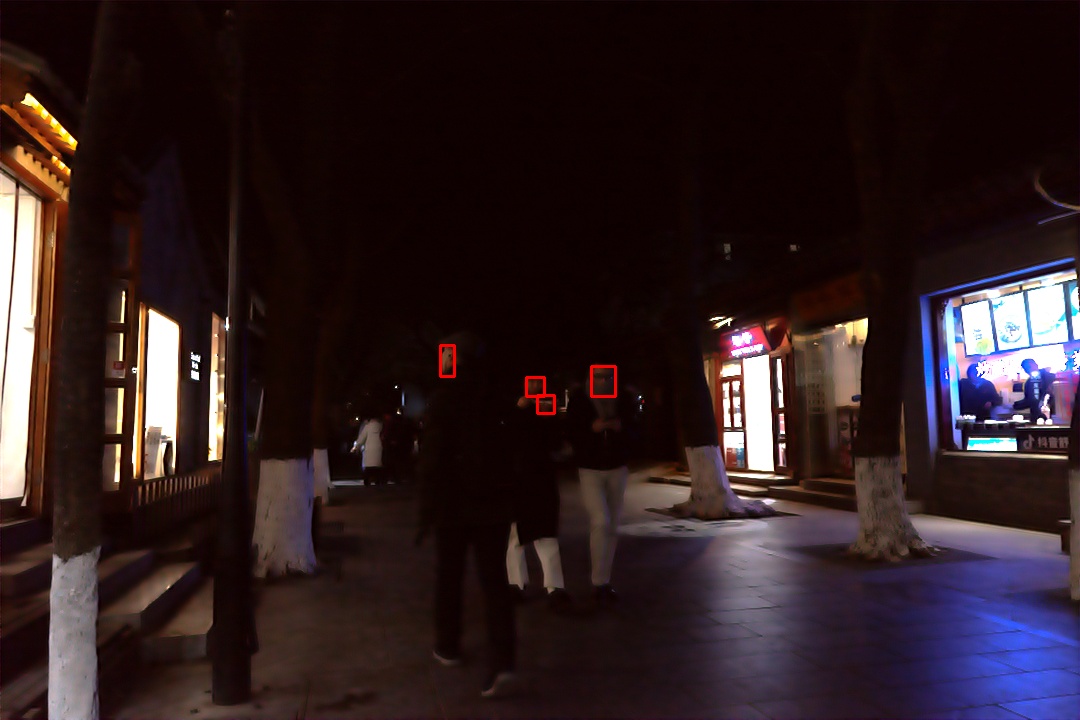}
      \caption{RUAS}
      \label{fig:darkface-3-b}
  \end{subfigure}
  \hfill
  \begin{subfigure}{0.196\linewidth}
    \includegraphics[width=1\linewidth]{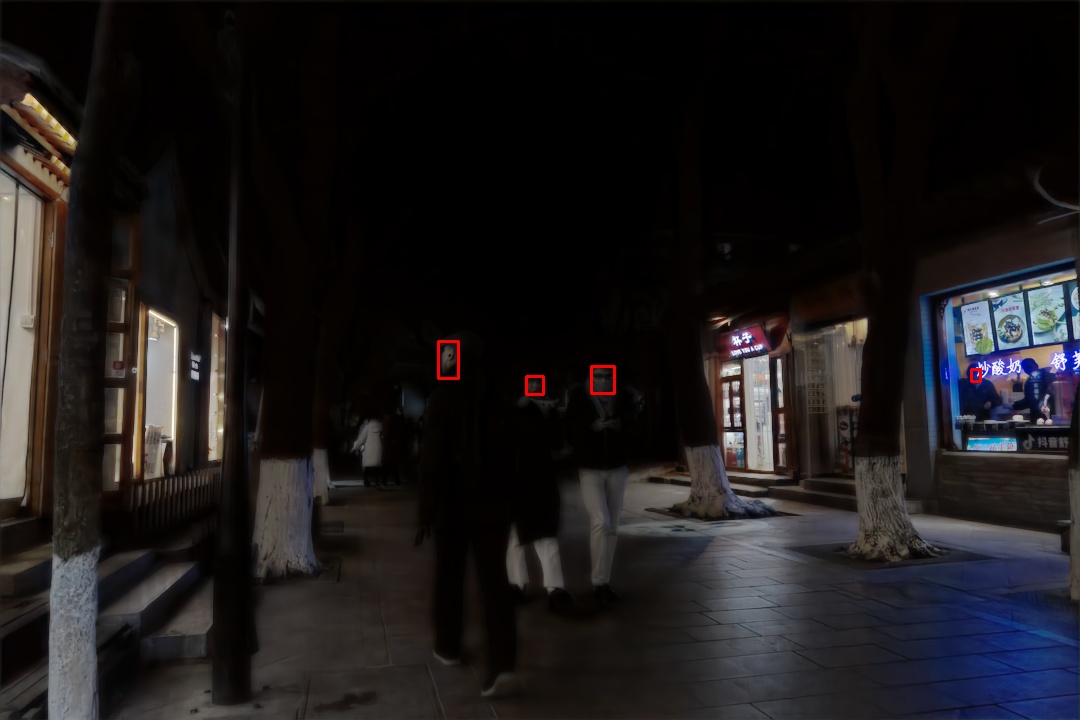}
      \caption{KinD}
      \label{fig:darkface-3-c}
  \end{subfigure}
  \hfill
  \begin{subfigure}{0.196\linewidth}
    \includegraphics[width=1\linewidth]{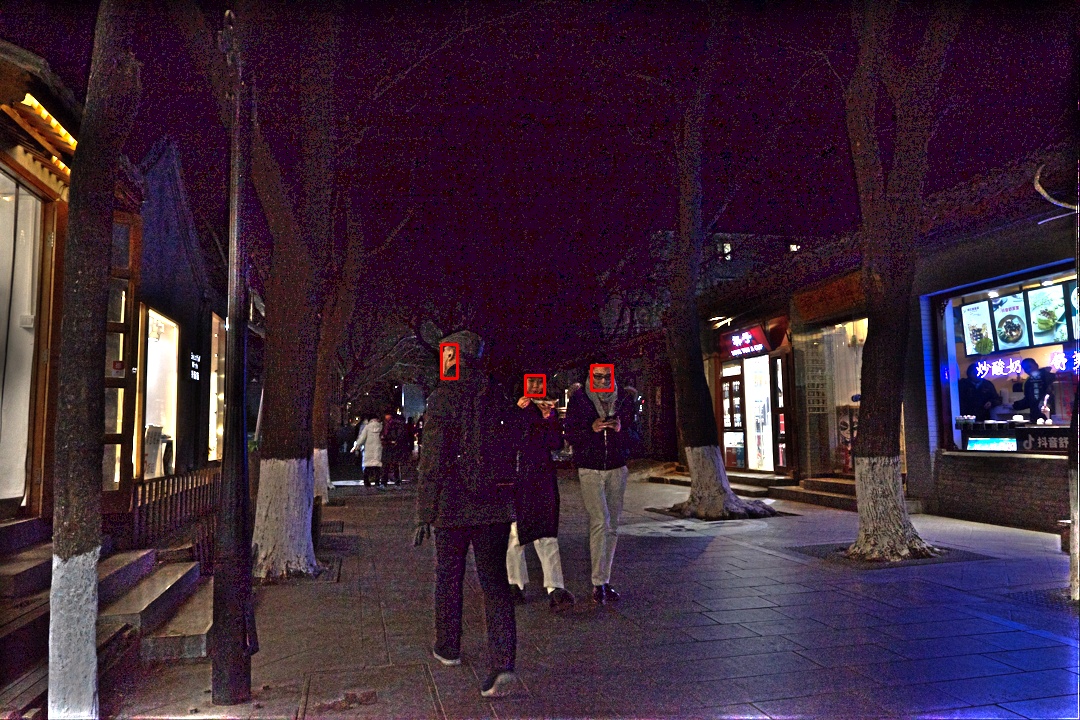}
      \caption{LIME}
      \label{fig:darkface-3-d}
  \end{subfigure}
  \hfill
  \begin{subfigure}{0.196\linewidth}
    \includegraphics[width=1\linewidth]{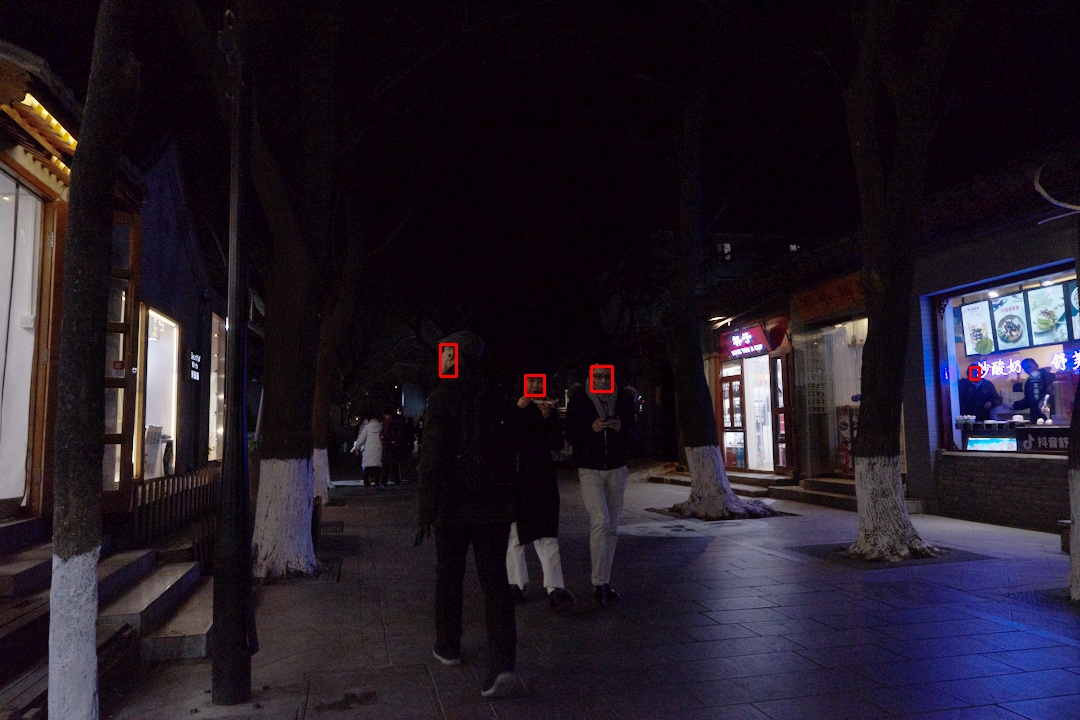}
      \caption{ZeroDCE++}
      \label{fig:darkface-3-f}
  \end{subfigure}
  
  \centering
  \begin{subfigure}{0.196\linewidth}
    \includegraphics[width=1\linewidth]{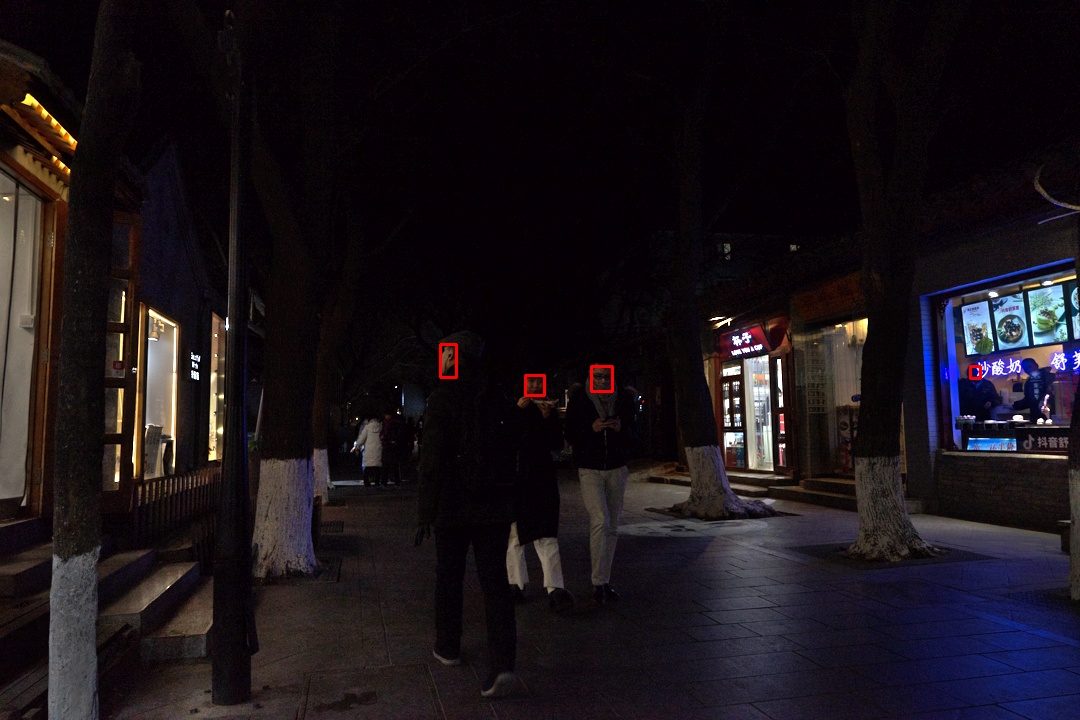}
      \caption{RetinexDIP}
      \label{fig:darkface-3-g}
  \end{subfigure}
  \hfill
  \begin{subfigure}{0.196\linewidth}
    \includegraphics[width=1\linewidth]{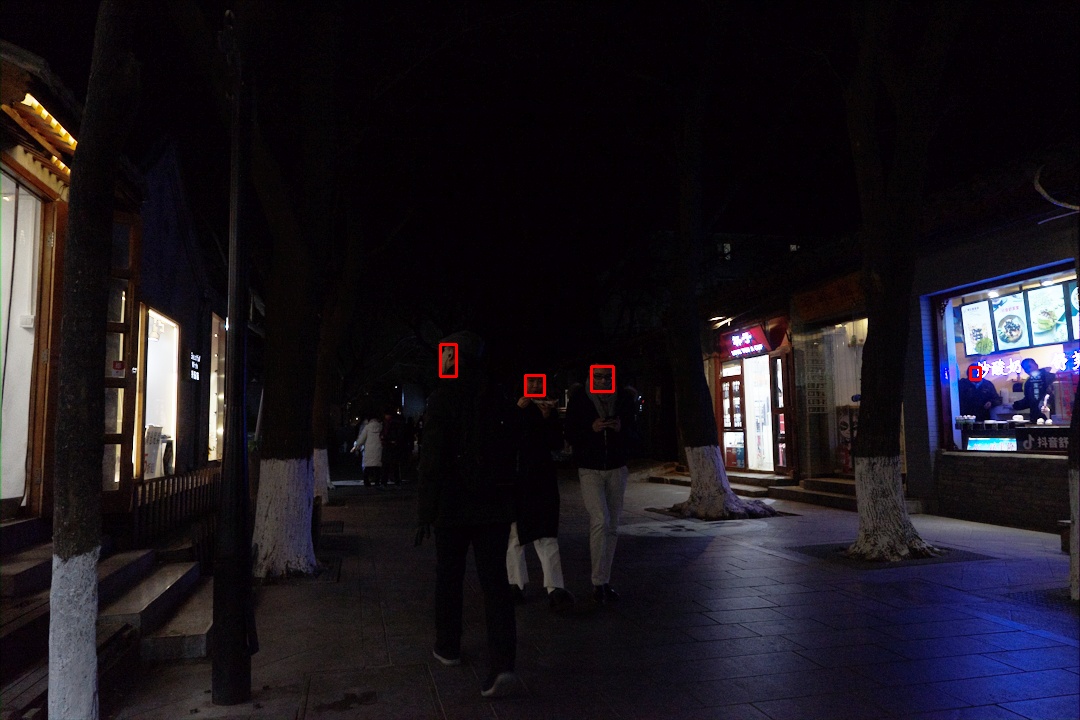}
      \caption{SCI}
      \label{fig:darkface-3-h}
  \end{subfigure}
  \hfill
  \begin{subfigure}{0.196\linewidth}
    \includegraphics[width=1\linewidth]{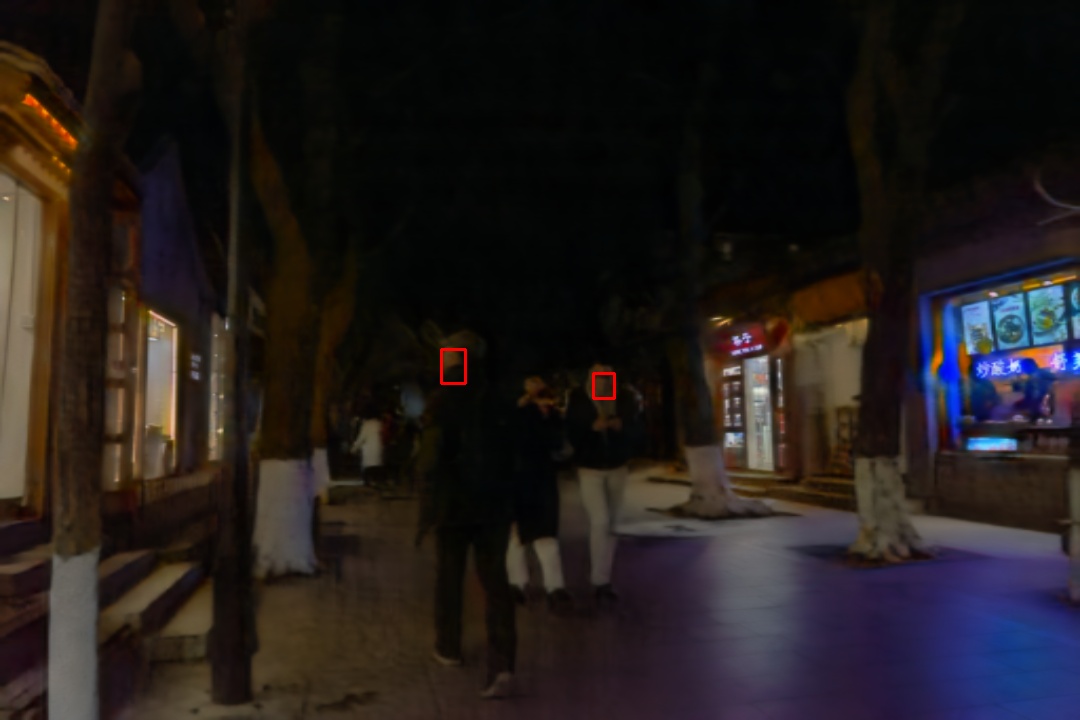}
      \caption{NightEnhance}
      \label{fig:darkface-3-i}
  \end{subfigure}
  \hfill
  \begin{subfigure}{0.196\linewidth}
    \includegraphics[width=1\linewidth]{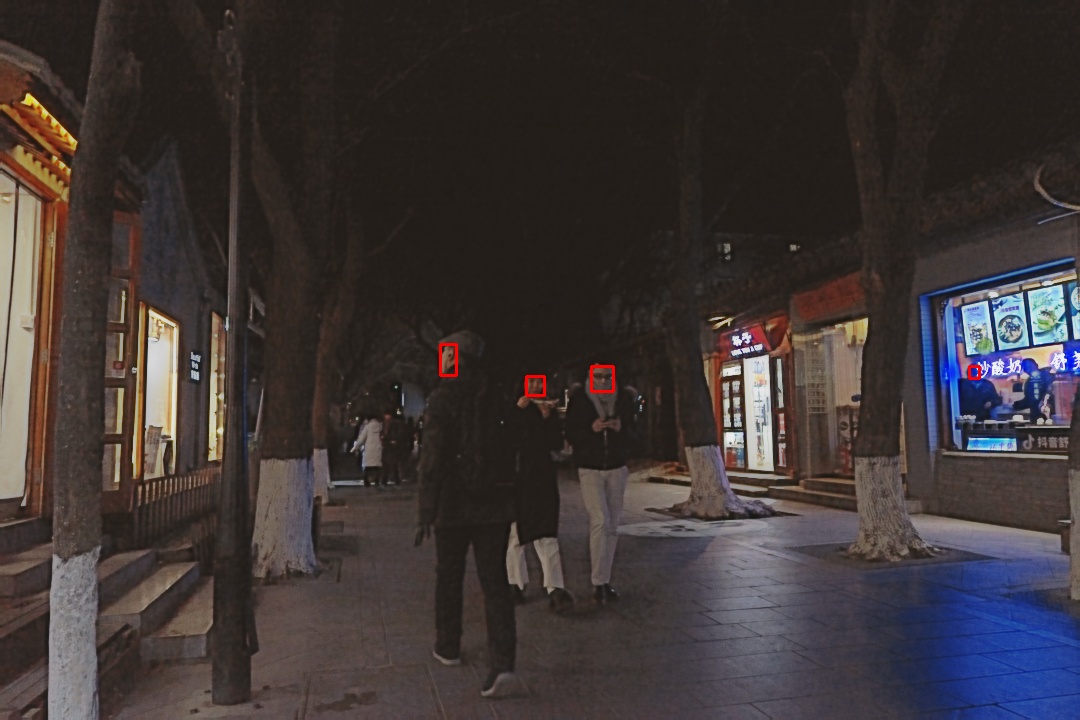}
      \caption{PairLIE}
      \label{fig:darkface-3-e}
  \end{subfigure}
  \hfill
  \begin{subfigure}{0.196\linewidth}
    \includegraphics[width=1\linewidth]{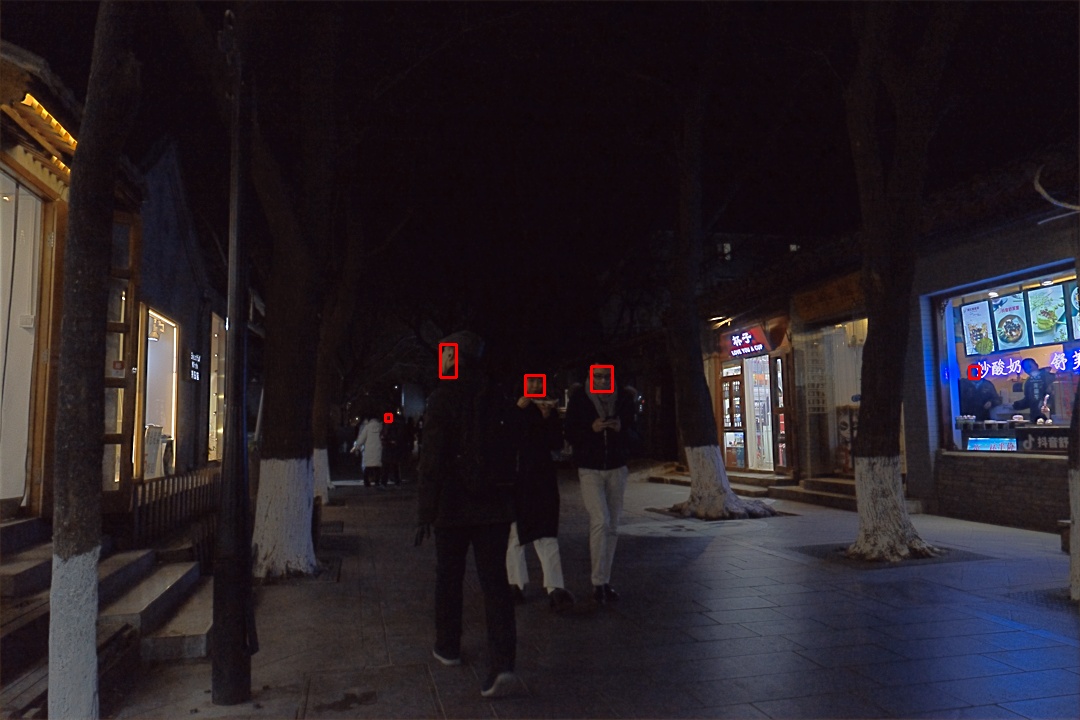}
      \caption{Ours}
      \label{fig:darkface-3-j}
  \end{subfigure}
\setlength{\abovecaptionskip}{-3pt} 
\setlength{\belowcaptionskip}{-3pt}
  \caption{{\colorRevision A visual comparison of enhancement and face detection results on DARKFACE. Please zoom in for better view. }}
  \label{fig:darkface-3}
\end{figure*}
\begin{figure*}[htbp]
  \centering
  \begin{subfigure}{0.196\linewidth}
    \includegraphics[width=1\linewidth]{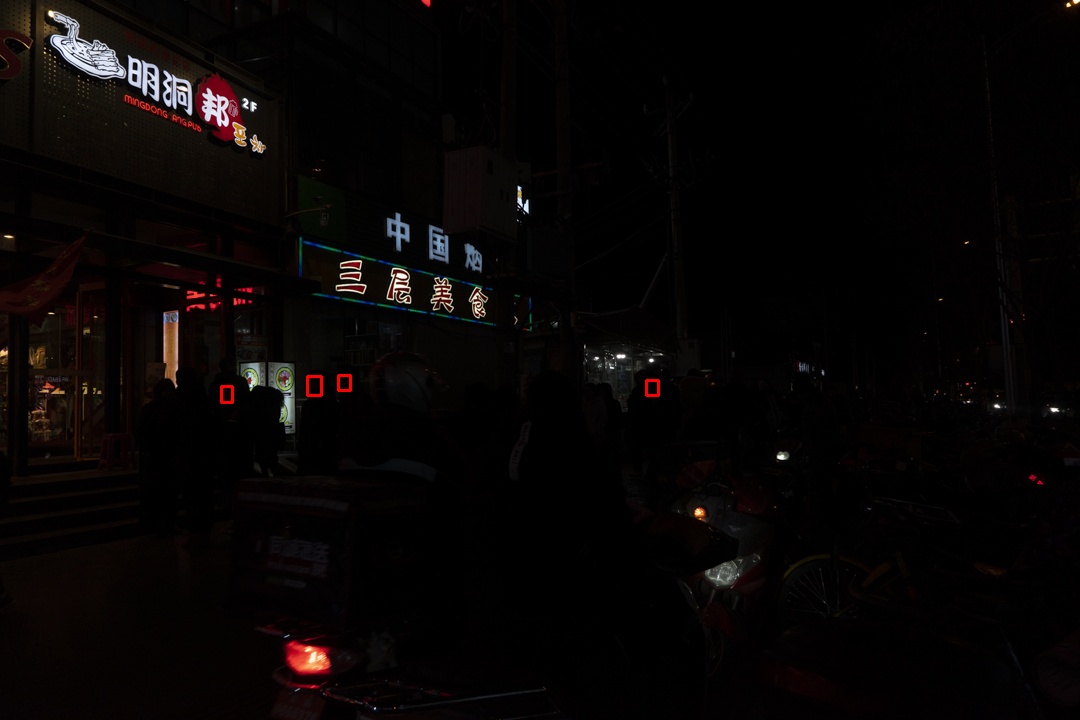}
      \caption{Input}
      \label{fig:darkface-4-a}
  \end{subfigure}
  \hfill
  \begin{subfigure}{0.196\linewidth}
    \includegraphics[width=1\linewidth]{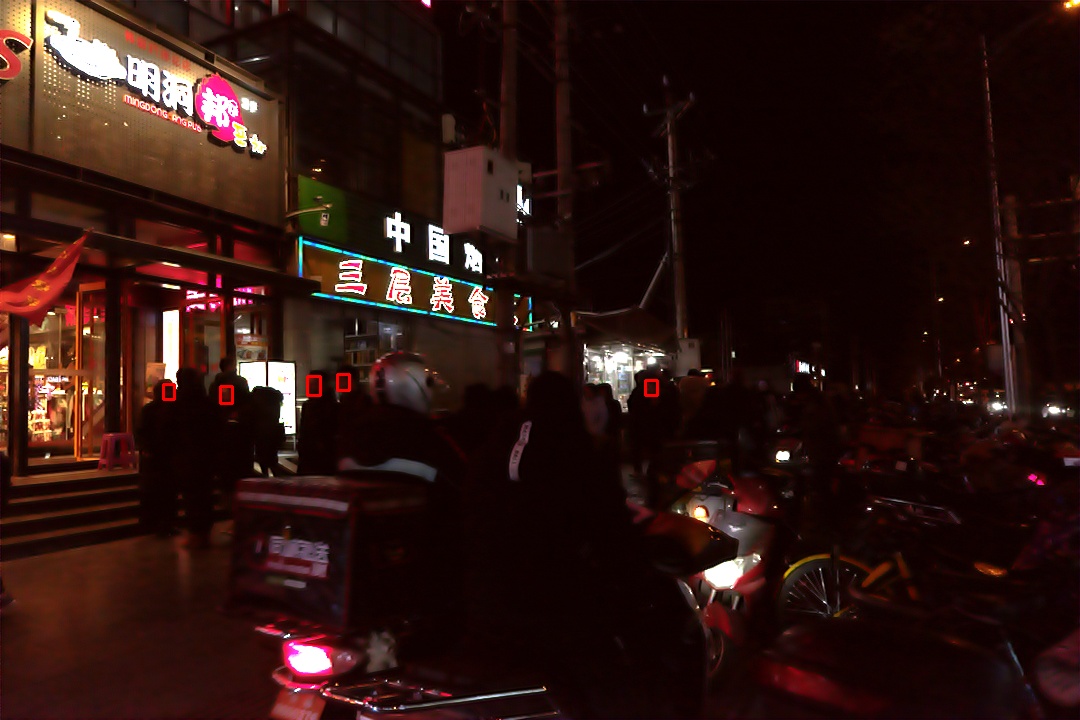}
      \caption{RUAS}
      \label{fig:darkface-4-b}
  \end{subfigure}
  \hfill
  \begin{subfigure}{0.196\linewidth}
    \includegraphics[width=1\linewidth]{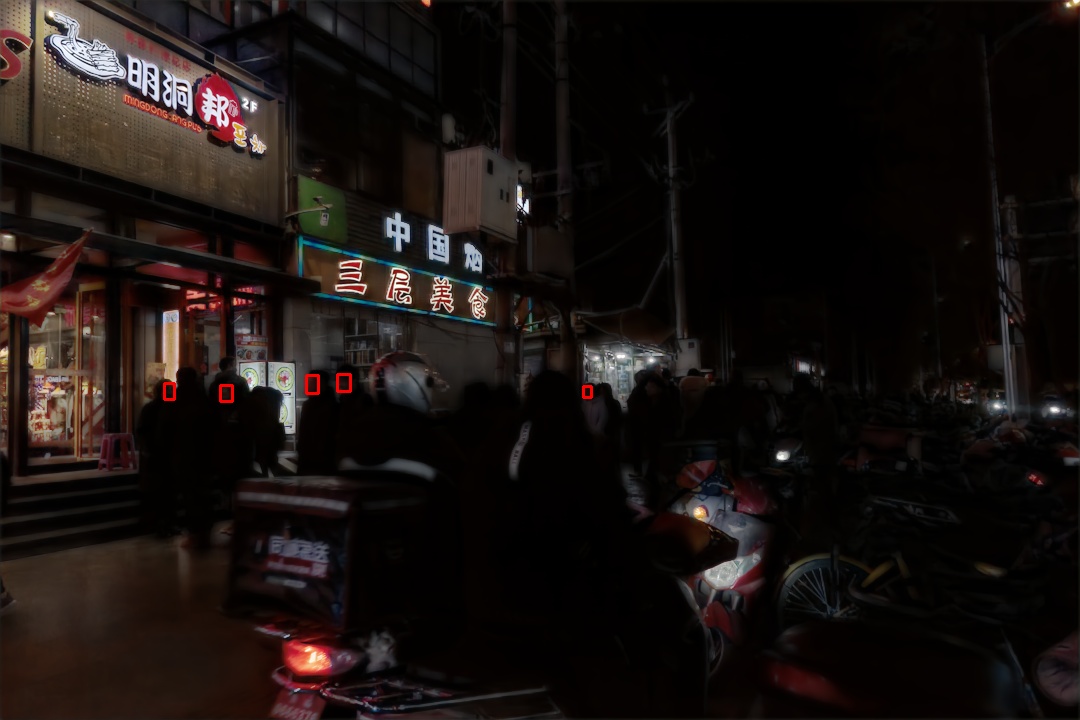}
      \caption{KinD}
      \label{fig:darkface-4-c}
  \end{subfigure}
  \hfill
  \begin{subfigure}{0.196\linewidth}
    \includegraphics[width=1\linewidth]{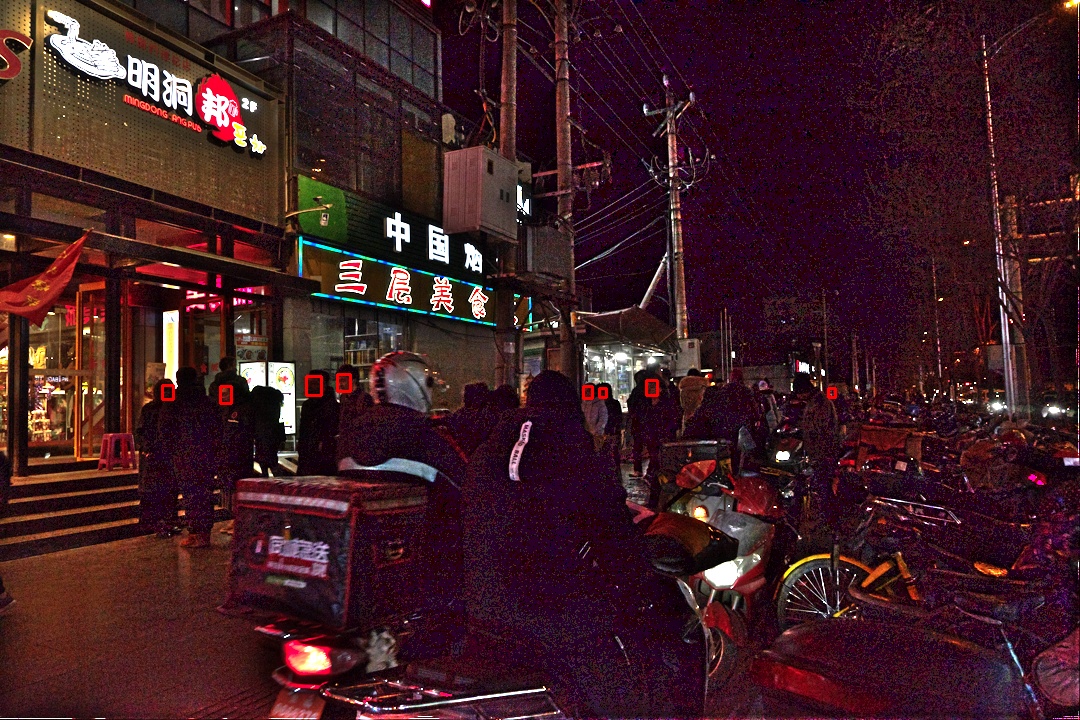}
      \caption{LIME}
      \label{fig:darkface-4-d}
  \end{subfigure}  
  \hfill
  \begin{subfigure}{0.196\linewidth}
    \includegraphics[width=1\linewidth]{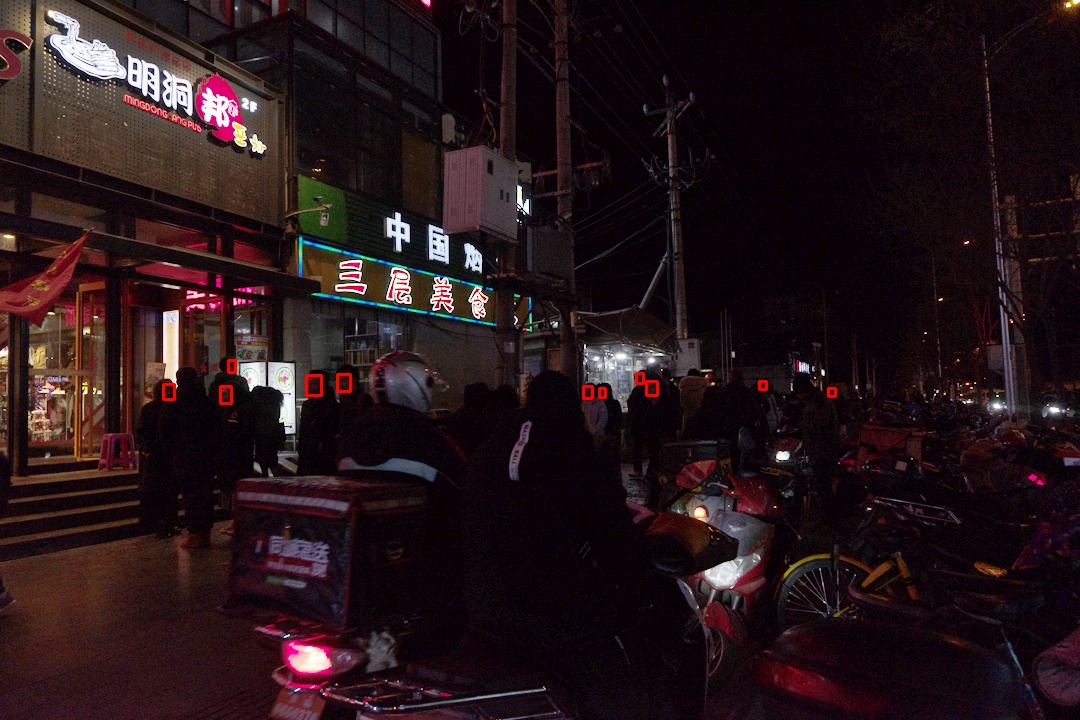}
      \caption{ZeroDCE++}
      \label{fig:darkface-4-f}
  \end{subfigure}
  
  \centering
  \begin{subfigure}{0.196\linewidth}
    \includegraphics[width=1\linewidth]{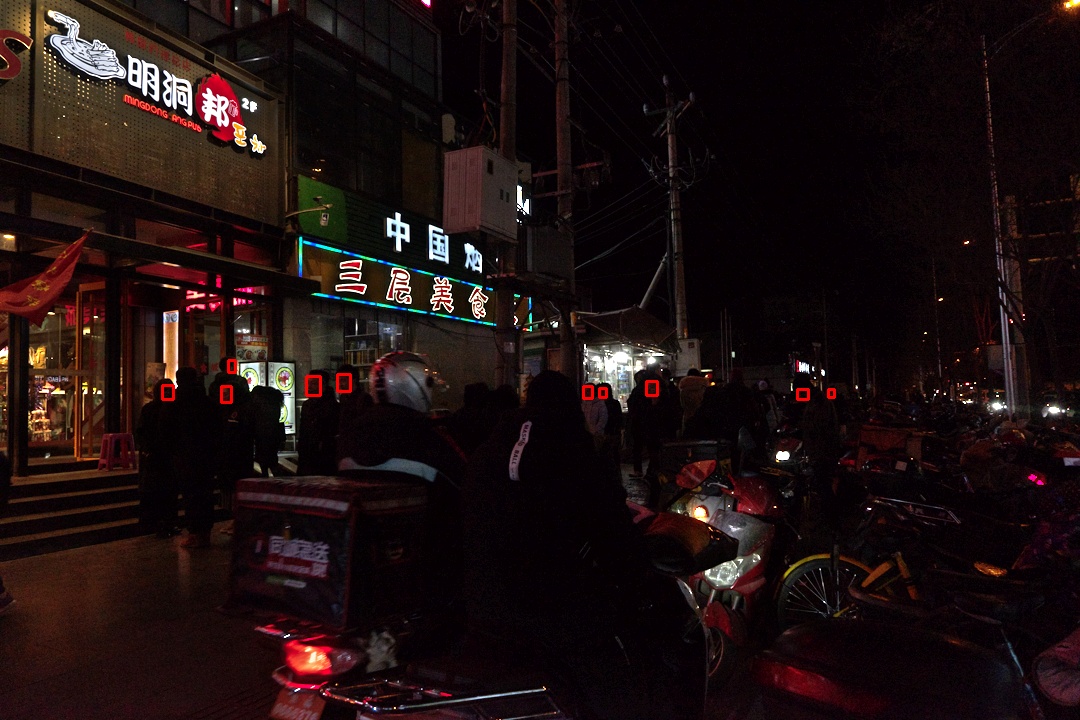}
      \caption{RetinexDIP}
      \label{fig:darkface-4-g}
  \end{subfigure}
  \hfill
  \begin{subfigure}{0.196\linewidth}
    \includegraphics[width=1\linewidth]{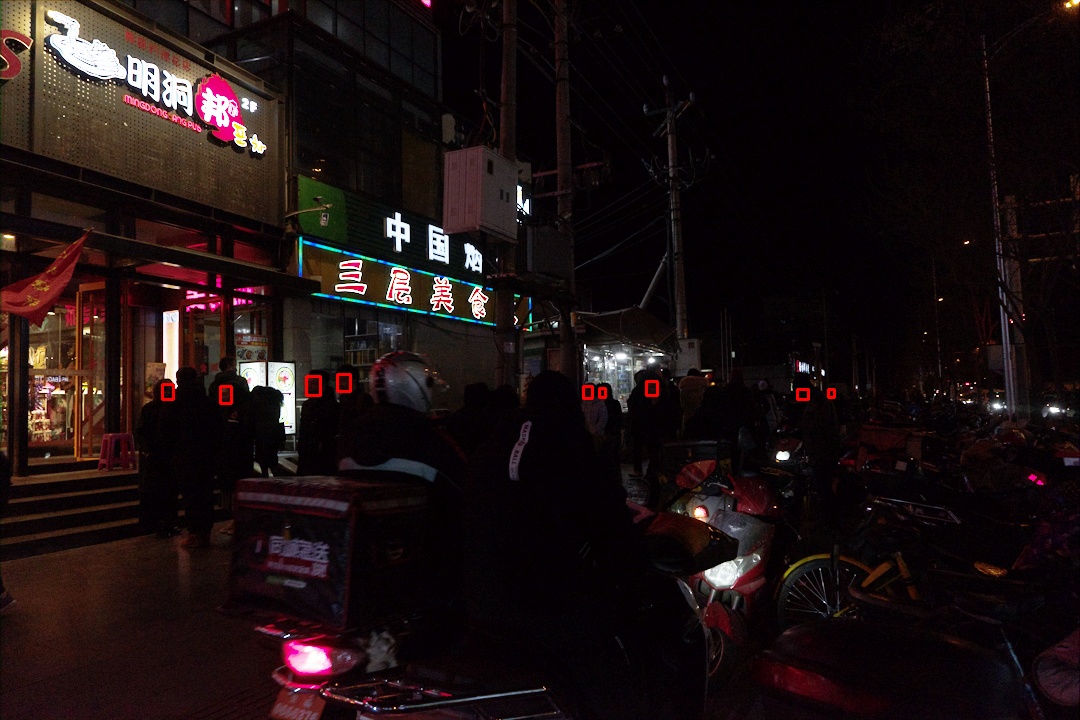}
      \caption{SCI}
      \label{fig:darkface-4-h}
  \end{subfigure}
  \hfill
  \begin{subfigure}{0.196\linewidth}
    \includegraphics[width=1\linewidth]{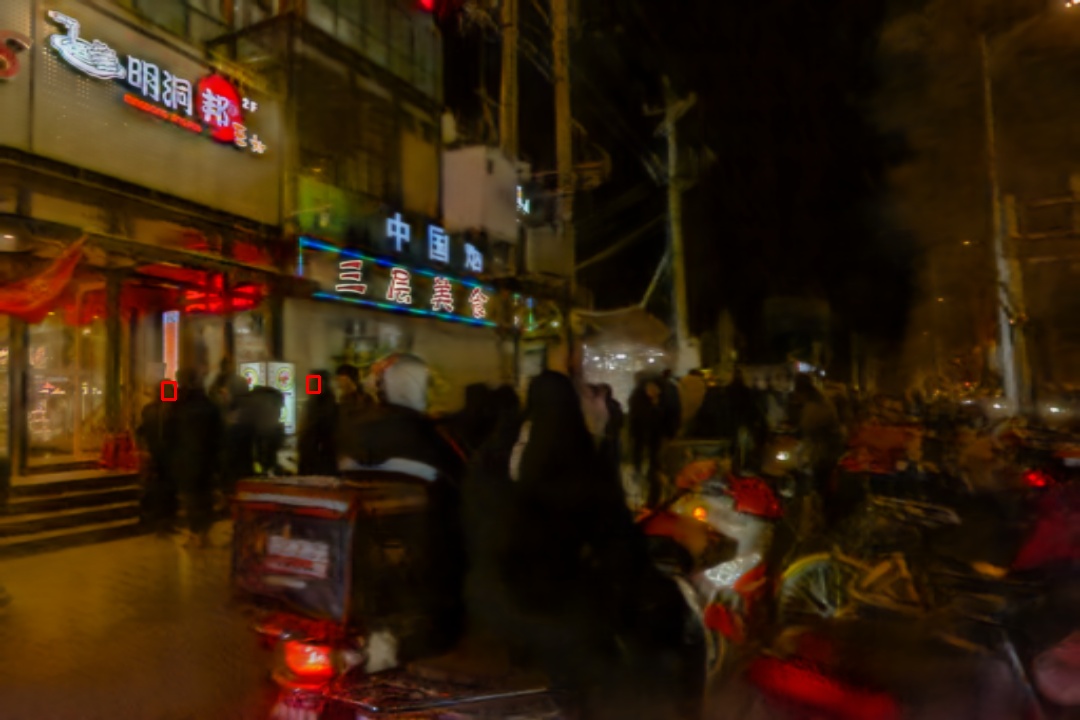}
      \caption{NightEnhance}
      \label{fig:darkface-4-i}
  \end{subfigure}
  \hfill
  \begin{subfigure}{0.196\linewidth}
    \includegraphics[width=1\linewidth]{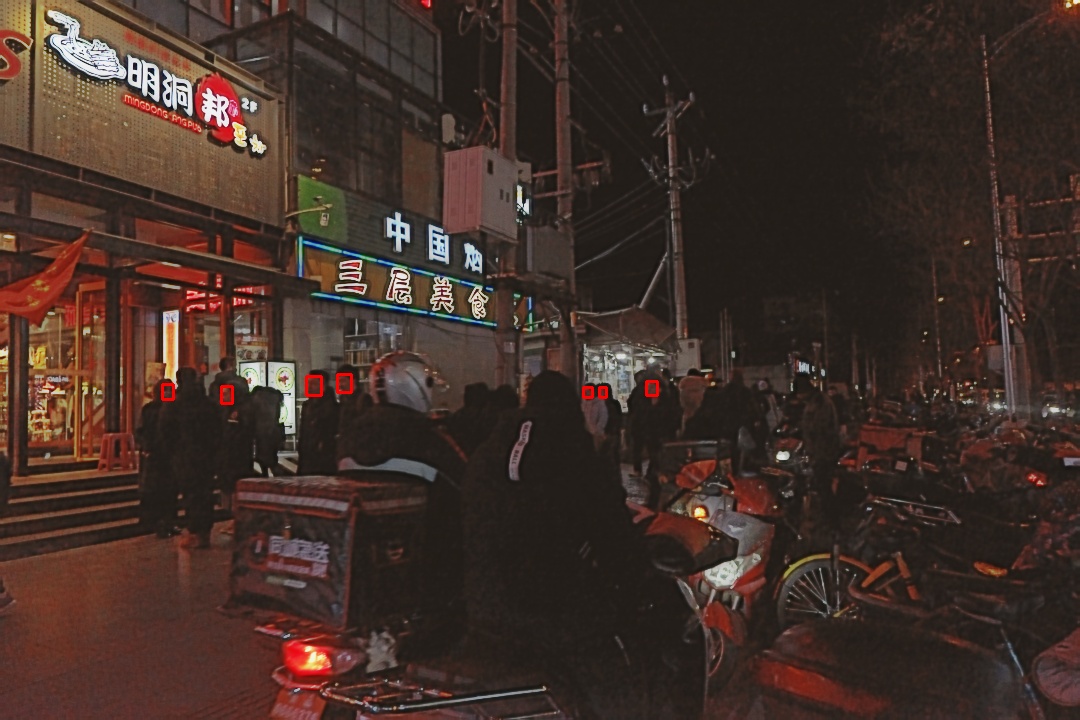}
      \caption{PairLIE}
      \label{fig:darkface-4-e}
  \end{subfigure}
  \hfill
  \begin{subfigure}{0.196\linewidth}
    \includegraphics[width=1\linewidth]{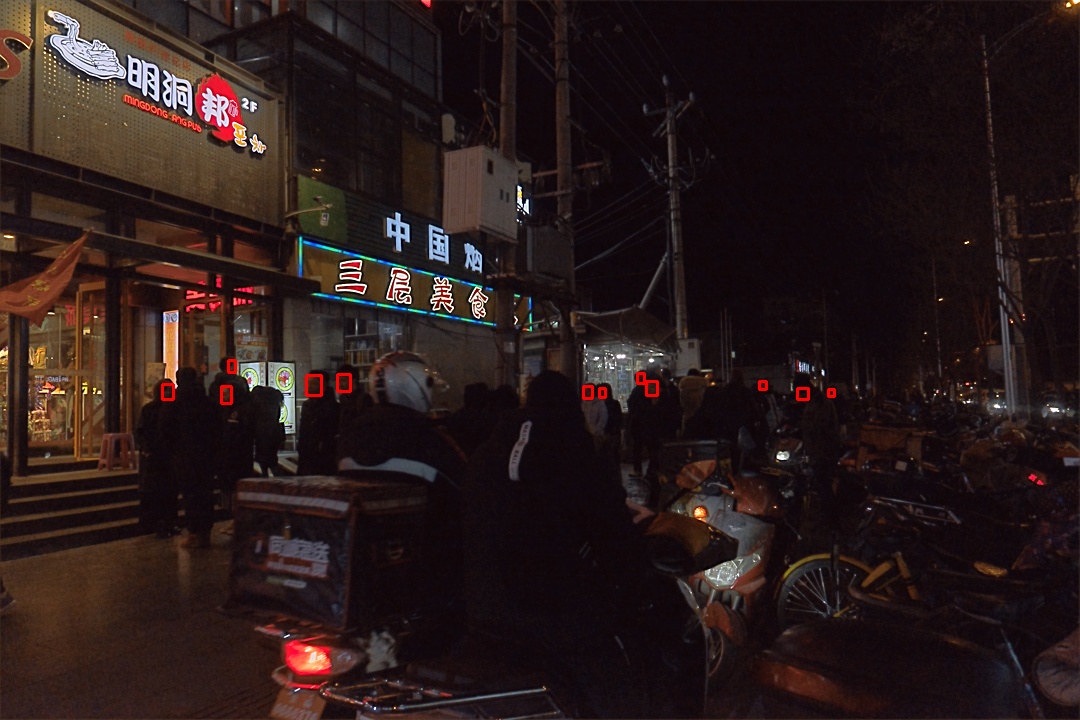}
      \caption{Ours}
      \label{fig:darkface-4-j}
  \end{subfigure}
\setlength{\abovecaptionskip}{-3pt} 
\setlength{\belowcaptionskip}{-3pt}
  \caption{{\colorRevision A visual comparison of enhancement and face detection results on DARKFACE. Please zoom in for better view.} }
  \label{fig:darkface-4}
\end{figure*}

\begin{figure}[htbp]
  \centering
\includegraphics[trim={14mm 0 0 0},clip,width=1\linewidth]{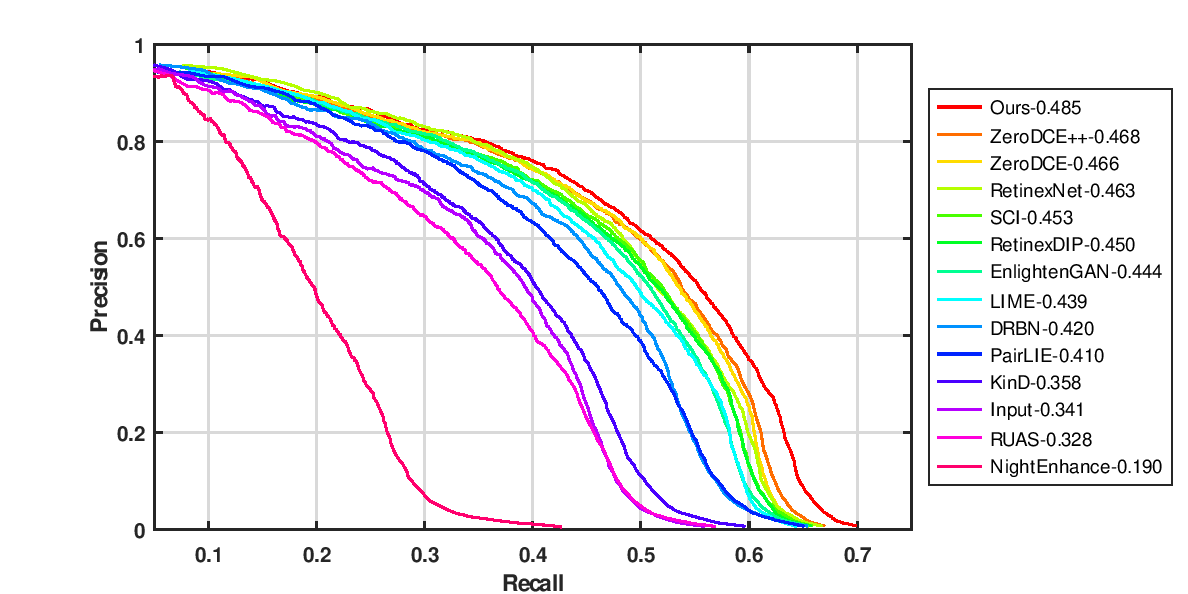}
  \caption{{\colorRevision The Precision-Recall Curve of performing detection on DARKFACE~\citep{yuan2019ug} by DSFD~\citep{li2018dsfd} for the LLIE methods.}}
  \label{fig:pr-curve}
\end{figure}

\subsection{Ablation Study}

We evaluate the effect of the proposed offset $\betam$, the variance suppression loss $\mathcal{L}_{VS}$, the reverse degradation loss $\mathcal{L}_{RD}$, and the mask function $\mathcal{M}$ in $\mathcal{L}_{RD}$. 
{\colorRevision We also conduct ablation studies on the choice of scalar or matrix of predicted parameters. }
We also trial other formulations of the mapping function $g$ besides Eq.~\ref{eq::mapping}. The expressions are listed below in Eq.~\ref{eq::othermapping}. 
\begin{small} 
\begin{equation}\label{eq::othermapping}
    \begin{split}
        g_1 (c) &= \frac{1+c}{1-c + e} \\ 
        g_2 (c) &= 3\tanh^{-1}\left(\frac{1}{2} + \frac{1}{2}c\right) = \frac{3}{2}\log(\frac{3+c}{1-c + e}) \\
        g_3 (c) &= 3 \log(g(c) + 1), \\
    \end{split}
\end{equation}
\end{small}
where $e=10^{-8}$ is a small value preventing error in program. The explanation for choosing the formulations are discussed in the appendix.

The results are shown in Table.~\ref{tab:ablation_normal}. We can conclude that each of $\betam$, $\mathcal{L}_{VS}$, $\mathcal{L}_{RD}$, and $\mathcal{M}$ leads to better enhancement. 
{\colorRevision The constant parameters cannot fully encode the complex expressions of the coefficients. We show an example in Fig.~\ref{fig::var_loss_show} and the heatmap of decomposed parameters in Fig.~\ref{fig::decomp} in the Appendix. }
For different choices of the mapping function, $g_3$ has the comparable performance to $g$. However, the double non-linearity of $\log$ and $\tan$ in $g_3$ yields slower speed and thus we choose $g$ in Eq.~\ref{eq::mapping}.

\begin{table}[htbp]
\small
  \centering
  \caption{{\colorRevision The results of our model with various settings on LOL-v1. The frames per second (FPS) are used for speed comparison.}}
    \begin{tabular}{ccccccccc}
    \toprule
         & w/o $\betam$ & w/o $\mathcal{L}_{VS}$ & w/o $\mathcal{L}_{RD}$ & w/o $\mathcal{M}$ \\
    \midrule
    PSNR  & 10.44 & 20.15 & 17.62 & 21.52 \\
    SSIM  & 0.476 & 0.721 & 0.656 & 0.757 \\
    \midrule
    \midrule
          & $g_1$    & $g_2$    & $g_3$    & Ours \\
    \midrule
    PSNR  & 17.43 & 20.34 & 21.62 & 21.54 \\
    SSIM  & 0.679 & 0.724 & 0.768 & 0.766 \\
    FPS & 986 & 891 & 831 & 954 \\
    \midrule
    \midrule
          & $b\&c$    & $\bb\&c$    & $b\&\cc$    & $\bb\&\cc$ \\
    \midrule
    PSNR  & 14.29 & 17.78 & 19.40 & 21.54 \\
    SSIM  & 0.547 & 0.653 & 0.692 & 0.766 \\
    \bottomrule
    \end{tabular}%
  \label{tab:ablation_normal}%
\end{table}%


\vspace{-1.5mm}
\subsection{Extentions}
\vspace{-1.5mm}
\noindent\textbf{Results on DARKFACE.}
Image enhancement can serve as a preprocessing for high-level downstream tasks. We thus evaluate our method on a face detection dataset in low-light conditions, i.e., DARKFACE~\citep{yuan2019ug}. The detector of Dual Shot Face Detector (DSFD)~\citep{li2018dsfd} pretrained on WIDERFACE~\citep{yang2016wider} is used for predicting face location. For the enhanced data through each LLIE method, a Precision-Recall curve (PR curve) can be drawn. The stacked PR curves for all methods are shown in Fig.~\ref{fig:pr-curve}. We can see our method has the best performance. Four samples for visual comparison is illustrated in Fig.~\ref{fig:darkface-1} to Fig.~\ref{fig:darkface-4}, where only the enhancement by our method detects the most faces correctly. Note that though some methods like NightEnhance~\citep{jin2022unsupervised} could generate brighter scenes, the introduction of more artifacts and noise reduces the face detection accuracy, which is not consistently in line with human intuition.

\noindent\textbf{Efficiency.}
As shown in Fig.~\ref{fig::space_time_perf}, our methods achieve the best overall balance between visual quality and model size/run time consistently on LOL-v1~\citep{Chen2018Retinex} and LOL-v2~\citep{yang2021sparse}. Compared to NightEnhance~\citep{jin2022unsupervised}, our models are much faster and small. Compared to SCI~\citep{ma2022toward} and ZeroDCE++~\citep{li2021learning}, ours have much higher PSNR.

\begin{figure}[htbp]
\small
  \centering
  \begin{subfigure}{0.495\linewidth}
    \includegraphics[trim={0 30mm 0 30mm},clip,width=1\linewidth]{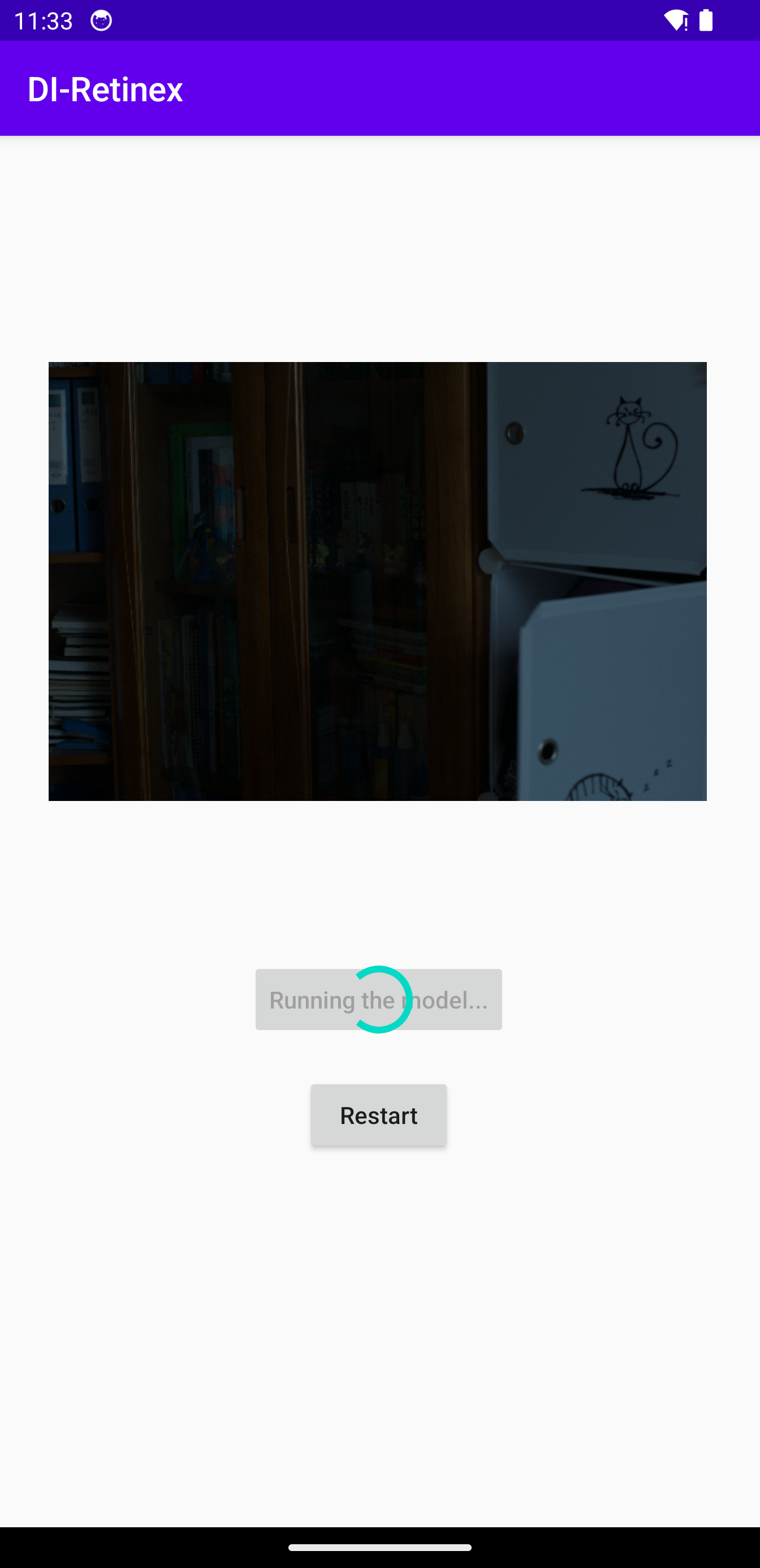}
      \caption{Screenshot during inference}
      \label{fig:mobile_input}
  \end{subfigure}
  \hfill
  \begin{subfigure}{0.495\linewidth}
    \includegraphics[trim={0 30mm 0 30mm},clip,width=1\linewidth]{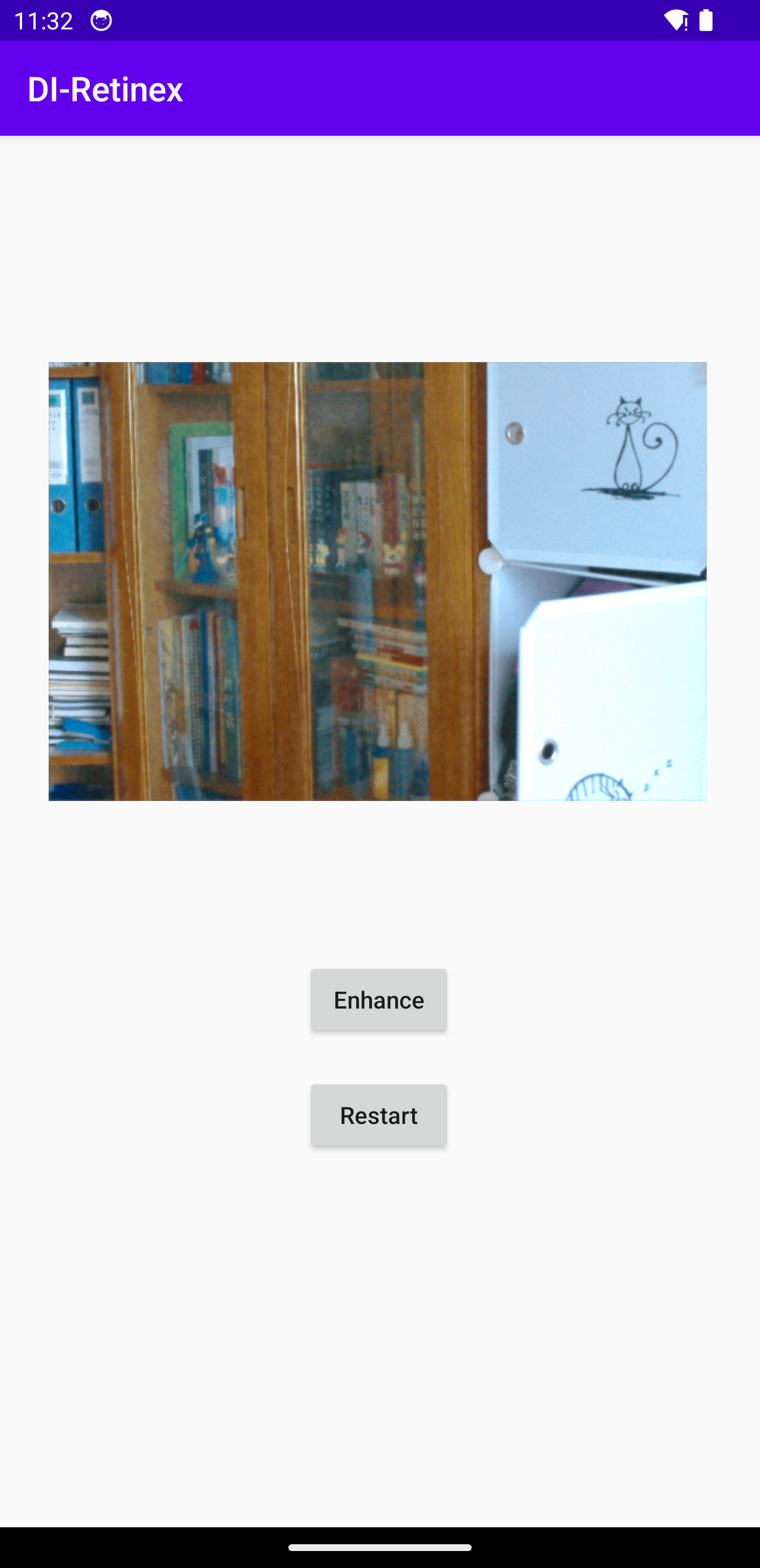}
      \caption{Screenshot of the result}
      \label{fig:mobile_result}
  \end{subfigure}
  \caption{{\colorRevision The mobile implementation based on Android.}}
  \label{fig::mobile}
  \vspace{-0.2cm}
\end{figure}

\noindent\textbf{Mobile Implementation.}
We implement out method and create an application on Android system by Android Studio based on the repository of PyTorch on Android \footnote{\href{https://github.com/pytorch/android-demo-app/}{https://github.com/pytorch/android-demo-app/}}. As seen in Fig.~\ref{fig::mobile}, our model can achieve an accurate enhancement result on a type of mobile device, Huawei Mate 20. The runtime on the device is 0.439 second per frame. The experiment shows the feasibility of applying our lightweight model to mobile application. We also trial implementing several other lightweight methods on the mobile device and test their inference time. For NightEnhance~\citep{jin2022unsupervised} and PairLIE~\citep{fu2023pairlie}, the buffer usage exceeds the device limit during converting datatype. The results are shown in Table~\ref{tab:mobile_time}. Note that the mobile implementation of PyTorch is on CPU.
\begin{table}[htbp]
  \centering
  \caption{{\colorRevision The average runtime [s] and PSNR of the lightweight LLIE models on the mobile device, Huawei Mate 20.}}
    \begin{tabular}{ccccc}
    \toprule
    & ZeroDCE & ZeroDCE++ & SCI   & Ours \\
    \midrule
    runtime & 0.822 & 0.701 & 0.762 & 0.439 \\
    PSNR & 14.86 & 15.34 & 14.78 & 21.54 \\
    \bottomrule
    \end{tabular}%
  \label{tab:mobile_time}%
\end{table}%

\vspace{-1.5mm}
\section{Conclusion}
\vspace{-1.5mm}
In this work, we extend the classic Retinex theory, targeted initially at formulating scene radiance reaching human eyes, to a complex form taking noise, quantization error, non-linearity response function, and dynamic range limitation into consideration.
By incorporating these factors, we demonstrate that the LLIE problem can be modeled through a pixel-wise linear restoration function that includes an offset with non-zero mean, which was previously ignored in other works.
%
%
Besides we introduce a contrast brightness adjustment function for enhancing low-light images in a zero-shot learning manner. A masked reverse degradation loss in Gamma space and a variance suppression loss for regulating the offset feature are derived from the proposed LLIE model. 
The reverse degradation process darkens an enhanced image back to the input low-light image, thereby reducing the offset. Our method can achieve the best performance with the fewest parameters and the least runtime.
\end{sloppypar}


  


\bibliographystyle{abbrvnat}

{\footnotesize
\bibliography{main}
}

  \section*{Appendix} 
    
\subsection*{Notation}
We first briefly summarise the notations in the paper.

\begin{equation*}
    \begin{split}
        \I&: \text{Image} \\
        \L&: \text{Illuminance} \\
        \R&: \text{Reflectance} \\
        (\cdot)_l&: \text{Component for the low light image} \\
        (\cdot)_h&: \text{Component for the normally exposed image} \\
        \epsilonm &: \text{Noise term} \\
        \deltam &: \text{Quantization error} \\
        \Delta \I &: \text{Offset term of the dynamic range compression} \\
        \oslash &: \text{Element-wise division} \\
        \prodm &: \text{Element-wise multiplication} \\
        \alpham &: \text{The ratio term we want to derive in Eq.~\ref{eq::newllie_model}} \\
        \betam &: \text{The offset term we want to derive in Eq.~\ref{eq::newllie_model}} \\
        \alpham' &: \text{The ratio term in reverse degradation process Eq.~\ref{eq::reverseDegradation}} \\
        \betam' &: \text{The offset term in reverse degradation process Eq.~\ref{eq::reverseDegradation}} \\
        k &: \text{The maximum value of intensity} \\
        q &: \text{The least significant bit} \\
        g &: \text{The mapping function we use} \\
    \end{split}
\end{equation*}

\subsection{Derivation of Eq.~\ref{eq::newllie_model}}\label{sec::derivationNewLLIE}

We have the expressions of DI-Retinex theory for both low light and normally exposed images as follows:
\begin{align*}
        \left\{
            \begin{aligned}
                \I_l &= \mathcal{G}\left(\L_l \prodm \R + \epsilonm_l\right) + \deltam_l + \Delta\I_l\\
                \I_h &= \mathcal{G}\left(\L_h \prodm \R + \epsilonm_h\right) + \deltam_h + \Delta\I_h.
            \end{aligned}
        \right. 
\end{align*}

Plug the expression of $\mathcal{G}$.

\begin{align*}
        \left\{
            \begin{aligned}
                \I_l &= \mu + \lambda\left(\L_l \prodm \R + \epsilonm_l\right)^{\gamma} + \deltam_l + \Delta\I_l\\
                \I_h &= \mu + \lambda\left(\L_h \prodm \R + \epsilonm_h\right)^{\gamma} + \deltam_h + \Delta\I_h.
            \end{aligned}
        \right. 
\end{align*}

By transforming the first expression, we have
\begin{equation*}
    \R = \left\{ \left[\frac{1}{\lambda}\left(\I_l - \deltam_l - \Delta\I_l - \mu\right) \right]^{\frac{1}{\gamma}} - \epsilonm_l\right\} \oslash \L_l.
\end{equation*}

Plug it into the expression of extended Retinex theory for normally exposed image. Suppose $\r=\L_h \oslash \L_l \ge \onem$. Then we have
\begin{equation*}
\begin{split}
    \I_h = &\lambda\left(\r \prodm  \left[\frac{1}{\lambda}\left(\I_l - \deltam_l - \Delta\I_l - \mu\right) \right]^{\frac{1}{\gamma}} - \r \prodm \epsilonm_l + \epsilonm_h\right)^{\gamma} \\
    & + \mu + \deltam_h + \Delta\I_h.
\end{split}
\end{equation*}

\noindent  According to the approximation that 
\begin{equation}\label{eq::approximation}
    (1+x)^n=1+nx ,
\end{equation}
\noindent when $x\rightarrow 0$ by Taylor expansion, we can get 
\begin{equation*}
\begin{split}
    \I_h \approx &\lambda\r^\gamma \prodm  \left[\frac{1}{\lambda}\left(\I_l - \deltam_l - \Delta\I_l - \mu\right) \right]^{}  \\
    & + \lambda \gamma \r^{\gamma-1} (\epsilonm_h - \r \prodm \epsilonm_l) \left[\frac{1}{\lambda}\left(\I_l - \deltam_l - \Delta\I_l - \mu\right) \right]^{\frac{\gamma-1}{\gamma}} \\
    & + \mu + \deltam_h + \Delta\I_h.
\end{split}
\end{equation*}

\noindent Let the second term be
$$T = \lambda \gamma \r^{\gamma-1} (\epsilonm_h - \r \prodm \epsilonm_l) \left[\frac{1}{\lambda}\left(\I_l - \deltam_l - \Delta\I_l - \mu\right) \right]^{\frac{\gamma-1}{\gamma}},$$ 
and use approximation in Eq.~\ref{eq::approximation}, thus we have 
\begin{equation*}
\begin{split}
    T &= \lambda^{\frac{1}{\gamma}} \gamma \r^{\gamma-1} (\epsilonm_h - \r \prodm \epsilonm_l) \left(\I_l - \deltam_l - \Delta\I_l - \mu\right)^{\frac{\gamma-1}{\gamma}} \\
    &\begin{split}
        \approx &\lambda^{\frac{1}{\gamma}} \gamma \r^{\gamma-1} (\epsilonm_h - \r \prodm \epsilonm_l) \\
        &\left[\I_l^{\frac{\gamma-1}{\gamma}} + {\frac{1-\gamma}{\gamma}}(\deltam_l + \Delta\I_l + \mu) \I_l^{\frac{-1}{\gamma}}\right].
    \end{split}
\end{split}
\end{equation*}

\noindent Note that $\gamma \approx 0.5$, $\r \ge \onem$ and thus $\r^{\gamma-1} \le \onem$. The noise term $(\epsilonm_h - \r \prodm \epsilonm_l)$ is small. When $\I_l\in[1,255]$, we have $\I_l^{\frac{\gamma-1}{\gamma}} \le \onem$ and $\I_l^{\frac{-1}{\gamma}} \le \onem$, which ensures $T\rightarrow 0$ compared to the maximum intensity of $255$ for $\I_h$. When $\I_{l,i,j}=0$ at $(i,j)$-th pixel, $\L_{l,i,j}$ also equals zero and thus $\r_{i,j}=+\infty \Rightarrow \r_{i,j}^{\gamma-1}\rightarrow 0$. Besides, the probability of $\I_{l,i,j}$ exactly equaling to zero is extremely small among pixels in reality. 
The second term $T$ is thus much smaller than the other terms and can be safely omitted, which yields
\begin{equation*}
\begin{split}
    \I_h &\approx \r^\gamma \prodm \left(\I_l - \deltam_l - \Delta\I_l - \mu\right) + \mu + \deltam_h + \Delta\I_h\\
    &= \r^\gamma \prodm \I_l + \left[\mu + \deltam_h + \Delta\I_h - \r^\gamma \prodm \left( \deltam_l + \Delta\I_l + \mu\right)\right].
\end{split}
\end{equation*}

\noindent Therefore, we can obtain
\begin{align*}
    &\I_h \approx \alpham \I_l + \betam. \\
    \text{where} &\left\{
    \begin{aligned}
        \alpham &= \r^{\gamma} = \left(\L_h \oslash \L_l \right)^\gamma \\
        \betam &= \mu + \deltam_h + \Delta\I_h - \r^\gamma \prodm \left( \deltam_l + \Delta\I_l + \mu\right).
    \end{aligned}
    \right.
\end{align*}

\subsection{Derivation of Eq.~\ref{eq::reverseDegradation}}\label{sec::derivationReverseDegradation}
We want to show 
\begin{align*}
    &\I_l \approx \alpham' \I_h + \betam', \\
    \text{where} &\left\{
    \begin{aligned}
        \alpham' &= \r^{-\gamma} = \left(\L_l \oslash \L_h \right)^\gamma \\
        \betam' &= \mu + \deltam_l + \Delta\I_l - \r^{-\gamma} \prodm \left( \deltam_h + \Delta\I_h + \mu\right).
    \end{aligned}
    \right.
\end{align*}
Similar to the derivation of Eq.~\ref{eq::newllie_model},  we start from
\begin{align*}
        \left\{
            \begin{aligned}
                \I_l &= \mathcal{G}\left(\L_l \prodm \R + \epsilonm_l\right) + \deltam_l + \Delta\I_l\\
                \I_h &= \mathcal{G}\left(\L_h \prodm \R + \epsilonm_h\right) + \deltam_h + \Delta\I_h.
            \end{aligned}
        \right. 
\end{align*}
and we can get
\begin{equation*}
\begin{split}
    \I_l = &\lambda\left(\r^{-1} \prodm  \left[\frac{1}{\lambda}\left(\I_h - \deltam_h - \Delta\I_h - \mu\right) \right]^{\frac{1}{\gamma}} - \r^{-1} \prodm \epsilonm_h + \epsilonm_l\right)^{\gamma} \\
    & + \mu + \deltam_l + \Delta\I_l.
\end{split}
\end{equation*}
Then by using approximation in Eq.~\ref{eq::approximation} we have
\begin{equation*}
\begin{split}
    \I_l \approx &\lambda\r^{-\gamma} \prodm  \left[\frac{1}{\lambda}\left(\I_h - \deltam_h - \Delta\I_h - \mu\right) \right]^{}  \\
    & + \lambda \gamma \r^{1-\gamma} (\epsilonm_l - \r^{-1} \prodm \epsilonm_h) \left[\frac{1}{\lambda}\left(\I_h - \deltam_h - \Delta\I_h - \mu\right) \right]^{\frac{\gamma-1}{\gamma}} \\
    & + \mu + \deltam_l + \Delta\I_l.
\end{split}
\end{equation*}
Let the second term be
\begin{equation*}
\begin{split}
    T' &= \lambda^{\frac{1}{\gamma}} \gamma \r^{1-\gamma} (\epsilonm_l - \r^{-1} \prodm \epsilonm_h) \left(\I_h - \deltam_h - \Delta\I_h - \mu\right)^{\frac{\gamma-1}{\gamma}} \\
    &\begin{split}
        \approx &\lambda^{\frac{1}{\gamma}} \gamma \r^{1-\gamma} (\epsilonm_l - \r^{-1} \prodm \epsilonm_h) \\
        &\left[\I_h^{\frac{\gamma-1}{\gamma}} + {\frac{1-\gamma}{\gamma}}(\deltam_h + \Delta\I_h + \mu) \I_h^{\frac{-1}{\gamma}}\right].
    \end{split}
\end{split}
\end{equation*}
Note that $\r \ge \onem$, and thus $\r^{-1} \ge \onem$. The noise term $(\epsilonm_l - \r^{-1} \prodm \epsilonm_h)$ is small. When $\I_h\in[1,255]$, we have $\I_h^{\frac{\gamma-1}{\gamma}} \le \onem$ and $\I_h^{\frac{-1}{\gamma}} \le \onem$, which ensures $T\rightarrow 0$ compared to the maximum intensity. When $\I_{h,i,j}=0$ at $(i,j)$-th pixel, $\L_{h,i,j}$ also equals zero and thus $\r_{i,j}^{1-\gamma}\rightarrow 0$. Besides, the probability of $\I_{h,i,j}$ exactly equaling to zero is extremely small among pixels in reality. 
The second term $T'$ is thus much smaller than the other terms and can be safely omitted, which yields
\begin{equation*}
\begin{split}
    \I_l &\approx \r^{-\gamma} \prodm \left(\I_h - \deltam_h - \Delta\I_h - \mu\right) + \mu + \deltam_l + \Delta\I_l\\
    &= \r^{-\gamma} \prodm \I_h + \left[ \mu + \deltam_l + \Delta\I_l - \r^{-\gamma} \prodm \left( \deltam_h + \Delta\I_h + \mu\right) \right].
\end{split}
\end{equation*}
Therefore, we can obtain
\begin{align*}
    &\I_l \approx \alpham' \I_h + \betam. \\
    \text{where} &\left\{
    \begin{aligned}
        \alpham' &= \r^{-\gamma} = \left(\L_l \oslash \L_h \right)^\gamma \\
        \betam' &= \mu + \deltam_l + \Delta\I_l - \r^{-\gamma} \prodm \left( \deltam_h + \Delta\I_h + \mu\right).
    \end{aligned}
    \right.
\end{align*}

\subsection{Discussion of Linear Model and the Offset $\betam$}\label{sec::discussionBeta}
The expression of $\betam$ is given by
$$\betam = \mu + \deltam_h + \Delta\I_h - \alpham \prodm \left(\mu + \deltam_l + \Delta\I_l\right),$$
\noindent where $\alpham=\left(\L_h \oslash \L_l \right)^\gamma \ge \onem$. We analyse the magnitude of $\betam$ by computing its mean and variance. Given $\mathbb{E}[\alpham] =\overline{\alpham} > 1$ and $\mathbb{E}[\deltam_l]=\mathbb{E}[\deltam_h]=0$ and suppose $\alpham$ is independent of $\Delta\I_h$ and $\Delta\I_l$, we have
$$\mathbb{E}[\betam] = \mu(1-\overline{\alpham} ) + \mathbb{E}[\Delta\I_h] - \overline{\alpham}\mathbb{E}[\Delta\I_l] .$$
A properly exposed has equal probability of overexposure and underexposure. Thus $\mathbb{E}[\Delta\I_h] = 0$. Also a negative overflow exceeding $0$ caused by a negative image irradiance is seldom and thus $\mathbb{E}[\Delta\I_l]\rightarrow 0$. $\mu$ is a negative number such that the Gamma transformation can map $[0, +\infty)$ to $[0,255]$. $(1-\overline{\alpham}) < -1$ (found in the following statistical experiments) makes the magnitude of $\mathbb{E}[\betam]$ significant. Even if $\mu(1-\overline{\alpham} ) = \overline{\alpham}\mathbb{E}[\Delta\I_l]$ in some rare case, $\Delta\I_l$ has non-zero values only in some regions of an image and $\mu(1-\overline{\alpham} )$ has large values in the regions where $\Delta\I_l$ has zero values. Therefore, $\betam$ is not negligible in terms of considerable mean.

We then compute $\betam$'s variance. The variance of $\mu$ is zero as it is a constant. Since the quantization error follows a uniform distribution from $\-q/2$ to $q/2$, its variance is $q^2/12$. Note that with the assumption of flux consistency along time dimension across spatial pixels during receiving radiance in a short period, $\alpham$ can be further treated as a constant matrix. So it can be considered as a scalar $\alpha$ when computing $\betam$'s variance,
\begin{equation*}
\begin{split}
    Var(\betam) &= (q^2/12) + Var(\Delta\I_h) +\alpha^2 \left((q^2/12) + Var(\Delta\I_l) \right) \\
    &= (1+\alpha^2) (q^2/12) + Var(\Delta\I_h) + \alpha^2Var(\Delta\I_l).
\end{split}
\end{equation*}
As shown, the factor of $\alpha^2$ enlarges the variance of $\betam$. The amplified variance also makes $\betam$ non-negligible.





\subsection{Discussion on Reverse Degradation}\label{sec::discussionBetaPrime}


The expression of $\betam'$ is given by
$$\betam' = \mu + \deltam_l + \Delta\I_l - \alpham' \prodm \left(\mu + \deltam_h + \Delta\I_h\right),$$
\noindent where $\alpham'=\left(\L_l \oslash \L_h \right)^\gamma \le \onem$. We analyse the magnitude of $\betam'$ by computing its mean and variance. Given $\mathbb{E}[\alpham'] =\overline{\alpham'} \in [0,1]$ and $\mathbb{E}[\deltam_h]=\mathbb{E}[\deltam_l]=0$ and suppose $\alpham'$ is independent of $\Delta\I_h$ and $\Delta\I_l$, we have
$$\mathbb{E}[\betam'] = \mu(1-\overline{\alpham'} ) + \mathbb{E}[\Delta\I_l] - \overline{\alpham'}\mathbb{E}[\Delta\I_h].$$
A properly exposed has equal probability of overexposure and underexposure. Thus $\mathbb{E}[\Delta\I_h] = 0$. Also a negative overflow exceeding $0$ caused by a negative image irradiance is seldom and thus $\mathbb{E}[\Delta\I_l]\rightarrow 0$. $\mu$ is a negative number such that the Gamma transformation can map $[0, +\infty)$ to $[0,255]$. $1-\overline{\alpham'}\in [0,1]$. This makes $\mathbb{E}[\betam'] \approx \mu(1-\overline{\alpham'} )$ getting small. 
For even more safely removing $\betam'$, we compute reverse degradation loss in Gamma space by taking a power of $\frac{1}{\gamma}$, which suppresses the magnitude of $\betam'$.

{\colorRevision
We then compute the variance of $\betam'$. Similar to Section~\ref{sec::discussionBeta}, we can get
\begin{equation*}
\begin{split}
    Var(\betam) &= (q^2/12) + Var(\Delta\I_l) +\alpha'^2 \left((q^2/12) + Var(\Delta\I_h) \right) \\
    &= (1+\alpha'^2) (q^2/12) + Var(\Delta\I_l) + \alpha'^2Var(\Delta\I_h),
\end{split}
\end{equation*}
where $\alpha'$ represents the scalar form of $\alpham$, assuming flux consistency across spatial pixels over a brief time period when receiving radiance. As depicted, the factor $\alpha'^2$ assumes a small value, given that $\alpha'$ is less than 1 while the max pixel value is 255. The variance of $\betam'$ closely approximates $(q^2/12) + Var(\Delta\I_l)$, which denotes the variance arising from error and offset terms in low-light images. It is important to note that the variance is small and thus approaching zero.
}




\subsection{Discussion on Mapping Function $g$ and Contrast Brightness Algorithm}\label{sec::MappingFunc_g}

We use four nonlinear function mapping $[-1,1]$ to $[0, +\infty)$. According to the finding of statistical experiments in Section~\ref{sec::discussionBeta}. The ratio between the exposures of low light and normally exposed images is mainly within $[1,10]$, we thus make sure that most output values of mapping function lie in the region. The formula of four used functions are given below and their curves are drawn in Fig.~\ref{fig:compare_g}.
\begin{equation*}
    \begin{split}
        g (c) &= \tan\left( \frac{45 + (45 - \tau) c}{180} \pi \right) \\
        g_1 (c) &= \frac{1+c}{1-c + e} \\ 
        g_2 (c) &= 3\tanh^{-1}\left(\frac{1}{2} + \frac{1}{2}c\right) = \frac{3}{2}\log(\frac{3+c}{1-c + e}) \\
        g_3 (c) &= 3 \log(g(c) + 1).
    \end{split}
\end{equation*}
A shifted reciprocal function like $g_1$ is a straightforward and common choice. But we can see that $g_1$ has an extremely steep rise when $c>0.8$ and thus generates the worst result in Table.~\ref{tab:ablation_normal}. To alleviate the steep rise, we apply a logarithm function to shifted $g_1$ and form $g_2$. It coincidentally becomes a shifted inverse hyperbolic tangent function. But it still cannot lead to a good enough performance. Then we choose a tangent function $g$ and its performance is satisfactory. We further try its logarithm $g_3$ with a gentler slope. It has even better effect but the nesting of two nonlinear functions yields more computations. Considering the trade-off, we just use $g$. 

\begin{figure}
  \centering
    \includegraphics[width=0.4\linewidth]{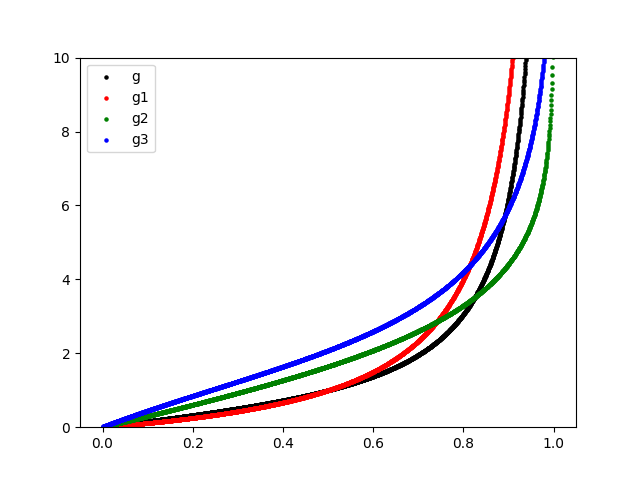}
  \caption{The curves of four used mapping functions.}
  \label{fig:compare_g}
\end{figure}

The majority of existing works about image contrast enhancement are based on histogram equalization~\citep{padmavathy2018image,vijayalakshmi2020comprehensive,Park2018ContrastEF}. Though some variants of Eq.~\ref{eq::contrastbrightness3} is used empirically in industry for adjusting image's contrast, there is no academic paper discussing it to the best of our knowledge. One of our contributions lies in employing it to solve low light enhancement problem with the zero-reference regulation of two losses. Compared with other functions with single coefficient, the formula involving two coefficients is more difficult for a neural network to learn. Our derived reverse degradation loss and variance suppression loss can effectively guide the network to learn the coefficient pair of the adjustment function. We experiment various mapping functions $g$ aforementioned and modified the original formulation with the best choice of $g$. 
{\colorRevision
\subsection{Network Structure}
Our model is extremely simple and its structure is shown in Table~\ref{tab:structure}.
\vspace{-3.3mm}
\begin{table}[htbp]
\renewcommand{\arraystretch}{0.4}
\captionsetup{font={footnotesize}}
\scriptsize
  \centering
  \caption{The structure of the network. $c_{int}$ and $c_{out}$ are the internal and output channel and are specified in paper. The output is split for $\mathbf{b}$ and $\mathbf{c}$}
  \vspace{-3.5mm}
    \begin{tabular}{cccccc}
    \toprule
    layer & channel\_in & channel\_out & kernel & stride & padding \\
    \midrule
    Conv+ReLU & 3     & $c_{int}$    & 3     & 1     & 1 \\
    Conv+ReLU & $c_{int}$    & $c_{int}$    & 3     & 1     & 1 \\
    Conv+Tanh & $c_{int}$    & $c_{out}$ & 3     & 1     & 1 \\
    \bottomrule
    \end{tabular}%
  \label{tab:structure}%
  \vspace{-4mm}
\end{table}%

\subsection{Decomposition Visualization}
We showcase an visual example of decomposition of our enhancement model in Fig.~\ref{fig::decomp}.
\begin{figure}[hpt]
\small
  \centering
  \begin{minipage}[b]{0.125\linewidth}
    \centering
  \begin{subfigure}{1.0\linewidth}
    \includegraphics[width=1\linewidth]{figs/lolv2/input/low00776.png}
      \caption{\(\I_l\)}
  \end{subfigure}
  \hfill
  \begin{subfigure}{1.0\linewidth}
    \includegraphics[width=1\linewidth]{figs/lolv2/ours_/low00776.png}
      \caption{\(\tilde \I_h\)}
  \end{subfigure}
  \end{minipage}
  \hfill
  \begin{subfigure}{0.280\linewidth}
    \includegraphics[width=1\linewidth]{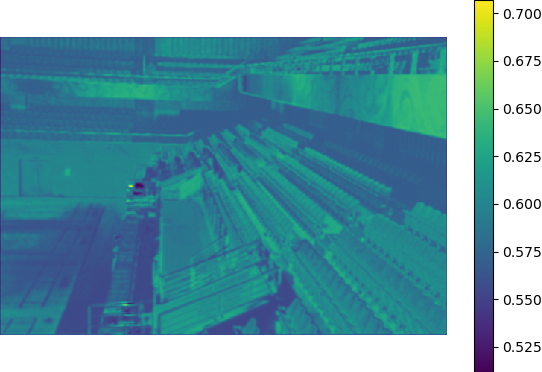}
      \caption{Heatmap of \(\bb\)}
  \end{subfigure}
  \hfill
  \begin{subfigure}{0.280\linewidth}
    \includegraphics[width=1\linewidth]{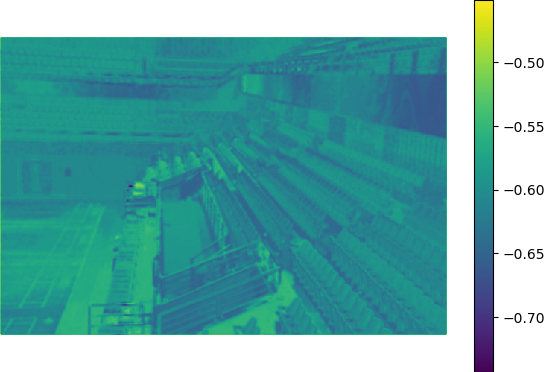}
      \caption{Heatmap of \(\cc\)}
  \end{subfigure}
  \hfill
  \begin{subfigure}{0.280\linewidth}
    \includegraphics[width=1\linewidth]{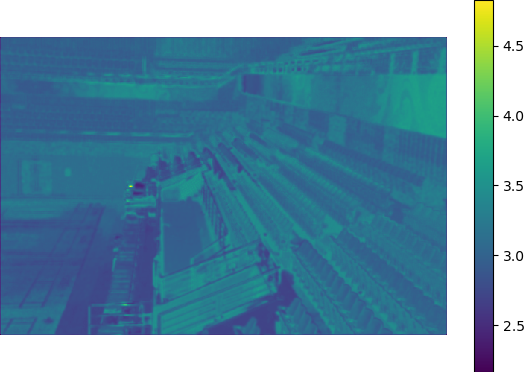}
      \caption{Heatmap of \(\a\)}
  \end{subfigure}
  \caption{{\colorRevision The decomposition results of the model.}}
  \label{fig::decomp}
  \vspace{-0.2cm}
\end{figure}
}
\end{document}